\documentclass[twoside]{article}

\usepackage[accepted]{aistats2025}
\usepackage[round]{natbib}
\usepackage{url,hyperref}
\usepackage{makecell}
\usepackage{tabularx,colortbl,booktabs}
\usepackage{amsmath,amssymb,amsthm,amsfonts,enumitem,color}
\usepackage[pdftex]{graphicx}
\usepackage[export]{adjustbox}
\usepackage{subfigure}
\usepackage{stackengine}
\usepackage[dvipsnames]{xcolor}
\usepackage{siunitx}
\allowdisplaybreaks

\definecolor{mplblue}{RGB}{31, 119, 180}
\definecolor{darkorange}{RGB}{255, 140, 0}

\usepackage{mathtools}

\newcommand\numberthis{\addtocounter{equation}{1}\tag{\theequation}}

\usepackage{amsmath,amssymb,amsthm,amsfonts,bm,bbm, nicefrac}
\usepackage{algorithm}
\usepackage{algorithmic}
\usepackage{xcolor}         
\definecolor{mydarkblue}{rgb}{0,0.08,0.45}
\definecolor{DarkBlue}{rgb}{0,0,0.55}
\definecolor{mygreen}{rgb}{0,0.7,0.7}

\def\vr{{\bm{r}}}

\DeclareMathAlphabet{\mathsfit}{\encodingdefault}{\sfdefault}{m}{sl}
\SetMathAlphabet{\mathsfit}{bold}{\encodingdefault}{\sfdefault}{bx}{n}

\def\gA{{\mathcal{A}}}
\def\gB{{\mathcal{B}}}
\def\gC{{\mathcal{C}}}

\def\gE{{\mathcal{E}}}

\def\gG{{\mathcal{G}}}
\def\gH{{\mathcal{H}}}

\def\gN{{\mathcal{N}}}

\def\gT{{\mathcal{T}}}

\def\gX{{\mathcal{X}}}

\def\sN{{\mathbb{N}}}

\def\sP{{\mathbb{P}}}

\def\sR{{\mathbb{R}}}
\def\sS{{\mathbb{S}}}

\def\sV{{\mathbb{V}}}

\newcommand{\tr}{\textup{Tr}}
\newcommand{\hatmu}{\widehat{\mu}}
\newcommand{\hatsig}{\widehat{\Sigma}}
\newcommand{\sigr}{\Sigma_\textup{r}}
\newcommand{\mur}{\mu_\textup{r}}

\newcommand{\E}{\mathbb{E}}

\DeclareMathOperator*{\argmax}{arg\,max}
\DeclareMathOperator*{\argmin}{arg\,min}

\newenvironment{enumerate*}%
{\begin{enumerate}[leftmargin=*,topsep=0pt]%
		\setlength{\itemsep}{0pt}%
		\setlength{\parskip}{0pt}}%
	{\end{enumerate}}

\theoremstyle{plain}
\newtheorem{theorem}{Theorem}

\newtheorem{lemma}{Lemma}

\theoremstyle{definition}
\newtheorem{definition}{Definition}

\theoremstyle{remark}
\newtheorem{remark}{Remark}

\def\showcomments{1}
\ifnum\showcomments=1

\begin{document}

%

%

\twocolumn[

\aistatstitle{A Multi-Armed Bandit Approach to Online Selection and Evaluation of Generative Models}

\aistatsauthor{ Xiaoyan Hu \And Ho-fung Leung \And  Farzan Farnia }

\aistatsaddress{The Chinese University of Hong Kong \And Independent Researcher \And  The Chinese University of Hong Kong} ]

\begin{abstract}
  Existing frameworks for evaluating and comparing generative models consider an offline setting, where the evaluator has access to large batches of data produced by the models. However, in practical scenarios, the goal is often to identify and select the best model using the fewest possible generated samples to minimize the costs of querying data from the sub-optimal models. In this work, we propose an online evaluation and selection framework to find the generative model that maximizes a standard assessment score among a group of available models. We view the task as a multi-armed bandit (MAB) and propose upper confidence bound (UCB) bandit algorithms to identify the model producing data with the best evaluation score that quantifies the quality and diversity of generated data. Specifically, we develop the MAB-based selection of generative models considering the Fréchet Distance (FD) and Inception Score (IS) metrics, resulting in the FD-UCB and IS-UCB algorithms. We prove regret bounds for these algorithms and present numerical results on standard image datasets. Our empirical results suggest the efficacy of MAB approaches for the sample-efficient evaluation and selection of deep generative models. The project code is available at \url{https://github.com/yannxiaoyanhu/dgm-online-eval}.
\end{abstract}

\section{\MakeUppercase{INTRODUCTION}}
\label{sec:intro}
Deep generative models have achieved astonishing results across a wide array of machine learning datasets. Quantitative comparisons between generative models, trained using different methods and architectures, are commonly performed by evaluating assessment metrics such as Fréchet Inception Distance (FID)~\citep{NIPS2017_8a1d6947, NEURIPS2023_0bc795af} and Inception Score (IS)~\citep{NIPS2016_8a3363ab}. Due to the growing applications of generative models to various learning tasks, the machine learning community has continuously adapted evaluation methodologies to better suit the characteristics of the newly introduced applications.

A common characteristic of standard evaluation frameworks for deep generative models is their offline assessment process, which requires a full batch of generated data for assigning scores to the models. While this offline evaluation does not incur significant costs for moderate-sized generative models, producing large batches of samples from large-scale models can be costly. In particular, generating a large batch of high-resolution image or video data could be expensive and hinder the application of existing evaluation scores for ranking generative models.

In this work, we focus on the \emph{online evaluation and selection} of generative modeling schemes, where we consider a group of generative models and attempt to identify the model with the best score by assessing the fewest number of produced data. By limiting the number of generated samples in the assessment process, online evaluation can save on the costs associated with producing large batches of samples. Additionally, online evaluation can significantly reduce the time and computational expenses required to identify a well-performing model in assessing a large group of generative models. Figure~\ref{fig1} shows an example of FD-based evaluation of three pretrained models where the embeddings are extracted by InceptionNet.V3. While the offline evaluation procedure requires a large batch of generated data from each model, our online evaluation approach queries generated data from models adaptively and limits sample generation from the sub-optimal models, thus collecting samples from the best model more frequently~(Figure~\ref{fd-results-fig1}).

\begin{figure*}[!t]
\centering
\includegraphics[width=0.78\textwidth]{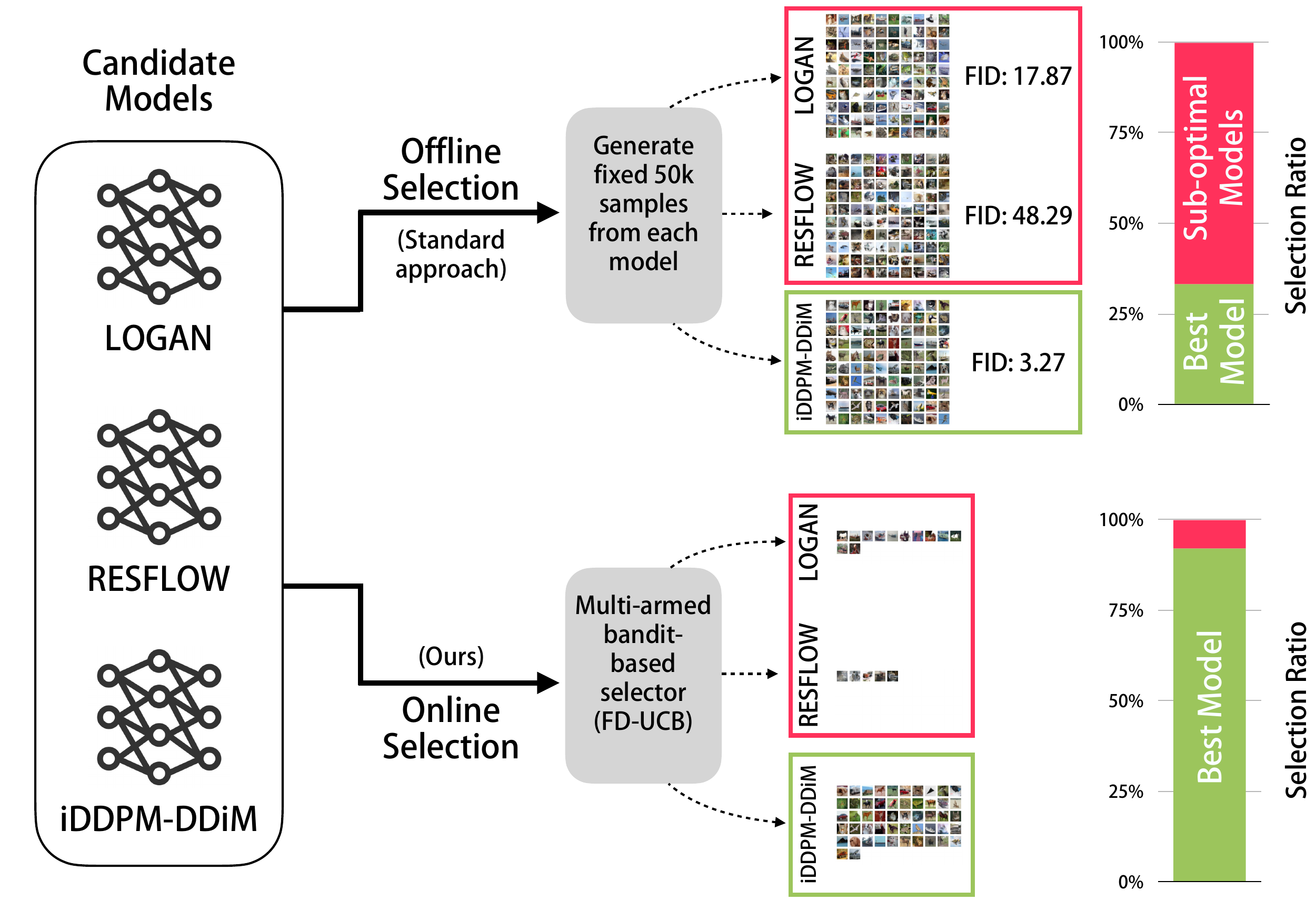}
\caption{FID-based evaluation and selection among CIFAR10 generative models: The standard offline evaluation requires a large batch of data from every model. In contrast, our proposed online approach leverages the UCB multi-armed bandit strategy to identify the best model with fewer generations from the suboptimal models.}
\label{fig1}
\end{figure*}

To measure the performance of an online evaluation algorithm, we use the \textit{regret} notion widely used in the online learning literature~\citep{bubeck2012regret}. In the online learning task, the evaluation algorithm selects a generative model and observes a mini-batch of samples generated by that model in each round. Our goal is to minimize the algorithm’s regret, defined as the cumulative difference between the scores of the selected models and the best possible score from the model set. Therefore, the learner's regret quantifies the cost of generating samples from sub-optimal models. We aim to develop online learning algorithms that result in lower regret values in the assessment of generated data. 

To address the described online evaluation of generative models, we utilize the multi-armed bandit (MAB) framework. In a standard MAB setup, the online learner seeks to identify the arm with the highest expected value of a random score. While this formulation has been widely considered in many learning settings, it cannot solve the online evaluation of generative models, as standard IS and FD metrics do not simplify to the expectation of a random variable and represent a non-linear function of the data distribution.

By deriving concentration bounds for the FD and IS scores, we propose the optimism-based FD-UCB and IS-UCB algorithms to address the online assessment of generative models. These algorithms apply the upper confidence bound (UCB) approach using data-dependent bounds that we establish for the FD and IS scores. Furthermore, we analytically bound the regret of the IS-UCB and FD-UCB methods, demonstrating their sub-linear regret growth assuming a full-rank covariance matrix of the embedded real data.

We discuss the results of several numerical applications of FD-UCB and IS-UCB to standard image datasets and generative modeling frameworks. We compare the performance of these algorithms with the Greedy algorithm, which selects the model with the highest estimated score, and a Naive-UCB baseline that applies the UCB algorithm with a data-independent upper confidence bound. Our numerical results show a significant improvement in our proposed algorithms compared to the Greedy and Naive-UCB baselines. Additionally, FD-UCB can attain satisfactory performance under the standard image data embeddings, including InceptionNet.V3~\citep{7780677}, DINOv2~\citep{oquab2024dinov2}, and CLIP~\citep{pmlr-v139-radford21a}. Our work demonstrates the effectiveness of MAB-based approaches to the evaluation and selection of generative models. The following summarizes the main contributions of this work:
\begin{itemize}[leftmargin=*]
\item Proposing a MAB-based evaluation framework for generative models that aims to minimize the regret of misidentifying the score-maximizing model from online generated data.
\item Developing the FD-UCB and IS-UCB algorithms by applying the upper-confidence-bound framework to our data-dependent estimation of the scores.
\item Proving sub-linear regret bounds for the proposed FD-UCB and IS-UCB algorithms.
\item Demonstrating satisfactory empirical performance of FD-UCB and IS-UCB in comparison to the Greedy and Naive-UCB baseline algorithms.
\end{itemize}

\section{\MakeUppercase{RELATED WORK}}
\label{sec:related_work}
\noindent\textbf{Assessment of deep generative models.}\quad The evaluation of generative models has been extensively studied in the literature. Several evaluation metrics are proposed, including distance-based metrics such as Wasserstein critic~\citep{pmlr-v70-arjovsky17a}, Fr\'echet Inception Distance (FID)~\citep{NIPS2017_8a1d6947} and Kernel Inception Distance (KID)~\citep{2018demystifying}, and diversity/quality-based metrics such as Precision/Recall~\citep{NEURIPS2018_f7696a9b, NEURIPS2019_0234c510}, density and coverage~\citep{pmlr-v119-naeem20a}, Vendi~\citep{friedman2023the}, RKE~\citep{NEURIPS2023_1f5c5cd0}, and scalable FKEA-Vendi \cite{ospanov2024towards}. In addition, the related works develop metrics quantifying the generalizability of the generative models, including the authenticity score~\citep{pmlr-v162-alaa22a}, the FLD score~\citep{jiralerspong2023feature}, the Rarity score~\citep{han2023rarity}, and the KEN score~\citep{zhang_interpretable_2024,zhang2024identification} measuring the novelty of the generated samples. In this paper, we primarily focus on FID and Inception Score, which have been frequently used for evaluating generative models.

\noindent\textbf{Role of embeddings in the quantitative evaluation results.}\quad Due to the high-dimensionality of images, evaluation of the generated images mostly relies on the embeddings extracted by pretrained networks on the ImageNet dataset, e.g., InceptionNet.V3~\citep{7780677}. However, \cite{pmlr-v119-naeem20a} shows that such pretrained embeddings could exhibit unexpected behaviors. Recently, several large pretrained models have been proposed, including DINOv2~\citep{oquab2024dinov2} and CLIP~\citep{cherti2023reproducible}. \cite{NEURIPS2023_0bc795af} shows that DINOv2-ViT-L/14 enables more interpretable evaluation of generative models. In addition, \cite{kynknniemi2023the} demonstrates that FID scores computed based on the embedding extracted by CLIP agree more with human-based assessments. In this paper, we provide the numerical results for the mentioned embeddings extracted by different pretrained models.

\noindent\textbf{Online learning using diversity-related evaluation metrics.}\quad Online learning is a sequential decision-making framework where an agent aims to minimize a cumulative loss function revealed to her sequentially. One popular setting is the multi-armed bandit~(MAB), whose study dates back to the work of~\cite{LAI19854, dc35850b-2ca1-314f-9e0d-470713436b17}. At each step, the agent chooses among several arms, each associated with a reward distribution, and aims to maximize a pre-specified performance metric. The primary concern of this body of literature considers maximizing the expected return~\citep{0b15adb3-762e-3845-a962-f0948b51f864,10.5555/944919.944941, bubeck2012regret}. Recent works study performance metrics related to the variance or entropy of the reward distribution, including the mean-variance criterion~\citep{NIPS2012_83f25503, pmlr-v119-zhu20d} in risk-sensitive MAB, and \emph{informational MAB}~(IMAB) where the agent maximizes the entropy rewards~\citep{10196497}. However, to the best of our knowledge, the evaluation of generative models has not been exclusively studied in an online learning context. In a concurrent work, \citet{rezaei2025be} aim to maximize kernel-based evaluation scores by finding a mixture of generative models. In this work, we primarily focus on identifying the best (single) model.

\noindent\textbf{Online training of generative models.}\quad Several related works focus on training generative adversarial networks (GANs)~\citep{NIPS2014_5ca3e9b1} using online learning frameworks. \cite{grnarova2018an} proposes to train \emph{semi-shallow} GANs using the \emph{Follow-the-Regularized-Leader}~(FTRL) approach. \cite{daskalakis2018training} shows that optimistic mirror decent~(OMD) can be applied to address the limit cycling problem in training Wasserstein GANs~(WGANs). Also, the recent paper \cite{park2024llm} studies the \emph{no-regret} behaviors of large language model~(LLM) agents. This reference proposes an unsupervised training loss, whose minimization could automatically result in known no-regret learning algorithms. On the other hand, our focus is on the evaluation of generative models which does not concern the models' training.

\section{\MakeUppercase{PRELIMINARIES}}
\label{sec:pre}
\subsection{Inception Score}

Inception score (IS) is a standard metric for evaluating generative models, defined as
\begin{equation}
    \textup{IS}:=\exp\left\{\E_{X_\textup{g}\sim p_\textup{g}}[\textup{KL}(p_{Y|X_\textup{g}}\| p_{Y_\textup{g}})]\right\},
\end{equation}
where $X_\textup{g}\sim p_\textup{g}$ is a generated image, $p_{Y|X_\textup{g}}$ is the conditional class distribution assigned by the InceptionNet.V3 pretrained on ImageNet~\citep{7780677}, and $p_{Y_\textup{g}}=\E_{X_\textup{g}\sim p_\textup{g}}[p_{Y|X_\textup{g}}]$ is the marginal class distribution. Further, we have that $\textup{IS}=\exp[I(X_\textup{g};Y_\textup{g})]=\exp[H(Y_\textup{g})-H(Y_\textup{g}|X_\textup{g})]$, where $I(\cdot;\cdot)$ is the mutual information, and $H(\cdot)$ is the Shannon entropy. A higher IS implies that the images generated by the model have higher diversity, since $p_{Y_\textup{g}}$ would be more uniformly distributed to increase $H(Y_\textup{g})$, and possess higher fidelity, because $p_{Y|X_\textup{g}}$ is closer to a one-hot vector to enforce a smaller $H(Y_\textup{g}|X_\textup{g})$. 

\subsection{Fréchet Distance}

Fréchet Distance (FD)~\citep{DOWSON1982450} is another standard metric for evaluating generative models. Let $f(x)\in\sR^d$ denote the $d$-dimensional feature of an image $x$. The FD between the feature distributions of the generated images $p_{f(X_\textup{g})}$ and the real images $p_{f(X_\textup{r})}$, which is computed by
\begin{equation}
    \textup{FD}=\|\mu_\textup{g}-\mu_\textup{r}\|_2^2+\textup{Tr}[\Sigma_\textup{g}+\Sigma_\textup{r}-2(\Sigma_\textup{g}\Sigma_\textup{r})^\frac{1}{2}].
\end{equation}
The well-known FID metric~\citep{NIPS2017_8a1d6947} is computed as the FD where the image data feature is extracted by InceptionNet.V3~\citep{7780677}. Recently, \cite{kynknniemi2023the,NEURIPS2023_0bc795af} propose computing the FD distance using the CLIP \citep{pmlr-v139-radford21a} and DINOv2 \citep{oquab2024dinov2} embeddings, respectively, to boost the score's ranking consistency with human evaluations.

\section{\MakeUppercase{ONLINE EVALUATION OF GENERATIVE MODELS}}
\newcounter{protocol}
\makeatletter
\newenvironment{protocol}[1][htb]{%
  \let\c@algorithm\c@protocol
  \renewcommand{\ALG@name}{Protocol}
  \begin{algorithm}[#1]%
  }{\end{algorithm}
}
\makeatother

In this section, we introduce the framework of online evaluation of generative models, which is given in Protocol~\ref{protocol}. We denote by $\gG:=[G]$ the set of generative models. For each generator $g\in[G]$, we denote by $p_g\in\Delta(\gX)$ its generative distribution over the space $\gX$, which can be texts or images. Given an evaluation metric, e.g., FD or IS, the corresponding \emph{score} of the generator $g\in[G]$ is denoted by $\nu_g$. The evaluation proceeds in $T$ steps. At each step $t\in[T]$, the evaluating algorithm $\gA$ picks a generator $g_t\in[G]$ and collects a batch of generated samples $\{x_t^j\sim p_g\}_{j=1}^b$, where $b\in\sN_+$ is the (fixed) batch size. The evaluator aims to minimize the \emph{regret}
\begin{equation}
\label{def_reg}
    \textup{Regret}(T)=\sum_{t=1}^T(\nu^*-\nu_{g_t}),
\end{equation}
where $\nu^*:=\argmax_{g\in[G]}\nu_g$~(if the higher the score the better). 

Regarding the challenges of online evaluation of generative models, note that the empirical estimation of the score could be biased and generator-dependent. In addition, the generative distribution typically lies in a high-dimensional space, which makes it difficult to estimate the score from limited data. Moreover, the evaluation metrics often incorporate higher-order information of the generative distribution~(e.g., computing FD requires the covariance matrix). Hence, analyzing the concentration properties of the score estimation would be challenging.

\begin{protocol}[t]
\caption{Online Evaluation and Selection of Generative Models}
\begin{algorithmic}[1]
\REQUIRE step $T\in\sN_+$, evaluation metric, a set $\gG\leftarrow[G]$ of generators for evaluation, evaluator $\gA$, batch size $b\in\sN_+$\vspace{1mm}
\STATE \textbf{Initialize:} estimated score $\widehat{\nu}_g$ and generated samples $\gH_g\leftarrow\varnothing$ for each $g\in[G]$. \vspace{1mm}
\FOR{step $t=1,2,\cdots,T$ \vspace{1mm}}
    \STATE Pick generator $g_t\sim\gA$.\vspace{1mm}
    \STATE Generate a batch $\{x_t^j\sim p_{g_t}\}_{j=1}^b$ of samples from $g_t$.\vspace{1mm}
    \STATE Update $\gH_{g_t}\leftarrow\gH_{g_t}\cup\{x_t^j\}_{j=1}^b$ and estimated score $\widehat{\nu}_{g_t}$.\vspace{1mm}
\ENDFOR
\end{algorithmic}
\label{protocol}
\end{protocol}

\noindent\textbf{Use of regret metric.}\quad The choice of regret metric~\ref{def_reg} follows the standard online learning literature for bandit problems~\citep{bubeck2012regret}. Particularly, if the evaluator can attain a \textit{sub-linear} regret, then the overall selection converges to the best model with high probability. Additionally, in the online evaluation task, it is possible that the best model has a very similar performance to that of the second-best model. In such a case, the regret would take the difference between the models into account, and in a case where the top-$k$ models have similar scores, generating samples from them would weighted more equally. On the other hand, if the best two models have significantly different scores, the regret would sharply grow by picking the suboptimal models.

\section{\MakeUppercase{FRÉCHET DISTANCE-BASED ONLINE EVALUATION AND SELECTION}}
\label{sec:5}
\begin{figure*}[!ht]
\centering
    \makebox[25pt][r]{\makebox[20pt]{\raisebox{50pt}{\rotatebox[origin=c]{90}{\textbf{Avg. Regret}}}}}
    \subfigure{\includegraphics[width=0.3\textwidth]{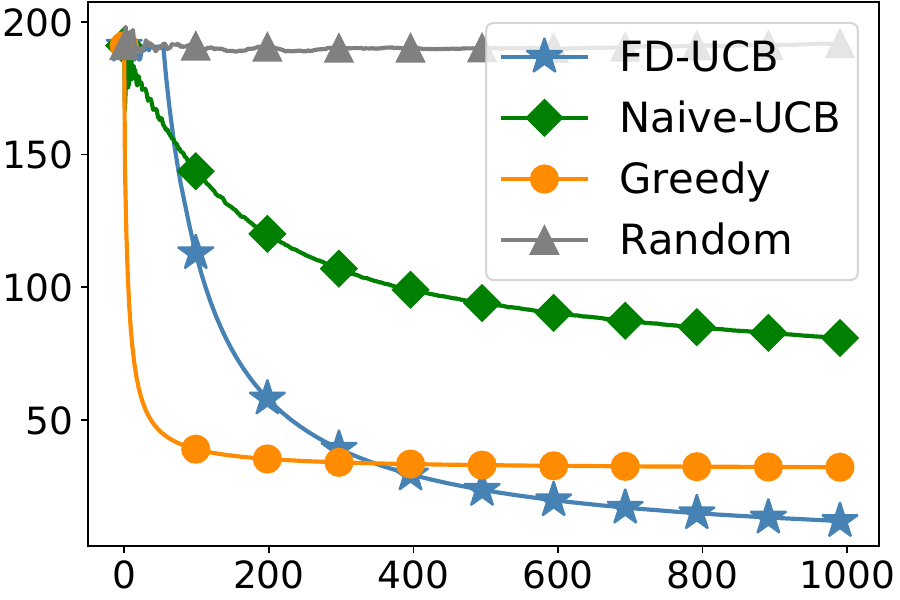}}
    \hfill
    \subfigure{\includegraphics[width=0.3\textwidth]{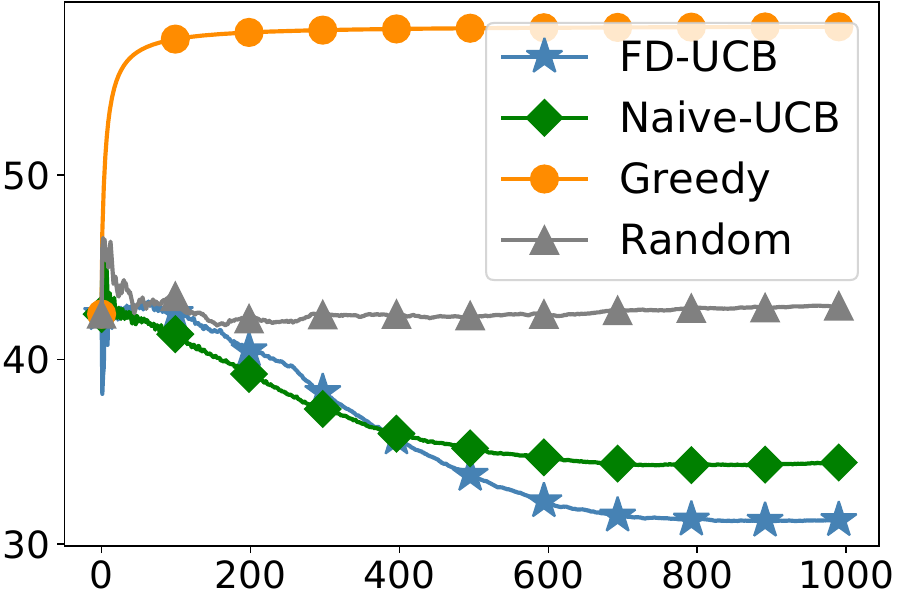}}
    \hfill
    \subfigure{\includegraphics[width=0.3\textwidth]{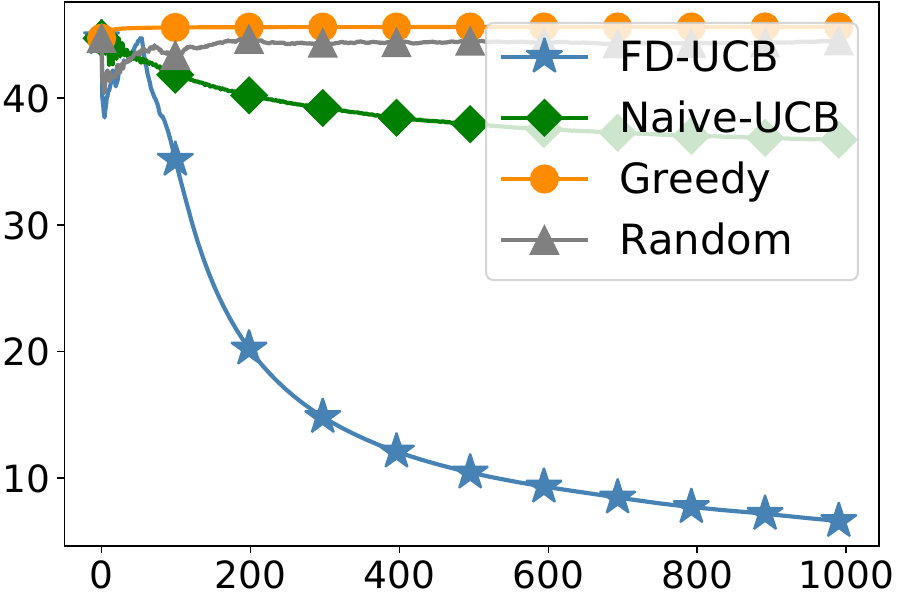}} \\
    
    \makebox[25pt][r]{\makebox[20pt]{\raisebox{50pt}{\rotatebox[origin=c]{90}{\textbf{OPR}}}}}
    \stackunder{\subfigure{\includegraphics[width=0.3\textwidth]{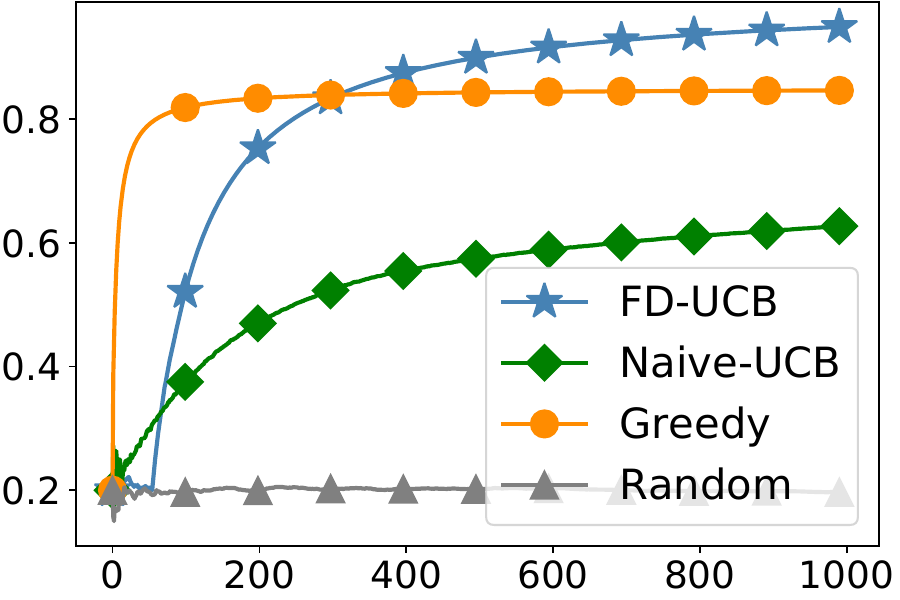}}}{\textbf{CIFAR10}}
    \hfill
    \stackunder{\subfigure{\includegraphics[width=0.3\textwidth]{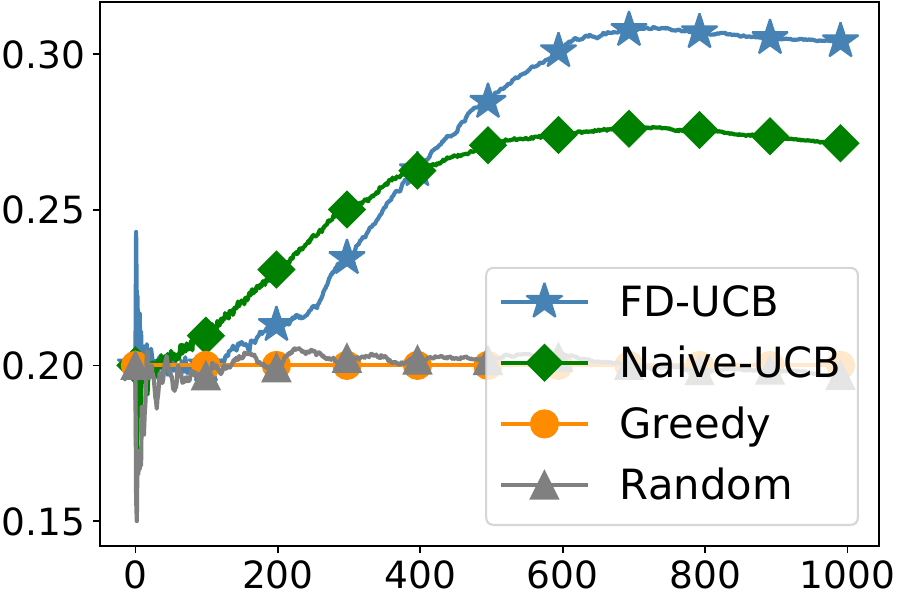}}}{\textbf{ImageNet}}
    \hfill
    \stackunder{\subfigure{\includegraphics[width=0.3\textwidth]{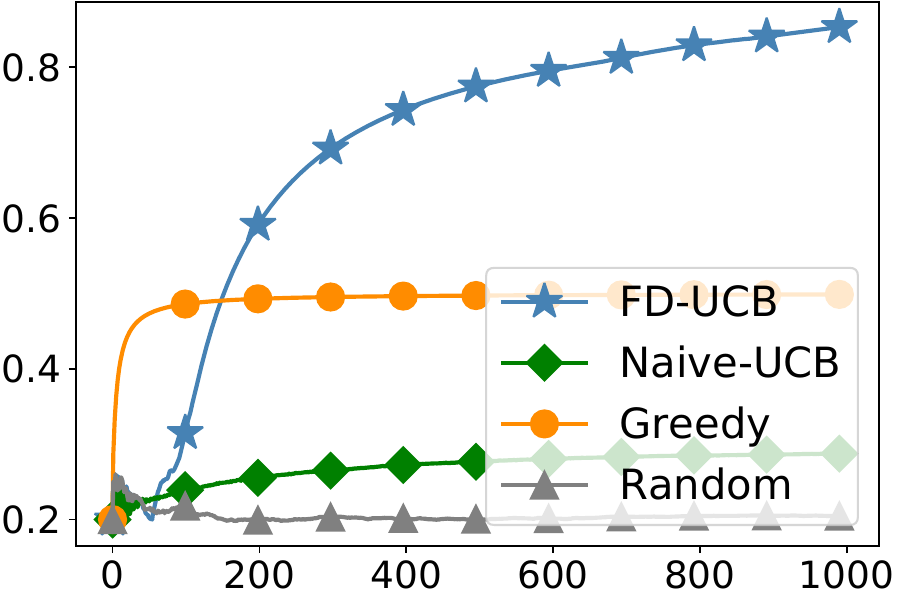}}}{\textbf{FFHQ}}
\caption{Online FD-based evaluation and selection among standard generative models: The $x$-axis is the number of online steps. At each step, the algorithm samples a batch of five generated images from the chosen model. The image data embeddings are extracted by CLIP~\citep{cherti2023reproducible}. Results are averaged over 20 trials.}
\label{fd-results}
\end{figure*}

In this section, we consider the online evaluation of generative models by Fréchet Distance (FD). Given $n\in\sN_+$ generated images $x^1,\cdots,x^n\sim p_g$ queried from model $g$, the empirical FD is computed by
\begin{equation}
\label{emp_fid}
    \widetilde{\textup{FD}}_g^n=\|\widehat{\mu}_g^n-\mu_\textup{r}\|_2^2+\textup{Tr}[\widehat{\Sigma}_g^n+\Sigma_\textup{r}-2(\widehat{\Sigma}_g^n\Sigma_\textup{r})^\frac{1}{2}],
\end{equation}
where
\begin{equation*}
    \widehat{\mu}_g^n=\frac{1}{n}\sum_{i=1}^n f(x^i),\ \widehat{\Sigma}_g^{n}=\frac{1}{n}\sum_{i=1}^n(f(x^i)-\widehat{\mu}_g^n)(f(x^i)-\widehat{\mu}_g^n)^\top
\end{equation*}
are the mean vector and the empirical covariance matrix, respectively, and $f(x^i)\in\sR^d$ is the feature of the $i$-th generated image which is extracted by, e.g., the InceptionNet.V3. 
The FD-based evaluation is typically performed on a large batch of generated samples~(10-50 thousands) to reduce the estimation variance, which can be sample-inefficient and costly.

To enable sample-efficient online evaluation of generative models, we adapt the optimism-based online learning framework to the FD score. To this end, we first derive an optimistic FD score in the following theorem. We defer the theorem's proof to Appendix~\ref{p_thm_ofid}.

\begin{theorem}[Optimistic FD score]
\label{thm_ofid}
Assume for any generator $g$, the (random) embedding $f(X_g) \sim \gN(\mu_g,\Sigma_g)$ follows a multivariate Gaussian, and the covariance matrix $\Sigma_\textup{r}$ of the real data is positive definite. Then, with probability at least $1-\delta$, we have
\begin{equation}
\label{ofid}
    \widehat{\textup{FD}}_g^n=\widetilde{\textup{FD}}_g^n-\gB^n_{g}\le\textup{FD}_g
\end{equation}
for any $n\ge 4\bm{r}(\Sigma_g)+\log(3/\delta)$, where $\bm{r}(\Sigma_g):=\frac{\tr[\Sigma_g]}{\|\Sigma_g\|_2}$ is the effective rank of the covariance matrix $\Sigma_g$. In addition, the bonus is given by
\begin{equation}
\label{fid-b}
\begin{aligned}
    \gB_g^n:=&2\Delta_{\mu_g}^n\cdot\left(\Delta_{\mu_g}^n+\|\widehat{\mu}_g^n-\mur\|_2\right)+\tr[\sigr^\frac{1}{2}]\sqrt{8\Delta_{\Sigma_g}^n}\\
    &+\tr[\Sigma_g]\sqrt{\frac{8}{n}\log\left(\frac{6}{\delta}\right)}+\frac{8\|\Sigma_g\|_2}{n}\log\left(\frac{6}{\delta}\right),
\end{aligned}
\end{equation}
where $\Delta_{\mu_g}^n$ and $\Delta_{\Sigma_g}^n$ are defined in (\ref{eq52}) and (\ref{eq57}), respectively.
\end{theorem}

\begin{figure*}[!ht]
\centering
    \makebox[25pt][r]{\makebox[20pt]{\raisebox{50pt}{\rotatebox[origin=c]{90}{\textbf{Avg. Regret}}}}}
    \subfigure{\includegraphics[width=0.3\textwidth]{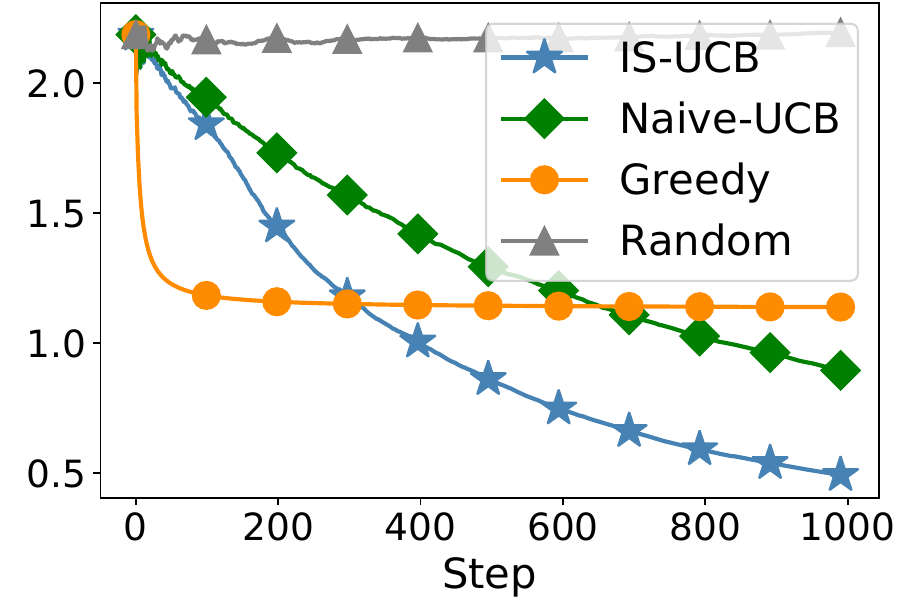}}
    \hfill
    \subfigure{\includegraphics[width=0.3\textwidth]{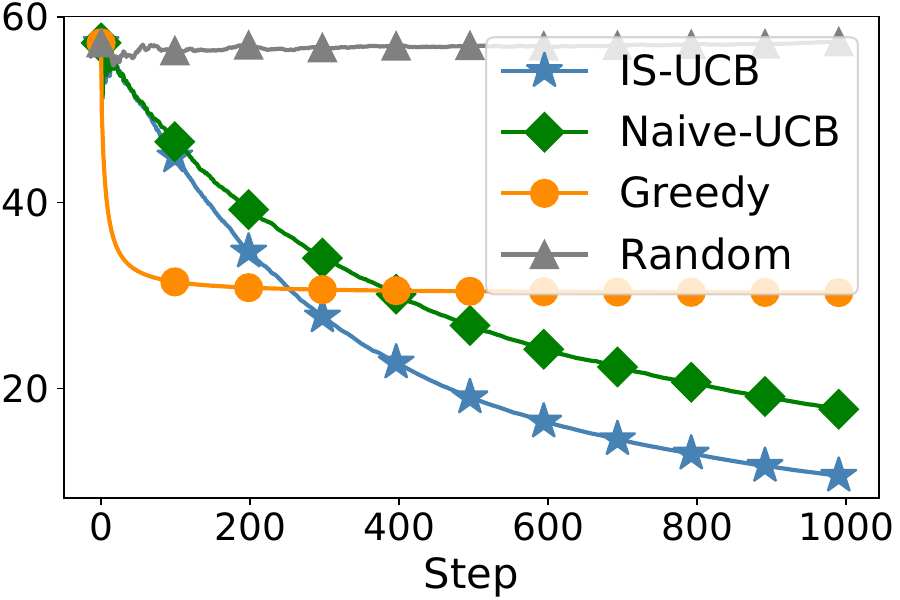}}
    \hfill
    \subfigure{\includegraphics[width=0.3\textwidth]{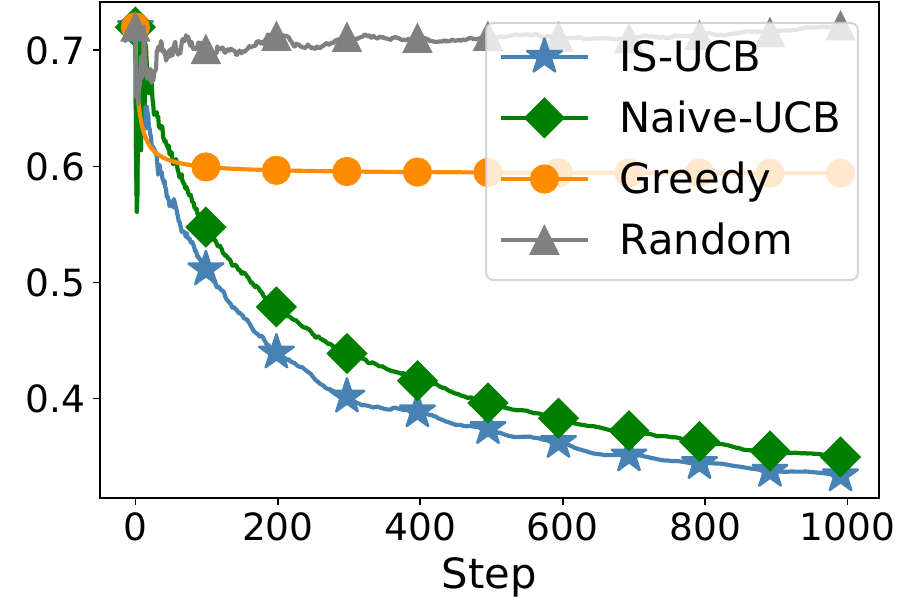}} \\
    
    \makebox[25pt][r]{\makebox[20pt]{\raisebox{50pt}{\rotatebox[origin=c]{90}{\textbf{OPR}}}}}
    \stackunder{\subfigure{\includegraphics[width=0.3\textwidth]{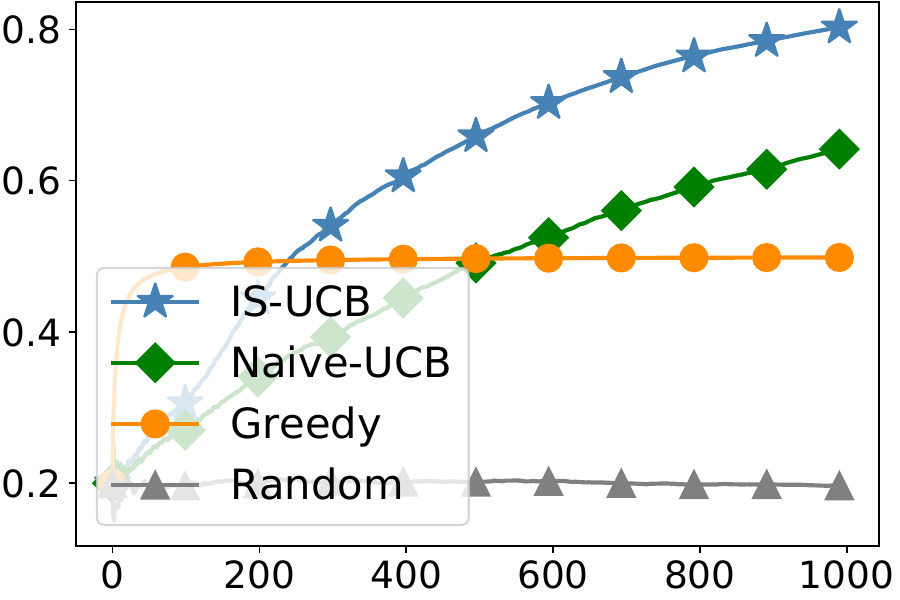}}}{\textbf{CIFAR10}}
    \hfill
    \stackunder{\subfigure{\includegraphics[width=0.3\textwidth]{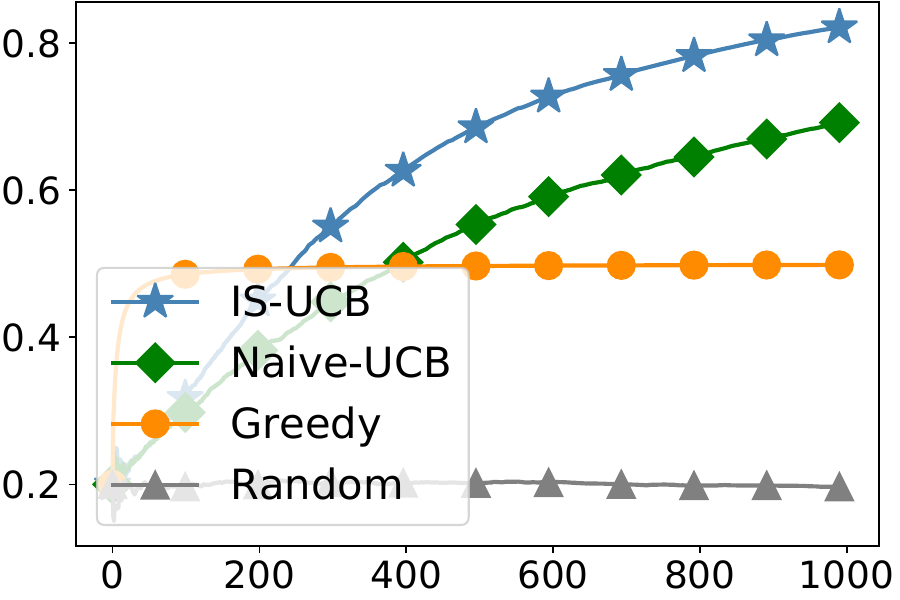}}}{\textbf{ImageNet}}
    \hfill
    \stackunder{\subfigure{\includegraphics[width=0.3\textwidth]{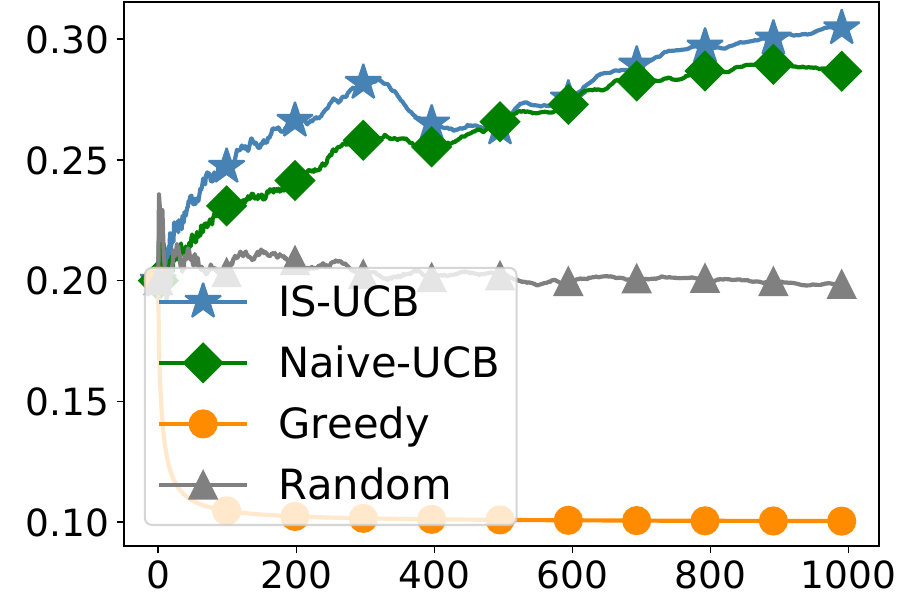}}}{\textbf{FFHQ}}
\caption{Online IS-based evaluation and selection among standard generative models: The $x$-axis is the number of steps. At each step, the algorithm samples a batch of five generated images from the chosen model. Results are averaged over 20 trials.}
\label{is-results}
\end{figure*}

\begin{remark}[Model-dependent parameters]
The bonus term~(\ref{fid-b}) involves model-dependent parameters $\tr[\Sigma_g]$ and $\|\Sigma_g\|_2$ of the covariance matrix. Note that the presence of the norms of $\Sigma_g$ in the concentration bound for FD is inevitable, because, without any assumptions on the norm of $\Sigma_g$, the concentration bound cannot have a finite value. To address this, two approaches can be considered: 1) when the norm of the data embeddings is bounded, e.g., the normalized CLIP embeddings where $\|X_g\|_2\le 1$, we can substitute $\tr[\Sigma_g]$ and $\|\Sigma_g\|_2$ in the above bound with the embedding upper-bound norm. In this case, we can extend our analysis beyond multivariate Gaussians~(Theorem~\ref{thm-bounded-norm} in the Appendix). 2) On the other hand, when using the InceptionV3 and DINOv2 embeddings that have generally unbounded output norm, we can use the already generated data to estimate the parameters in the upper-bound. This approach avoids any offline estimation of the parameters required to have an upper-bound and preserves the online format of the FD-UCB algorithm, and conditioned that the estimated norms converge fast to the underlying value, will provide a correct solution. Our numerical results indicate that with only 50 samples, the norm terms in the bound will be relatively close to their actual values. See Part~3 in Appendix~\ref{sec:aed} for our experimental validation.
\end{remark}

\begin{algorithm}[!ht]
\caption{FD-UCB}
\label{alg-fid-ucb}
\begin{algorithmic}[1]
\REQUIRE set of models $\gG$, image embedding $f$, step $T \in \sN_+$, $\mu_\textup{r} \in \sR^d$ and $\Sigma_\textup{r} \in \sS^d$ of real data, failure probability $\delta\leftarrow\delta/T$, batch size $b \in \sN_+$, number of burn-in samples $N\ge\max_{g\in\gG} \{4\vr(\Sigma_g) + \log(3/\delta)\}$.
\STATE \textbf{Initialize:} $\widehat{\textup{FD}}_g\leftarrow-\infty,\gH_g\leftarrow\varnothing,n_g\leftarrow0$ for each $g\in\gG$.
\FOR{model $g \in \gG$}
    \STATE Generate $M$ samples $\{x_0^j \sim p_g \}_{j=1}^N$. \\
    \STATE Update $\gH_{g}\leftarrow\gH_{g}\cup\{f(x_0^j)\}_{j=1}^N$ and visitation $n_{g}\leftarrow N$.
\ENDFOR
\FOR{step $t=1,2,\cdots,T$}
    \STATE Pick model $g_t\leftarrow\argmin_{g\in\gG}\widehat{\textup{FD}}_g$.
    \STATE Model $g_t$ generates batch samples $\{x_t^j\sim g_t\}_{j=1}^b$.
    \STATE Update $\gH_{g_t}\leftarrow\gH_{g_t}\cup\{f(x_t^j)\}_{j=1}^b$ and visitation $n_{g_t}\leftarrow n_{g_t}+b$.
    \STATE Update $\widehat{\mu}_{g_t}$ and $\widehat{\Sigma}_{g_t}$.
    \STATE Compute the optimistic FD score
    \begin{equation*}
    \begin{aligned}
        \widehat{\textup{FD}}_{g_t}\leftarrow&-\gB_{g_t}^{n_{g_t}}+\|\widehat{\mu}_{g_t}-\mu_\textup{r}\|_2^2\\
        &+\textup{Tr}\left[\widehat{\Sigma}_{g_t}+\Sigma_\textup{r}-2(\widehat{\Sigma}_{g_t}\Sigma_\textup{r})^\frac{1}{2}\right],
    \end{aligned}
    \end{equation*}
    where $\gB^n_{gt}$ is given by Equation~(\ref{fid-b}).
\ENDFOR
\end{algorithmic}
\end{algorithm}
Based on Theorem~\ref{thm_ofid}, we propose FD-UCB in Algorithm~\ref{alg-fid-ucb} as an FD-based online evaluation algorithm. At the beginning, FD-UCB samples $N$ generated data from each model, after which the algorithm can compute an optimist FD score~(\ref{ofid}) for each model~(lines \num{2}-\num{5}). The evaluation proceeds iteratively, where at each iteration $t\in[T]$, the evaluator picks model $g_t$ with the lowest estimated FD and queries a batch of images from the model~(lines \num{7}-\num{8}). Then, the estimated FD of generator $g_t$ is updated~(line \num{11}). It can be shown that FD-UCB attains sub-linear regret bound, which is formalized in Theorem~\ref{thm_fid_reg} in Appendix~\ref{p_thm_fid_reg}.

\section{\MakeUppercase{INCEPTION SCORE-BASED ONLINE EVALUATION AND SELECTION}}
\label{sec:6}
In this section, we focus on evaluating generative models by Inception score (IS), which is given by $\textup{IS}=\exp[I(X_g;Y_g)]$, where $X_g\sim p_g$ is the generated image, and $Y_g$ is the class assigned by the InceptionNet.V3. Given $n\in\sN_+$ generated images $x^1,\cdots,x^n\sim p_g$ queried from a generator $g$, we denote by 
\begin{gather}
    \widehat{H}^n(Y_g)=H\left(\frac{1}{n}\sum_{i=1}^n p_{Y|x^i}\right),\\
    \widehat{H}^n(Y_g|X_g)=\frac{1}{n}\sum_{i=1}^n H(Y|x^i)\label{eq_11},
\end{gather}
the empirical entropy of the marginal $d$-class distribution $p_{Y_g}$ and the conditional $d$-class distribution $p_{Y|X_g}$, respectively. Then, the empirical IS is computed by
\begin{equation}
\label{emp_is}
    \widetilde{\textup{IS}}_g^n=\exp\left\{\widehat{H}^n(Y_g)-\widehat{H}^n(Y_g|X_g)\right\}.
\end{equation}
For any $j\in[d]$, let $\widehat{V}^n(p_{Y|X_g}[j]):=\sV(p_{Y|x^1}[j],\cdots,p_{Y|x^n}[j])$ denote the sample variance for the probability that the generated sample is assigned to the $j$-th class. To derive an optimistic IS, we first define the \emph{optimistic marginal class distribution} denoted by
\begin{equation}
\label{eq_15}
    \widehat{p}_{Y_g}^n:=\textup{Clip}_{e^{-1}}(\widetilde{p}^n_{Y_g},\bm{\epsilon}^n_g)
\end{equation}
where $\bm{\epsilon}^n_g\in\sR_+^d$ is a $d$-dimensional vector whose $j$-th element is given by
\begin{equation}
\begin{aligned}
\label{eq_22}
    \bm{\epsilon}^n_g[j]:=&\sqrt{\frac{2\widehat{V}^n(p_{Y|X_g}[j])}{n}\log\left(\frac{4d}{\delta}\right)}\\
    &+\frac{7}{3(n-1)}\log\left(\frac{4d}{\delta}\right),
\end{aligned}
\end{equation}
and $\textup{Clip}_{e^{-1}}(p,\bm{\epsilon})$ is the following element-wise operator
\begin{equation}
\begin{cases}
    p[j]+\frac{e^{-1}-p[j]}{|e^{-1}-p[j]|}\bm{\epsilon}[j] & \textup{, if }|e^{-1}-p[j]|\ge \bm{\epsilon}[j]\\
    e^{-1} & \textup{, otherwise.}
\end{cases}
\end{equation}

Particularly, the optimistic marginal class distribution ensures that $\gE(\widehat{p}_{Y_g}^n)\ge H(Y_g)$ with high probability~(see Lemma~\ref{thm_omcd} in the Appendix), where we define
$$
\gE(z):=-\sum_{j}z[j]\cdot\log(z[j])
$$
for any $d$-dimensional vector $z\succ\bm{0}$. Next, we derive a generator-dependent optimistic IS in the following theorem. We defer the theorem's proof to Appendix~\ref{p_thm_ois}.

\begin{theorem}[Optimistic IS]
\label{thm_ois}
Let
\begin{equation}
\label{eq_29}
\begin{aligned}
    \widehat{\textup{IS}}^n_g:=\exp\Bigg\{&\gE\left(\widehat{p}_{Y_g}^n\right)-\widehat{H}^n(Y_g|X_g)\\
    &+\sqrt{\frac{2\widehat{V}^n(H(Y_g|X_g))}{n}\log\left(\frac{4d}{\delta}\right)}\\
    &+\frac{7\log d}{3(n-1)}\log\left(\frac{4d}{\delta}\right)\Bigg\},
\end{aligned}
\end{equation}
where $\widehat{H}^n(Y_g|X_g)$ is defined in Equation~(\ref{eq_11}), $\widehat{p}_{Y_g}^n$ the optimistic marginal class distribution~(\ref{eq_15}), and $\widehat{V}^n(H(Y_g|X_g))=\sV(H(Y|x^1),\cdots,H(Y|x^n))$ is the empirical variance for the entropy of the conditional class distribution. Then, with probability at least $1-\delta$, we have that $\widehat{\textup{IS}}^n_g\ge\textup{IS}_g$.
\end{theorem} 

\begin{figure*}[!ht]
\centering
    \subfigure{\includegraphics[width=0.38\textwidth]{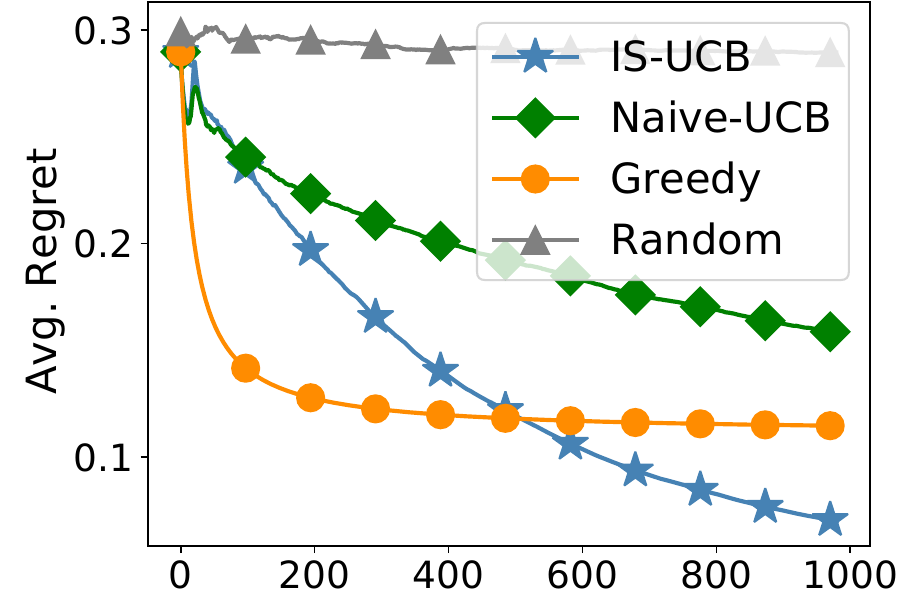}}
    \qquad
    \subfigure{\includegraphics[width=0.38\textwidth]{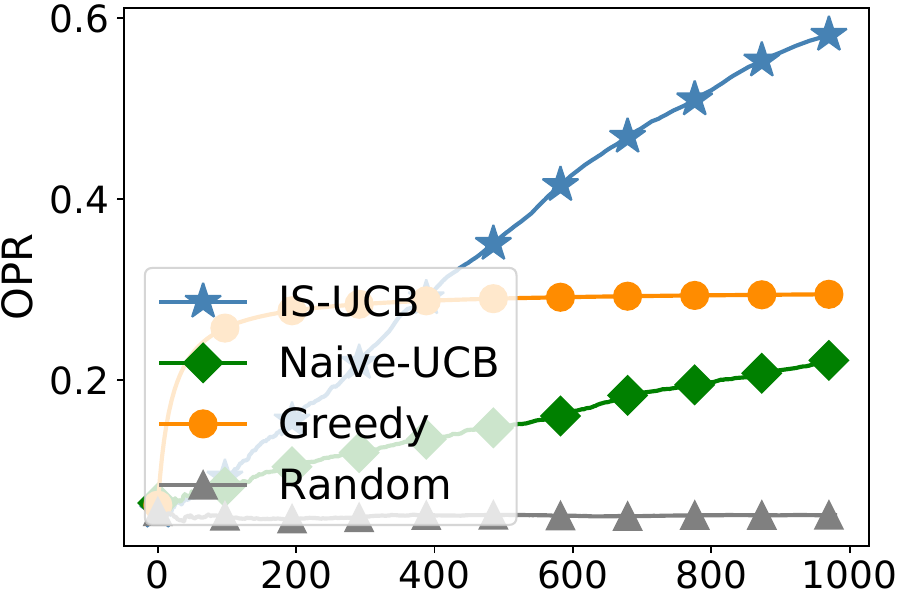}} \\
    \makebox[0pt][r]{\makebox[30pt]{\raisebox{50pt}{\rotatebox[origin=c]{90}{\textbf{Sampled Data}}}}}
    \stackunder{\subfigure{\includegraphics[width=0.17\textwidth]{new_results/summary_1000_UCB_b_IS.pdf}}}{\textbf{IS-UCB}}
    \qquad
    \stackunder{\subfigure{\includegraphics[width=0.17\textwidth]{new_results/summary_1000_UCB_h_IS.pdf}}}{\textbf{Naive-UCB}}
    \qquad
    \stackunder{\subfigure{\includegraphics[width=0.17\textwidth]{new_results/summary_1000_naive_none_IS.pdf}}}{\textbf{Greedy}}
\caption{Online IS-based evaluation and selection among variance-controlled FFHQ models: IS-UCB can identify models that generate images with more diversity. Results are averaged over 20 trials.}
\label{is-results-vc}
\end{figure*}

\begin{algorithm}[!ht]
\begin{algorithmic}[1]
\caption{IS-UCB}
\label{alg-is-ucb}
\REQUIRE steps $T$, set of generative models $\gG\leftarrow[G]$, batch size $b\in\sN_+$, failure probability $\delta\leftarrow\delta/T$.
\STATE \textbf{Initialize:} $\widehat{\textup{IS}}_g\leftarrow\infty,\widehat{H}(Y_g|X_g)\leftarrow 0,\widetilde{p}_{Y_g}\leftarrow\bm{0}^d,n_g\leftarrow0$ for each $g\in[G]$.
\FOR{step $t=1,2,\cdots,T$}
    \STATE Pick model $g_t\leftarrow\argmax_{g\in\gG}\widehat{\textup{IS}}_g$.
    \STATE Model $g_t$ generates batch samples $\{x_t^j\sim p_{g_t}\}_{j=1}^b$.
    \STATE Update
    \begin{gather*}
        \widetilde{p}_{Y_{g_t}}\leftarrow\frac{N_g\cdot\widetilde{p}_{Y_g}+\sum_{j=1}^b p_{Y|x_t^i}}{n_{g_t}+b},\\
        \widehat{H}(Y_{g_t}|X_{g_t})\leftarrow \frac{N_g\cdot\widehat{H}(Y_{g_t}|X_{g_t})+\sum_{j=1}^bH(Y|x_t^i)}{n_{g_t}+b}.
    \end{gather*}
    \STATE Update visitation $N_{g_t}\leftarrow N_{g_t}+b$.
    \STATE Update the empirical variances $\widehat{V}(H(Y_{g_t}|X_{g_t}))$ and $\widehat{V}(p_{Y|X_{g_t}}[j])$ for any $j\in[d]$.
    \STATE Compute the estimated IS according to Equation~(\ref{eq_29}).
\ENDFOR
\end{algorithmic}
\end{algorithm}

Based on Theorem~\ref{thm_ois}, we propose IS-UCB in Algorithm~\ref{alg-is-ucb}, an optimism-based online IS evaluation algorithm. Due to limited space in the main text, we derive the regret bound of the IS-UCB algorithm in Appendix~\ref{p_thm_is_reg}, which shows that IS-UCB attains $\widetilde{O}(\sqrt{T})$ regret.

\section{\MakeUppercase{EXPERIMENTAL RESULTS}}
\label{sec:7}
In this section, we present numerical results for the FD-UCB and IS-UCB algorithms. Our extensive experiments across standard image datasets, generative modeling frameworks, and image data embeddings demonstrate the effectiveness of the proposed online learning-based approaches to model evaluation and selection. Experimental details and additional results can be found in Appendix~\ref{sec:aed}.

\textbf{Baselines.}\quad For both FD-based and IS-based evaluation, we compare the performance of our proposed \textcolor{mplblue}{FD-UCB} and \textcolor{mplblue}{IS-UCB} with three baselines: \textcolor{darkorange}{Naive-UCB}, \textcolor{ForestGreen}{Greedy algorithm}, and \textcolor{darkgray}{Random algorithm}. Naive-UCB is a simplification of FD-UCB and IS-UCB which replaces the generator-dependent variables in the bonus function with data-independent and dimension-based terms, which does not exploit the generated data in evaluating the confidence bound. On the other hand, the Greedy algorithm always picks the generator with the lowest empirical FD~(\ref{emp_fid}) or the highest empirical IS~(\ref{emp_is}), and the Random algorithm selects the model randomly. For a fair comparison, we set the burn-in samples of FD-UCB to be zero. Additionally, for all these algorithms, each generator will be explored once at the beginning. 

\noindent\textbf{Experimental setup and performance metrics.}\quad We report two performance metrics at each step $t\in[T]$: 1) \emph{average regret}~(Avg. Regret), i.e., $(1/t)\cdot\textup{Regret}(t)$, and 2) \emph{optimal pick ratio}~(OPR), i.e., the ratio $(1/t)\cdot\sum_{i=1}^t\bm{1}(g_t=g^*)$ of picking the optimal model, which has the lowest empirical FD score or the highest empirical IS score for 50k generated images. For all the experiments, we use a batch size of 5, and the total evaluation step is $T=1,000$. Hence the total generated samples for each trial is $5$k. Results are averaged over 20 trials.

\noindent\textbf{Datasets and generators.}\quad We evaluate the above algorithms on standard image datasets, including CIFAR10~\citep{krizhevsky2009learning}, ImageNet~\citep{imagenet_cvpr09}, FFHQ~\citep{6909637}, and AFHQ~\citep{Choi_2020_CVPR}. For the first three datasets, we consider evaluation and selection among standard generative models, such as BigGAN-Deep~\citep{brock2018large}, iDDPM-DDIM~\citep{pmlr-v139-nichol21a}, and Efficient-vdVAE~\citep{hazami2022efficientvdvae}. Details of the chosen models can be found in Appendix~\ref{sec:aed}. Additionally, we consider variance-controlled generators for the FFHQ and AFHQ datasets, where we apply the standard truncation technique~\citep{NEURIPS2019_0234c510} to the pretrained StyleGAN2-ADA model~\citep{NEURIPS2020_8d30aa96}\footnote{The repository can be found at \url{github.com/NVlabs/stylegan2-ada-pytorch}} and synthesize $G=20$ models. Each model generates images from truncated random noises with parameters vary from \num{0.01} to \num{0.1}. A small~(large) truncation parameter can lead to generated samples with low~(high) diversity but high~(low) quality.

\subsection{Results of Online FD-based Evaluation and Selection}

\textbf{FID-based evaluation and selection.}\quad We first report results for the setup in Figure~\ref{fig1}, where we select among three standard generative models on the CIFAR10 dataset, including LOGAN~\citep{wu2020loganlatentoptimisationgenerative}, RESFLOW~\citep{NEURIPS2019_5d0d5594}, and iDDPM-DDiM~\citep{pmlr-v139-nichol21a}. Our proposed FD-UCB can quickly identify the best model with fewer queries to the suboptimal models~(Figure~\ref{fd-results-fig1}).

\begin{figure}[!ht]
\centering
    \subfigure{\includegraphics[width=0.35\textwidth]{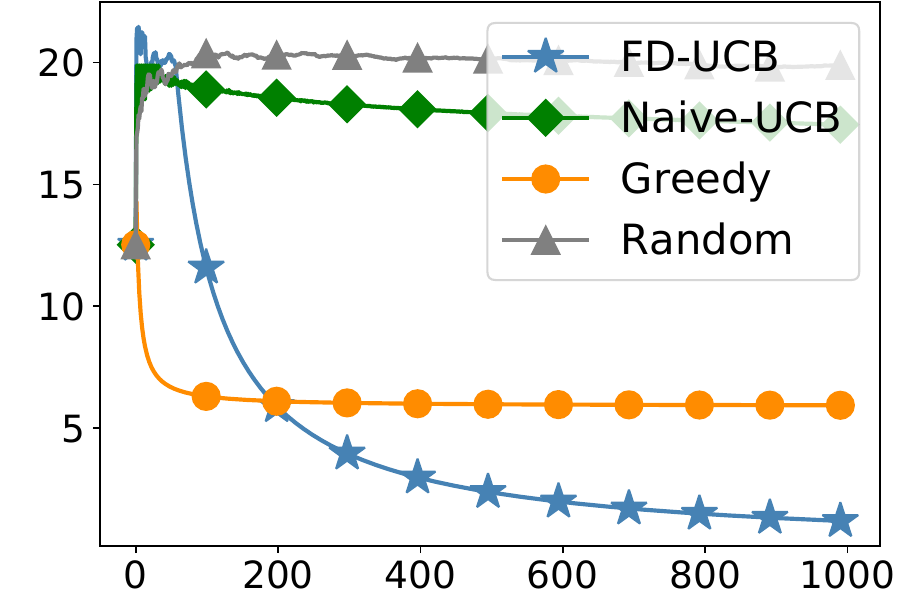}}
    \subfigure{\includegraphics[width=0.35\textwidth]{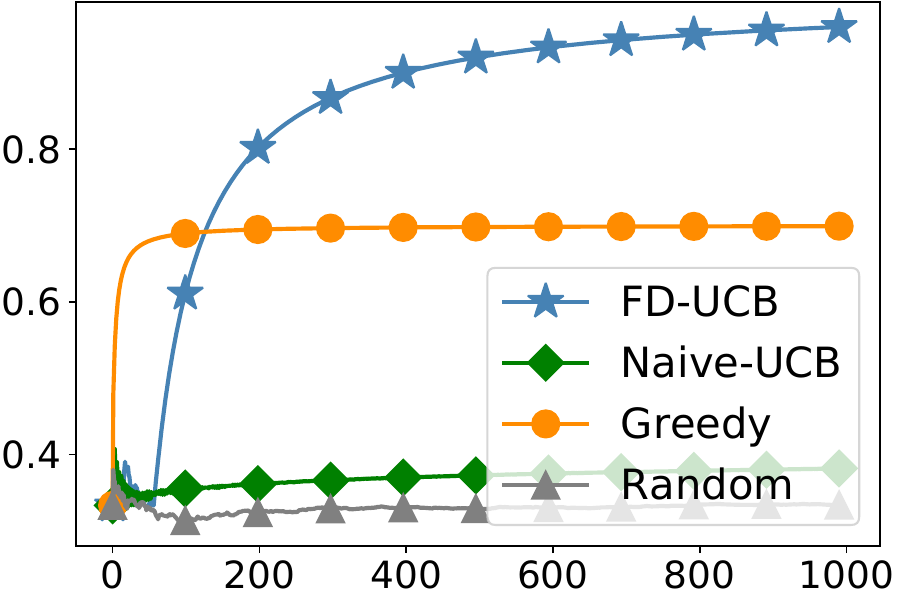}}
\caption{Online FD-based evaluation and selection among three CIFAR10 models, including LOGAN, RESFLOW, and iDDPM-DDiM~(Figure~\ref{fig1}). The image data embeddings are extracted by InceptionV3.Net. Results are averaged over 20 trials.}
\label{fd-results-fig1}
\end{figure}

\textbf{Performance on different embeddings.}\quad The results for CLIP embeddings are summarized in Figure~\ref{fd-results}, where each column and each row correspond to one dataset and one performance metric, respectively. The results show that FD-UCB outperforms Naive-UCB and the Greedy algorithm. We observe that naive-UCB converges to the best model much slower than FD-UCB, which suggests that the generator-dependent and data-driven bonus function can better exploit the properties of the generator, which is key to the sample-efficient online evaluation. We also test the InceptionNet.V3 and DINOv2 embeddings, and the results are summarized in Figures~\ref{fd-inception-results} and~\ref{fd-dinov2-results} in the Appendix. FD-UCB can maintain satisfactory performance over different embeddings.

\noindent\textbf{Performance on variance-controlled generators.}\quad We summarize the results on the FFHQ and AFHQ-Dog datasets in Figures~\ref{fd-results-vc-inception},~\ref{fd-results-vc-dinov2}, and~\ref{fd-results-vc-clip} in the Appendix. FD-UCB consistently outperformed the baselines.

\noindent\textbf{Performance on video and audio data.}\quad We consider the \textit{Fréchet Video Distance}~\cite[FVD]{unterthiner2019fvd} and \textit{Fréchet Audio Distance}~\cite[FAD]{kilgour2019frechetaudiodistancemetric} metrics for video and audio generation, respectively. We synthesize four arms~(generators) utilizing the MSR-VTT~\citep{7780940} and Magnatagatune~\citep{Law09} datasets. When an arm is selected, a video/audio clip is sampled from the dataset and perturbed by Gaussian noises with an arm-specific probability. Results on the audio data are summarized in Figure~\ref{fd-results-audio}. Additional results can be found in Figure~\ref{fd-results-sync} in the Appendix.

\begin{figure}[!t]
\centering
    \subfigure{\includegraphics[width=0.35\textwidth]{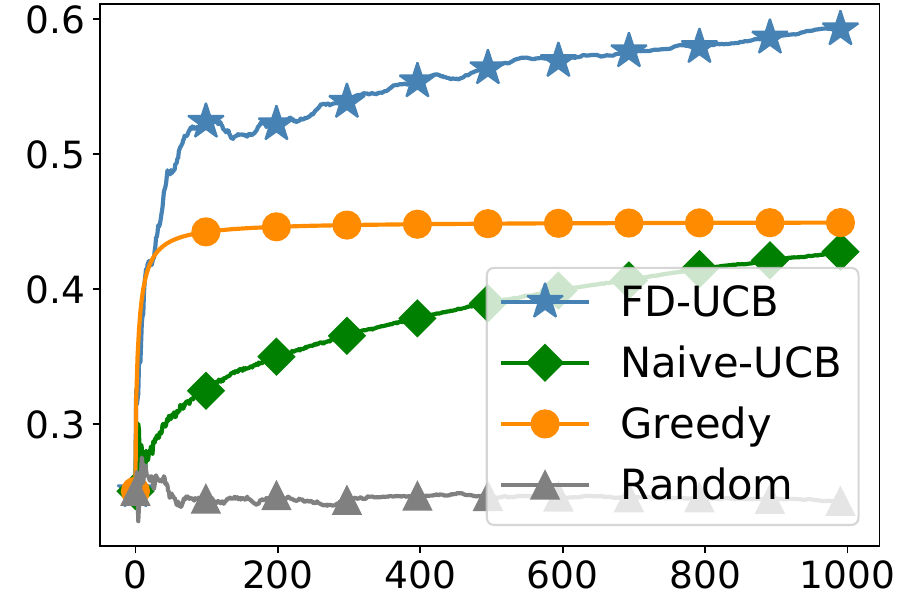}}
\caption{Online FAD-based evaluation and selection among four synthetic audio data. The embeddings are extracted by VGGish~\citep{hershey2017cnnarchitectureslargescaleaudio}. Results for the OPR metric are reported and averaged over 20 trials.}
\label{fd-results-audio}
\end{figure}

\subsection{Results of Online IS-Based Evaluation and Selection}

\textbf{Performance on the pretrained generators.}\quad The results are summarized in Figure~\ref{is-results}. IS-UCB~(blue) attains significantly better performance than naive-UCB~(orange) and the greedy algorithm~(green) on CIFAR10 and FFHQ datasets.

\noindent\textbf{Performance on variance-controlled generators.}\quad The results are summarized in Figure~\ref{is-results-vc}. The results suggest that FD-UCB consistently queries images that are diverse and of high quality. In contrast, Naive-UCB and the Greedy algorithm kept querying a proportion of images from the collapsed generators with smaller truncation parameters. The AFHQ-Dog dataset results are in the Appendix, Figure~\ref{is-results-afhqdog-vc}.

\section{\MakeUppercase{CONCLUSION}}
\label{sec:8}
In this work, we studied an online learning problem aiming to identify the generative model with the best evaluation score among a set of models. We proposed the multi-armed bandit-based FD-UCB and IS-UCB algorithms to address the online learning task and showed satisfactory numerical results for their application to image-based generative models. An interesting future direction to our work is to explore the application of other online learning frameworks, such as Thompson sampling, to address the online evaluation problem. Also, while we mostly focused on standard image datasets for testing the proposed algorithms, the application of the proposed FD-UCB and IS-UCB to real text-based and video-based generative models will be interesting for future studies. Such an application only requires applying text and video-based embeddings, as in Video IS~\citep{Saito_2020}, FVD~\citep{unterthiner2019fvd}, and FBD~\citep{xiang2021assessing}. Also, the extension of our MAB framework to a contextual bandit algorithm for prompt-guided generative models, as studied by \cite{hu2024online}, will remain for future studies.

\subsection*{Acknowledgments}
The work of Farzan Farnia is partially supported by a grant from the Research Grants Council of the Hong Kong Special Administrative Region, China, Project 14209920, and is partially supported by CUHK Direct Research Grants with CUHK Project No. 4055164 and 4937054. 
The authors would also like to thank the anonymous reviewers for their constructive feedback and suggestions.

\bibliographystyle{apalike}
\bibliography{ref}

\section*{Checklist}

 \begin{enumerate}

 \item For all models and algorithms presented, check if you include:
 \begin{enumerate}
   \item A clear description of the mathematical setting, assumptions, algorithm, and/or model. [Yes]
   \item An analysis of the properties and complexity (time, space, sample size) of any algorithm. [Yes]
   \item (Optional) Anonymized source code, with specification of all dependencies, including external libraries. [Yes]
 \end{enumerate}

 \item For any theoretical claim, check if you include:
 \begin{enumerate}
   \item Statements of the full set of assumptions of all theoretical results. [Yes]
   \item Complete proofs of all theoretical results. [Yes]
   \item Clear explanations of any assumptions. [Yes]     
 \end{enumerate}

 \item For all figures and tables that present empirical results, check if you include:
 \begin{enumerate}
   \item The code, data, and instructions needed to reproduce the main experimental results (either in the supplemental material or as a URL). [Yes]
   \item All the training details (e.g., data splits, hyperparameters, how they were chosen). [Yes]
         \item A clear definition of the specific measure or statistics and error bars (e.g., with respect to the random seed after running experiments multiple times). [Yes]
         \item A description of the computing infrastructure used. (e.g., type of GPUs, internal cluster, or cloud provider). [Yes]
 \end{enumerate}

 \item If you are using existing assets (e.g., code, data, models) or curating/releasing new assets, check if you include:
 \begin{enumerate}
   \item Citations of the creator If your work uses existing assets. [Yes]
   \item The license information of the assets, if applicable. [Yes]
   \item New assets either in the supplemental material or as a URL, if applicable. [Yes]
   \item Information about consent from data providers/curators. [Yes]
   \item Discussion of sensible content if applicable, e.g., personally identifiable information or offensive content. [Not Applicable]
 \end{enumerate}

 \item If you used crowdsourcing or conducted research with human subjects, check if you include:
 \begin{enumerate}
   \item The full text of instructions given to participants and screenshots. [Not Applicable]
   \item Descriptions of potential participant risks, with links to Institutional Review Board (IRB) approvals if applicable. [Not Applicable]
   \item The estimated hourly wage paid to participants and the total amount spent on participant compensation. [Not Applicable]
 \end{enumerate}

 \end{enumerate}

\newpage
\appendix
\onecolumn
\allowdisplaybreaks

\section{\MakeUppercase{Proofs in Section~\ref{sec:5}: FD-Based Evaluation}}

\subsection{Proof of Theorem \ref{thm_ofid}: Optimistic FD Score}
\label{p_thm_ofid}

\begin{proof}
The proof of Theorem~\ref{thm_ofid} is based on Theorem~\ref{thm_cef} and Lemma~\ref{thm_cemc} in the Appendix. Specifically, Inequality~(\ref{eq_39}) in Theorem~\ref{thm_cef} in Appendix~\ref{p_thm_cef} decomposes the estimation error of the empirical FD score~(\ref{emp_fid}) into the errors in estimating the mean~(i.e., $\|\widehat{\mu}_g^n-\mu_g\|_2$) and the covariance matrix~(i.e., $\|\widehat{\Sigma}_g^n-\Sigma_g\|_2$): With probability at least $1-\frac{\delta}{3}$, it holds that
\begin{equation*}
\begin{aligned}
    \left|\widetilde{\textup{FD}}_g^n-\textup{FD}_g\right|\le&2\|\mu_g-\hatmu_g^n\|_2\cdot\left(\|\mu_g-\hatmu_g^n\|_2+\|\widehat{\mu}_g^n-\mur\|_2\right)+\tr[\sigr^\frac{1}{2}]\sqrt{8\|\Sigma_g-\hatsig_g^n\|_2}\\
    &+\tr[\Sigma_g]\sqrt{\frac{8}{n}\log\left(\frac{6}{\delta}\right)}+\frac{8\|\Sigma_g\|_2}{n}\log\left(\frac{6}{\delta}\right).
\end{aligned}
\end{equation*}
Next, Lemma~\ref{thm_cemc} in Appendix~\ref{p_thm_cemc} derives generator-dependent concentration errors for the mean and the covariance matrix: With probability at least $1-\frac{2\delta}{3}$, it holds that
\begin{gather*}
    \|\hatmu_g^n-\mu_g\|_2\le \sqrt{\frac{1}{n}\left(\sqrt{8\cdot\tr[\Sigma_g^2]\log\left(\frac{6}{\delta}\right)}+8\|\Sigma_g\|_2\log\left(\frac{6}{\delta}\right)\right)}=:\Delta_{\mu_g}^n \\
    \left\|\Sigma_g-\widehat{\Sigma}_g^n\right\|_2\le20\kappa^2\|\Sigma_g\|_2\sqrt{\frac{4\bm{r}(\Sigma_g)+\log(3/\delta)}{n}}+(\Delta_{\mu_g}^n)^2=:\Delta_{\Sigma_g}^n
\end{gather*}
Therefore, let
\begin{equation*}
    \gB_g^n = 2 \Delta_{\mu_g}^n \cdot\left(\Delta_{\mu_g}^n+\|\widehat{\mu}_g^n-\mur\|_2\right)+\tr[\sigr^\frac{1}{2}]\sqrt{8 \Delta_{\Sigma_g}^n}+\tr[\Sigma_g]\sqrt{\frac{8}{n}\log\left(\frac{6}{\delta}\right)}+\frac{8\|\Sigma_g\|_2}{n}\log\left(\frac{6}{\delta}\right),
\end{equation*}
and we have that $\widetilde{\textup{FD}}_g^n-\gB_g^n \le \textup{FD}_g$ holds with probability at least $1-\delta$, which concludes the proof.
\end{proof}

\subsection{Regret of FD-UCB}
\label{p_thm_fid_reg}

\begin{theorem}[Regret of FD-UCB]
\label{thm_fid_reg}
Under the same conditions in Theorem~\ref{thm_ofid}, with probability at least $1-\delta$, the regret of the FD-UCB algorithm with a batch size $b\in\sN_+$ is bounded by
\begin{equation}
\label{eq_fid_reg}
    \widetilde{O}\left( \frac{D_1^\gG\cdot G^{1/4}}{b^{1/4}} T^{3/4} + \frac{D_2^\gG\cdot G^{1/2}}{b^{1/2}} T^{1/2} + \frac{D_3^\gG\cdot G}{b} \log T + N\cdot G\right),
\end{equation}
after running on $G:=|\gG|$ models for $T\in\sN_+$ steps, where $D_1^\gG$ and $D_2^\gG$ are model-dependent parameters defined in Equations~(\ref{eq16})-(\ref{eq18}), and the logarithmic factors are hidden in the notation $\widetilde{O}(\cdot)$.
\end{theorem}

\begin{proof}
Recall that $g_t$ denotes the generator picked at the $t$-th step. For convenience, we denote by $\widehat{\textup{FD}}_{g_t}$ the optimistic FD of generator $g_t$ computed at step $t$. We denote by $n_t$ the \emph{number of images} generated by model $g_t$ at the beginning of the $t$-th step. First, by Theorem~\ref{thm_cemc} and a union bound over $T$ steps, with probability at least $1-\delta$, we have that $\textup{FD}_{g_t}\ge \widehat{\textup{FD}}_{g_t}$ for any step $t\in[T]$ and $g\in[G]$. Hence, we have that
\begin{equation*}
    \textup{Regret}(T)=O(N\cdot G) + \sum_{t=1}^T\left(\textup{FD}_{g_t}-\textup{FD}^*\right)\le O(N\cdot G) + \sum_{t=1}^T\left(\textup{FD}_{g_t}-\widehat{\textup{FD}}_{g_t}\right),
\end{equation*}
where $\textup{FD}^*=\min_{g} \textup{FD}_g$, and the term $O(N\cdot G)$ corresponds to the regret incurred in the burn-in sampling phase. Further, by the definition of $\widehat{\textup{FD}}_{g_t}$, we further derive that
\begin{align*}
    \textup{Regret}(T) \le & O(N\cdot G) + \sum_{t=1}^T\left(\textup{FD}_{g_t}-\widetilde{\textup{FD}}_{g_t}^{n_t}+\gB^{n_t}_{g_t}\right)\\
    \le& O(N\cdot G) + 2\sum_{t=1}^T\gB^{n_t}_{g_t}\\
    \le& \widetilde{O}\left(N\cdot G+\sum_{t=1}^T\left(\Delta_{\mu_{g_t}}^{n_t}\cdot\left(\Delta_{\mu_{g_t}}^{n_t}+\|\mu_{g_t}-\mur\|_2\right)+\tr[\sigr^\frac{1}{2}]\sqrt{\Delta_{\Sigma_{g_t}}^{n_t}}+\tr[\Sigma_{g_t}]\sqrt{\frac{1}{n_t}}+\frac{\|\Sigma_{g_t}\|_2}{n_t}\right)\right),
\end{align*}
where we utilize $\|\widehat{\mu}_{g_t}^{n_t}-\mur\|_2 \le \Delta_{\mu_{g_t}}^{n_t} + \|\mu_{g_t}-\mur\|_2$ in the last inequality. Note that
\begin{gather*}
    \Delta_{\mu_{g_t}}^{n_t} = \sqrt{\frac{1}{n_t}\left(\sqrt{8\cdot\tr[\Sigma_{g_t}^2]\log\left(\frac{6T}{\delta}\right)}+8\|\Sigma_{g_t}\|_2\log\left(\frac{6T}{\delta}\right)\right)} = \widetilde{O}\left( n_t^{-1/2}\sqrt{\sqrt{\tr[\Sigma_{g_t}^2]}+\|\Sigma_{g_t}\|_2}\right)
\end{gather*}
and
\begin{align*}
    \Delta_{\Sigma_{g_t}}^{n_t} = & 20\kappa^2\|\Sigma_{g_t}\|_2\sqrt{\frac{4\bm{r}(\Sigma_{g_t})+\log(3T/\delta)}{n_t}}+(\Delta_{\mu_{g_t}}^{n_t})^2 \\
    = & \widetilde{O}\left( n_t^{-1/2}\sqrt{\|\Sigma_{g_t}\|_2^2 (\bm{r}(\Sigma_{g_t})+1)}+n_t^{-1} \left(\sqrt{\tr[\Sigma_{g_t}^2]}+\|\Sigma_{g_t}\|_2\right) \right). 
\end{align*}
Hence, the regret is further bounded by
\begin{equation*}
\begin{aligned}
    \widetilde{O}\Bigg(N\cdot G+&\sum_{t=1}^T n_t^{-1/4}\cdot\tr[\sigr^\frac{1}{2}]\left(\|\Sigma_{g_t}\|_2^2 (\bm{r}(\Sigma_{g_t})+1)\right)^{1/4} \\
    & + n_t^{-1/2}\cdot\left((\|\mu_{g_t}-\mur\|_2 + \tr[\sigr^\frac{1}{2}])\left(\sqrt{\tr[\Sigma_{g_t}^2]}+\|\Sigma_{g_t}\|_2\right)^{1/2} + \tr[\Sigma_{g_t}]\right) \\
    & + n_t^{-1} \cdot \left( \sqrt{\tr[\Sigma_{g_t}^2]}+\|\Sigma_{g_t}\|_2 \right) \Bigg).
\end{aligned}
\end{equation*}
Let
\begin{gather}
    D_1^\gG := \max_g \left\{\tr[\sigr^\frac{1}{2}]\left(\|\Sigma_{g}\|_2^2 (\bm{r}(\Sigma_{g})+1)\right)^{1/4}\right\}, \label{eq16}\\
    D_2^\gG := \max_g\left\{ (\|\mu_{g}-\mur\|_2 + \tr[\sigr^\frac{1}{2}])\left(\sqrt{\tr[\Sigma_{g}^2]}+\|\Sigma_{g}\|_2\right)^{1/2} + \tr[\Sigma_{g}] \right\}, \\
    D_3^\gG := \max_g \left\{ \sqrt{\tr[\Sigma_{g}^2]}+\|\Sigma_{g}\|_2 \right\}. \label{eq18}
\end{gather}
Then, we have that
\begin{equation*}
    \textup{Regret}(T) \le \widetilde{O}\left( \frac{D_1^\gG\cdot G^{1/4}}{b^{1/4}} T^{3/4} + \frac{D_2^\gG\cdot G^{1/2}}{b^{1/2}} T^{1/2} + \frac{D_3^\gG\cdot G}{b} \log T + N\cdot G \right)
\end{equation*}
where we use the fact that
\begin{gather*}
    \sum_{t=1}^T\frac{1}{n_t^\alpha}\le\frac{2^\alpha}{b^{\alpha}(1-\alpha)}\sum_{g=1}^G (N_g(T))^{1-\alpha}\le\frac{2^\alpha}{b^\alpha(1-\alpha)}G^\alpha T^{1-\alpha} \text{, where }0\le\alpha<1,\\
    \sum_{t=1}^T\frac{1}{n_t}\le\frac{2}{b}\sum_{g=1}^G\log(N_g(T))\le \frac{2G}{b}\log T,
\end{gather*}
where $b$ is the batch size, and $N_g(T)$ is the number of picks of model $g$ at the last step $T$. Therefore, we conclude the proof.
\end{proof}

\subsection{Concentration of Empirical FD~(\ref{emp_fid})}
\label{p_thm_cef}

The following theorem expresses the estimation error of the empirical FD by the concentration of the sample mean and covariance matrix, which facilitates the derivation of the bonus function.

\begin{theorem}[Concentration of empirical FD (\ref{emp_fid})]
\label{thm_cef}
Assume the covariance matrix $\Sigma_\textup{r}\succ\bm{0}$ is positive strictly definite. Then, with probability at least $1-\frac{\delta}{3}$, we have that
\begin{equation}
\label{eq_39}
\begin{aligned}
    \left|\widetilde{\textup{FD}}_g^n-\textup{FD}_g\right|\le&2\|\mu_g-\hatmu_g^n\|_2\cdot\left(\|\mu_g-\hatmu_g^n\|_2+\|\widehat{\mu}_g^n-\mur\|_2\right)+\tr[\sigr^\frac{1}{2}]\sqrt{8\|\Sigma_g-\hatsig_g^n\|_2}\\
    &+\tr[\Sigma_g]\sqrt{\frac{8}{n}\log\left(\frac{6}{\delta}\right)}+\frac{8\|\Sigma_g\|_2}{n}\log\left(\frac{6}{\delta}\right).
\end{aligned}
\end{equation}

\begin{proof}
For any generator $g$ and $n\in\sN_+$, we have that
\begin{gather*}
    \textup{FD}_g=\|\mu_g-\mu_\textup{r}\|_2^2+\textup{Tr}[\Sigma_g+\Sigma_\textup{r}-2(\Sigma_g\Sigma_\textup{r})^\frac{1}{2}],\\
    \widetilde{\textup{FD}}_g^n=\|\widehat{\mu}_g^n-\mu_\textup{r}\|_2^2+\textup{Tr}[\widehat{\Sigma}_g^n+\Sigma_\textup{r}-2(\widehat{\Sigma}_g^n\Sigma_\textup{r})^\frac{1}{2}].
\end{gather*}
\textbf{(1) Bound $\|\mu_g-\mur\|_2^2-\|\widehat{\mu}_g^n-\mur\|_2^2$. }\ We derive
\begin{equation}
\label{eq15}
\begin{aligned}
    & \left|\|\mu_g-\mu_\textup{r}\|_2^2-\|\widehat{\mu}_g^n-\mu_\textup{r}\|_2^2\right|\\
    = & (\|\mu_g-\mu_\textup{r}\|_2+\|\widehat{\mu}_g^n-\mu_\textup{r}\|_2)\cdot\left|\|\mu_g-\mu_\textup{r}\|_2-\|\widehat{\mu}_g^n-\mu_\textup{r}\|_2\right|\\
    \le & (\|\mu_g-\hatmu_g^n\|_2+2\|\widehat{\mu}_g^n-\mu_\text{r}\|_2)\cdot\|\widehat{\mu}_g^n-\mu_g\|_2.
\end{aligned}
\end{equation}
\noindent\textbf{(2) Bound $\textup{Tr}[(\Sigma_g\Sigma_\textup{r})^\frac{1}{2}-(\widehat{\Sigma}_g^n\Sigma_\textup{r})^\frac{1}{2}]$. }\ By Lemma~\ref{thm_oett}, if the covariance matrix $\Sigma_\textup{r}$ is positive strictly definite, then it holds that
\begin{equation}
\label{eq33}
    \left|\textup{Tr}[(\Sigma_g\Sigma_\textup{r})^\frac{1}{2}]-\textup{Tr}\left[(\widehat{\Sigma}_g^n\Sigma_\textup{r})^\frac{1}{2}\right]\right|\le\tr[\sigr^\frac{1}{2}]\sqrt{2\|\Sigma_g-\hatsig_g^n\|_2}.   
\end{equation}

\noindent\textbf{(3) Bound $\textup{Tr}[\Sigma_g-\widehat{\Sigma}_g^n]$. }\ Note that
\begin{equation}
\begin{aligned}
    \tr[\hatsig_g^n]=& \tr\left[\frac{1}{n}\sum_{i=1}^n(f(x^i)-\hatmu_g^n)(f(x^i)-\hatmu_g^n)^\top\right] \\
    = & \tr\left[\frac{1}{n}\sum_{i=1}^n(f(x^i)-\mu_g+\mu_g-\hatmu_g^n)(f(x^i)-\mu_g+\mu_g-\hatmu_g^n)^\top\right] \\
    = & \tr\left[\frac{1}{n}\sum_{i=1}^n (f(x^i)-\mu_g)(f(x^i)-\mu_g)^\top - (\mu_g-\hatmu_g^n)(\mu_g-\hatmu_g^n)^\top \right] \\
    = & \frac{1}{n}\sum_{i=1}^n \|f(x^i)-\mu_g\|_2^2 - \|\mu_g-\hatmu_g^n\|_2^2
\end{aligned}
\end{equation}
Hence, we obtain
\begin{equation}
\label{eq_34}
    \left|\tr[\Sigma_g]-\tr[\hatsig_g^n]\right|\le \left|\E[\|f(X_g)-\mu_g\|_2^2]-\frac{1}{n}\sum_{i=1}^n \|f(x^i)-\mu_g\|_2^2\right| + \|\mu_g-\hatmu_g^n\|_2^2
\end{equation}
Note that $f(X_g)-\mu_g\sim\gN(0,\Sigma_g)$. By Lemma~\ref{lem_slmg}, with probability at least $1-\frac{\delta}{3}$, it holds that
\begin{align*}
    \left|\frac{1}{n}\sum_{i=1}^n\|f(X^i)-\mu_g\|_2^2-\E[\|f(X_g)-\mu_g\|_2^2]\right|\le&\tr[\Sigma_g]\sqrt{\frac{8}{n}\log\left(\frac{6}{\delta}\right)}+\frac{8\|\Sigma_g\|_2}{n}\log\left(\frac{6}{\delta}\right).
\end{align*}

\noindent\textbf{Putting all together. }\ Therefore, combining Inequalities~(\ref{eq15}), (\ref{eq33}), and (\ref{eq_34}), with probability at least $1-\frac{\delta}{3}$, it holds that
\begin{equation*}
\begin{aligned}
    &\left|\widetilde{\textup{FD}}_g^n-\textup{FD}_g\right|\\
    \le&\underbrace{(\|\mu_g-\hatmu_g^n\|_2+2\|\widehat{\mu}_g^n-\mur\|_2)\cdot\|\widehat{\mu}_g^n-\mu_g\|_2}_\text{(1) error of $\|\hatmu_g^n-\mur\|_2^2$}+\underbrace{\tr[\sigr^\frac{1}{2}]\sqrt{8\|\Sigma_g-\hatsig_g^n\|_2}}_\text{(2) error of $2\tr[(\hatsig_g^n\sigr)^{\frac{1}{2}}]$}\\
    &+\underbrace{\tr[\Sigma_g]\sqrt{\frac{8}{n}\log\left(\frac{6}{\delta}\right)}+\frac{8\|\Sigma_g\|_2}{n}\log\left(\frac{6}{\delta}\right)+\|\mu_g-\hatmu_g^n\|_2^2}_\text{(3) error of $\tr[\hatsig_g^n]$}\\
    = & 2\|\mu_g-\hatmu_g^n\|_2\cdot\left(\|\mu_g-\hatmu_g^n\|_2+\|\widehat{\mu}_g^n-\mur\|_2\right)+\tr[\sigr^\frac{1}{2}]\sqrt{8\|\Sigma_g-\hatsig_g^n\|_2}\\
    &+\tr[\Sigma_g]\sqrt{\frac{8}{n}\log\left(\frac{6}{\delta}\right)}+\frac{8\|\Sigma_g\|_2}{n}\log\left(\frac{6}{\delta}\right),
\end{aligned}
\end{equation*}
which concludes the proof.
\end{proof}
\end{theorem}

\subsection{Concentration of Mean Vector and Covariance Matrix}
\label{p_thm_cemc}

\begin{lemma}[$L$2-norm error for mean and covariance matrix]
\label{thm_cemc}
Under the same conditions in Theorem~\ref{thm_cef}, with probability at least $1-\frac{2\delta}{3}$, we have that
\begin{equation}
\label{eq52}
     \|\hatmu_g^n-\mu_g\|_2\le \sqrt{\frac{1}{n}\left(\sqrt{8\cdot\tr[\Sigma_g^2]\log\left(\frac{6}{\delta}\right)}+8\|\Sigma_g\|_2\log\left(\frac{6}{\delta}\right)\right)}=:\Delta_{\mu_g}^n
\end{equation}
and
\begin{equation}
\label{eq57}
    \left\|\Sigma_g-\widehat{\Sigma}_g^n\right\|_2\le20\kappa^2\|\Sigma_g\|_2\sqrt{\frac{4\bm{r}(\Sigma_g)+\log(3/\delta)}{n}}+(\Delta_{\mu_g}^n)^2=:\Delta_{\Sigma_g}^n
\end{equation}
for $n\ge 4\bm{r}(\Sigma_g)+\log(3/\delta)$, where $\bm{r}(\Sigma_g)=\frac{\tr[\Sigma_g]}{\|\Sigma_g\|_2}$ is the effective rank of $\Sigma_g$.

\begin{proof}
\textbf{1. Concentration of squared $L$2-norm error. }\quad Note that
\begin{align*}
    \|\hatmu_g^n-\mu_g\|_2^2=\underbrace{\|\hatmu_g^n-\mu_g\|_2^2-\E[\|\hatmu_g^n-\mu_g\|_2^2]}_\text{(*): concentration}+\underbrace{\E[\|\hatmu_g^n-\mu_g\|_2^2]}_\text{expected L2-norm error},
\end{align*}
where $\hatmu_g^n-\mu_g \sim \gN(0,\frac{\Sigma_g}{n})$. By Lemma~\ref{lem_slmg}, we have $\E[\|\hatmu_g^n-\mu_g\|_2^2]=\frac{\tr[\Sigma_g]}{n}$, and with probability at least $1-\frac{\delta}{3}$, it holds that
\begin{equation*}
    \left|\|\hatmu_g^n-\mu_g\|_2^2-\E[\|\hatmu_g^n-\mu_g\|_2^2]\right|\le \frac{1}{n}\left(\sqrt{8\cdot\tr[\Sigma_g^2]\log\left(\frac{6}{\delta}\right)}+8\|\Sigma_g\|_2\log\left(\frac{6}{\delta}\right)\right).
\end{equation*}

\noindent\textbf{2. Concentration of sample covariance matrix. } 
We have that
\begin{align*}
    & \left\|\Sigma_g-\widehat{\Sigma}_g^n\right\|_2\\
    =& \left\|\Sigma_g-\frac{1}{n}\sum_{i=1}^n(f(x^i)-\widehat{\mu}_g^n)(f(x^i)-\widehat{\mu}_g^n)^\top\right\|_2\\
    = & \left\|\Sigma_g-\frac{1}{n}\sum_{i=1}^n(f(x^i)-\mu_g+\mu_g-\widehat{\mu}_g^n)(f(x^i)-\mu_g+\mu_g-\widehat{\mu}_g^n)^\top\right\|_2\\
    = & \left\|\Sigma_g-\frac{1}{n}\sum_{i=1}^n\left((f(x^i)-\mu_g) (f(x^i)-\mu_g)^\top -2(f(x^i)-\mu_g)(\mu_g-\widehat{\mu}_g^n)^\top+(\mu_g-\widehat{\mu}_g^n) (\mu_g-\widehat{\mu}_g^n)^\top\right)\right\|_2\\
    \le & \left\|\Sigma_g-\frac{1}{n}\sum_{i=1}^n(f(x^i)-\mu_g)(f(x^i)-\mu_g)^\top\right\|_2+\left\|(\mu_g-\widehat{\mu}_g^n)(\mu_g-\widehat{\mu}_g^n)^\top\right\|_2.\numberthis\label{eq_296}
\end{align*}
For the first term, note that $f(X_g)-\mu_g \sim \gN(0,\Sigma_g)$. Then, by Lemma~\ref{lem_dfcscm}, with probability at least $1-\frac{\delta}{3}$, it holds that
\begin{equation}
\label{eq_301}
    \left\|\Sigma_g-\frac{1}{n}\sum_{i=1}^n(f(x^i)-\mu_g)(f(x^i)-\mu_g)^\top\right\|_2\le20\kappa^2\|\Sigma_g\|_2\sqrt{\frac{4\bm{r}(\Sigma_g)+\log(3/\delta)}{n}}
\end{equation}
for $n\ge4\bm{r}(\Sigma_g)+\log(3/\delta)$, where $\bm{r}(\Sigma_g)=\frac{\tr[\Sigma_g]}{\|\Sigma_g\|_2}$ and $\kappa$ is a constant defined therein. For the second term, we have that
\begin{equation*}
\begin{aligned}
    \left\|(\mu_g-\widehat{\mu}_g^n)(\mu_g-\widehat{\mu}_g^n)^\top\right\|_2=\|\mu_g-\widehat{\mu}_g^n\|_2^2\le(\Delta_{\mu_g}^n)^2,
\end{aligned}
\end{equation*}
which concludes the proof.
\end{proof}
\end{lemma}

\subsection{Concentration Bound for Norm-Bounded Embeddings}

We extend the concentration bound for norm-bounded embeddings.

\begin{theorem}
\label{thm-bounded-norm}

Assume for any generator $g$, the (random) embedding $f(X_g)$ is $L2$-norm bounded, i.e., $\|f(X_g)\|_2 \le C$ for some positive number $C > 0$, and the covariance matrix $\sigr$ of the real data is positive definite. Then, with probability at least $1-\delta$, we have
\begin{equation*}
    \widetilde{\textup{FD}}_g^n-\gC^n_{g}\le\textup{FD}_g,
\end{equation*}
where
\begin{equation*}
    \gC^n_{g} := 2(\Delta_1^n+\|\widehat{\mu}_g^n-\mur\|_2)\cdot\Delta_1^n + \tr[\sigr^\frac{1}{2}]\sqrt{8\Delta_2^n} + 4C^2\sqrt{\frac{1}{2n}\log\left(\frac{3}{\delta}\right)},
\end{equation*}
and $\Delta_1^n$ and $\Delta_2^n$ are defined in~(\ref{eq-249}) and~(\ref{eq-254}), respectively.

\begin{proof}
The proof is similar to Theorem~\ref{thm_ofid}. Recall that for any generator $g$, we have that
\begin{gather*}
    \textup{FD}_g=\|\mu_g-\mu_\textup{r}\|_2^2+\textup{Tr}[\Sigma_g+\Sigma_\textup{r}-2(\Sigma_g\Sigma_\textup{r})^\frac{1}{2}],\\
    \widetilde{\textup{FD}}_g^n=\|\widehat{\mu}_g^n-\mu_\textup{r}\|_2^2+\textup{Tr}[\widehat{\Sigma}_g^n+\Sigma_\textup{r}-2(\widehat{\Sigma}_g^n\Sigma_\textup{r})^\frac{1}{2}].
\end{gather*}
\textbf{(1) Bound $\|\mu_g-\mur\|_2^2-\|\widehat{\mu}_g^n-\mur\|_2^2$. }\ We derive
\begin{equation*}
    \left|\|\mu_g-\mu_\textup{r}\|_2^2-\|\widehat{\mu}_g^n-\mu_\textup{r}\|_2^2\right| \le (\|\mu_g-\hatmu_g^n\|_2+2\|\widehat{\mu}_g^n-\mur\|_2)\cdot\|\widehat{\mu}_g^n-\mu_g\|_2.
\end{equation*}

\noindent\textbf{(2) Bound $\textup{Tr}[(\Sigma_g\Sigma_\textup{r})^\frac{1}{2}-(\widehat{\Sigma}_g^n\Sigma_\textup{r})^\frac{1}{2}]$. }\ By Lemma~\ref{thm_oett}, if the covariance matrix $\Sigma_\textup{r}$ is positive strictly definite, then it holds that
\begin{equation*}
    \left|\textup{Tr}[(\Sigma_g\Sigma_\textup{r})^\frac{1}{2}]-\textup{Tr}\left[(\widehat{\Sigma}_g^n\Sigma_\textup{r})^\frac{1}{2}\right]\right|\le\tr[\sigr^\frac{1}{2}]\sqrt{2\|\Sigma_g-\hatsig_g^n\|_2}.   
\end{equation*}

\noindent\textbf{(3) Bound $\textup{Tr}[\Sigma_g-\widehat{\Sigma}_g^n]$. }\ Note that
\begin{align*}
    \left|\tr[\Sigma_g]-\tr[\hatsig_g^n]\right|\le \left|\E[\|f(X_g)-\mu_g\|_2^2]-\frac{1}{n}\sum_{i=1}^n \|f(x^i)-\mu_g\|_2^2\right| + \|\mu_g-\hatmu_g^n\|_2^2.
\end{align*}
As $\|f(X_g)\|_2 \le C$, we have $\|f(X_g) - \mu_g\|_2^2 \le (\|f(X_g)\|_2 + \|\mu_g\|_2)^2 \le (2C)^2$. By Hoeffding's inequality, we have that with probability at least $1-\frac{\delta}{3}$, it holds that
\begin{align*}
    \left|\frac{1}{n}\sum_{i=1}^n\|f(X^i)-\mu_g\|_2^2-\E[\|f(X_g)-\mu_g\|_2^2]\right|\le 4C^2\sqrt{\frac{1}{2n}\log\left(\frac{6}{\delta}\right)}
\end{align*}
Therefore, we have
\begin{equation}
    \left| \textup{FD}_g - \widetilde{\textup{FD}}_g^n \right| \le  2(\|\mu_g-\hatmu_g^n\|_2+\|\widehat{\mu}_g^n-\mur\|_2)\cdot\|\widehat{\mu}_g^n-\mu_g\|_2 + \tr[\sigr^\frac{1}{2}]\sqrt{8\|\Sigma_g-\hatsig_g^n\|_2} + 4C^2\sqrt{\frac{1}{2n}\log\left(\frac{3}{\delta}\right)}.
\end{equation}
Next, we analyze the concentration error of the $L2$-norm error. Since
\begin{equation*}
    \|\hatmu_g^n-\mu_g\|_2^2=\underbrace{\|\hatmu_g^n-\mu_g\|_2^2-\E[\|\hatmu_g^n-\mu_g\|_2^2]}_\text{ concentration}+\underbrace{\E[\|\hatmu_g^n-\mu_g\|_2^2]}_\text{expected $L$2-norm error},
\end{equation*}
the expected $L$2-norm error is given by
\begin{equation*}
    \E[\|\hatmu_g^n-\mu_g\|_2^2] = \frac{\tr[\Sigma_g]}{n}.
\end{equation*}
Additionally, we have that $|\hatmu_g^n[i]-\mu_g[i]| \le 2C$ for the $i$-th entry where $i \in [d]$. Hence, by Hoeffding's inequality and a union bound over all dimensions, with probability at least $1-\frac{\delta}{3}$, it holds that
\begin{equation*}
    \|\hatmu_g^n-\mu_g\|_2^2-\E[\|\hatmu_g^n-\mu_g\|_2^2] \le \frac{2dC^2}{n}\log\left(\frac{6d}{\delta}\right).
\end{equation*}
Hence, we have
\begin{equation}
\label{eq-249}
    \|\hatmu_g^n-\mu_g\|_2 \le \sqrt{\frac{2dC^2}{n}\log\left(\frac{6d}{\delta}\right) + \frac{\tr[\Sigma_g]}{n}} := \Delta_1^n
\end{equation}
Further, we have that with probability at least $1-\frac{\delta}{3}$, it holds that
\begin{equation}
\label{eq-254}
    \left\|\Sigma_g-\widehat{\Sigma}_g^n\right\|_2\le20\kappa^2\|\Sigma_g\|_2\sqrt{\frac{4\bm{r}(\Sigma_g)+\log(3/\delta)}{n}}+(\Delta_1^n)^2 := \Delta_2^n
\end{equation}
for $n\ge4\bm{r}(\Sigma_g)+\log(3/\delta)$, which concludes the proof.

\end{proof}
\end{theorem}

\subsection{Auxiliary Definitions and Lemmas}

\begin{lemma}[Error decomposition for {$\textup{Tr}[(\widehat{\Sigma}_g^n\Sigma_\textup{r})^\frac{1}{2}]$}]
\label{thm_oett}

Under the same conditions in Theorem~\ref{thm_cef}, we have that
\begin{equation}
    \left|\textup{Tr}[(\Sigma_g\Sigma_\textup{r})^\frac{1}{2}]-\textup{Tr}[(\widehat{\Sigma}_g^n\Sigma_\textup{r})^\frac{1}{2}]\right|\le \tr[\sigr^\frac{1}{2}]\sqrt{2\|\Sigma_g-\hatsig_g^n\|_2}.
\end{equation}
\begin{proof}
Let $\eta:=\|\Sigma_g-\hatsig_g^n\|_2$ and $\sigr:=B B^\top$ denote the Cholesky decomposition of $\sigr$, where $B$ is invertible by the assumption $\sigr\succ\bm{0}$. Note that matrix $A\sigr=A B B^\top$ has the same set of eigenvalues with $B^\top A B$. We have that
\begin{align*}
    & \textup{Tr}[(\Sigma_g\Sigma_\textup{r})^\frac{1}{2}]-\textup{Tr}[(\widehat{\Sigma}_g^n\Sigma_\textup{r})^\frac{1}{2}]\\
    \le & \tr\left[\sqrt{B^\top(\Sigma_g-\hatsig_g^n+\hatsig_g^n+\eta I)B}\right]-\tr\left[\sqrt{B^\top\hatsig_g^n B}\right] \tag{Lemma~\ref{lem_additivity}}\\
    \le & \tr\left[\sqrt{B^\top(\Sigma_g-\hatsig_g^n+\eta I)B}+\sqrt{B^\top\hatsig_g^n B}\right]-\tr\left[\sqrt{B^\top\hatsig_g^n B}\right]\\
    \le & \tr\left[\sqrt{2\eta B^\top B}\right]=\tr[\sigr^\frac{1}{2}]\sqrt{2\eta},
\end{align*}
where the second inequality holds by Lemma~\ref{lem_square_add} and the fact that $B^\top(\Sigma_g-\hatsig_g^n+\eta I)B,B^\top\hatsig_g^n B\succeq \bm{0}$ are PSD, and the last inequality holds by $B^\top(2\eta I)B\succeq B^\top(\Sigma_g-\hatsig_g^n+\eta I)B$. By the same analysis, we can also derive that
\begin{equation*}
    \textup{Tr}[(\widehat{\Sigma}_g^n\Sigma_\textup{r})^\frac{1}{2}]-\textup{Tr}[(\Sigma_g\Sigma_\textup{r})^\frac{1}{2}]\le\tr\left[\sqrt{2\eta B^\top B}\right]=\tr[\sigr^\frac{1}{2}]\sqrt{2\eta},
\end{equation*}
which concludes the proof.
\end{proof}
\end{lemma}

\begin{definition}[Sub-Gaussian random variable]
A random variable $X$ with mean $\mu=\E[X]$ is sub-Gaussian with (positive) parameter $\sigma$, which is denoted by $X\in\textup{SG}(\sigma)$, if
\begin{equation*}
    \E[\exp(s(X-\mu))]\le\exp\left(\frac{s^2\sigma^2}{2}\right)
\end{equation*}
for all $s\in\sR$.
\end{definition}

\begin{definition}[Sub-Gaussian random vector]
A random vector $X\in\sR^d$ is a sub-Gaussian random vector with (positive) parameter $\sigma$ if $y^\top X\in\textup{SG}(\sigma)$. Particularly, the multivariate Gaussian $X\sim\gN(0,\Sigma)\in\textup{SG}(\sqrt{\|\Sigma\|_2})$.
\end{definition}

\begin{definition}[Sub-exponential random variable]
\label{def_se}
A random variable $X$ with mean $\mu=\E[X]$ is sub-exponential with (non-negative) parameters $(\nu,\alpha)$, which denoted by $X\in\textup{SE}(\nu,\alpha)$, if
\begin{equation*}
    \E[\exp(s(X-\mu))]\le\exp\left(\frac{s^2\nu^2}{2}\right)
\end{equation*}
for all $|s|<\frac{1}{\alpha}$. Moreover, we have
\begin{equation*}
    \sP[X-\E[X]\ge t]\le
    \begin{cases}
        \exp(-\frac{t^2}{2\nu^2}) & \textup{if } 0\le t\le \frac{\nu^2}{\alpha}, \\
        \exp(-\frac{t}{2\alpha}) & \textup{for } t> \frac{\nu^2}{\alpha}.
    \end{cases}
\end{equation*}
Equivalently, with probability at least $1-\delta$, it holds that
\begin{equation*}
    |X-\E[X]|\le\max\left\{\sqrt{2\nu^2\log\left(\frac{2}{\delta}\right)},2\alpha\log\left(\frac{2}{\delta}\right)\right\}.
\end{equation*}
\end{definition}

\begin{lemma}[Weighted sum of sub-exponential r.v.'s]
\label{lem_wsse}
Let $X_1,\cdots,X_n$ be $n$ independent random variables, where $X_i\in\textup{SE}(\nu_i,\alpha_i)$ is sub-exponential with mean $\E[X_i]=\mu_i$ and parameters $(\nu_i,\alpha_i)$. Then, for any $W=(w_1,\cdots,w_n)^\top\in\sR^n_+$, the (non-negative) weighted sum $\sum_{i=1}^n w_i X_i$ is sub-exponential with parameters $(\sqrt{\sum_{i=1}^n w_i^2 \nu_i^2},\max_i \{w_i\cdot\alpha_i\})$.

\begin{proof}
By Definition~\ref{def_se}, we have
\begin{equation*}
    \E[\exp(s\cdot w_i(X_i-\mu_i))]\le\exp\left(\frac{s^2 w_i^2\nu_i^2}{2}\right)
\end{equation*}
for any $i\in[n]$ and $|s|\cdot w_i<\frac{1}{\alpha_i}$ . Further, since $X_1,\cdots,X_n$ are independent, it holds that
\begin{equation*}
    \E[\exp(s\cdot \sum_{i=1}^n w_i(X_i-\mu_i))]\le\prod_{i=1}^n \exp\left(\frac{s^2 w_i^2\nu_i^2}{2}\right)=\exp\left(\frac{s^2\left(\sqrt{\sum_{i=1}^n w_i^2\nu_i^2}\right)^2}{2}\right)
\end{equation*}
for any $|s|<\min_i \frac{1}{w_i\cdot\alpha_i}$, which concludes the proof.
\end{proof}
\end{lemma}

\begin{lemma}[Squared L2-norm of multivariate Gaussian]
\label{lem_slmg}

Let $X\sim\gN(\bm{0},\Sigma)$ be any $d$-dimensional multivariate Gaussian with zero mean. Then, the r.v. $\|X\|_2^2 \sim \sum_{i=1}^d \lambda_i \chi^2_1$ is a weighted sum of i.i.d. chi-squared random variables with one degree of freedom, where $\{\lambda_i\}_{i=1}^d$ are the eigenvalues of the covariance matrix $\Sigma$. Further, let $X_1,\cdots,X_n\sim\gN(\bm{0},\Sigma)$ denote i.i.d. samples. Then, with probability at least $1-\delta$, it holds that
\begin{equation}
    \left|\frac{1}{n}\sum_{i=1}^n \|X_i\|_2^2 -\E[\|X\|_2^2]\right|\le \max\left\{\sqrt{\frac{8\cdot\tr[\Sigma^2]}{n}\log\left(\frac{2}{\delta}\right)},\frac{8\|\Sigma\|_2}{n}\log\left(\frac{2}{\delta}\right)\right\},
\end{equation}
where $\E[\|X\|_2^2]:=\sum_{i=1}^d \lambda_i=\tr[\Sigma]$.

\begin{proof}
Let $\Sigma=Q \Lambda Q^\top$ be the eigendecomposition of the covariance matrix, where we define $\Lambda := \textup{diag}(\lambda_1,\cdots,\lambda_d)$. Hence, we have $X\sim Q Z$, where $Z\sim\gN(\bm{0}, \Lambda)$ is \textit{isotropic} multivariate Gaussian. Hence, we derive
\begin{align*}
    \|X\|_2^2 = (QZ)^\top QZ = Z^\top Q^\top Q Z = Z^\top Z = \sum_{i=1}^d \lambda_i \chi_1^2.
\end{align*}
Since $\chi_1^2\in\textup{SE}(2,4)$ is sub-exponential with parameters $(2,4)$, by Lemma~\ref{lem_wsse} (with $w_i=\lambda_i$), we have $\|X\|_2^2\in\textup{SE}(2\sqrt{\tr[\Sigma^2]},4\|\Sigma\|_2)$ and has mean $\E[\|X\|_2^2]=\tr[\Sigma]$. We conclude the proof by Lemma~\ref{lem_bbse}.
\end{proof}
\end{lemma}

\begin{lemma}[Bernstein-type bound for sub-exponential r.v.]
\label{lem_bbse}
Let $X_1,\cdots,X_n$ be $n$ i.i.d. random variables, where each $X_i\in\textup{SE}(\nu,\alpha)$. Then, with probability at least $1-\delta$, it holds that
\begin{equation*}
    \left|\frac{1}{n}\sum_{i=1}^n X_i - \E[X]\right|\le \max\left\{\sqrt{\frac{2\nu^2}{n}\log\left(\frac{2}{\delta}\right)},\frac{2\alpha}{n}\log\left(\frac{2}{\delta}\right)\right\}
\end{equation*}
\begin{proof}
By Lemma~\ref{lem_wsse}, we have $\frac{1}{n}\sum_{i=1}^n X_i\in\textup{SE}(\frac{\nu}{\sqrt{n}},\frac{\alpha}{n})$, which concludes the proof.
\end{proof}
\end{lemma}

\begin{lemma}[Dimension-free concentration of sample covariance matrix \citep{50e90dda-bcf5-3696-b295-5da8c1bf2d0c, 10.3150/15-BEJ730, zhivotovskiy2022dimensionfree}]
\label{lem_dfcscm}

Let $X_1,\cdots,X_n$ are i.i.d sub-Gaussian random vectors with zero mean and covariance matrix $\Sigma$. Let
\begin{equation}
    \bm{r}(\Sigma):=\frac{\tr(\Sigma)}{\|\Sigma\|_2}
\end{equation}
denote the \emph{effective rank}. Then, with probability at least $1-\delta$, the sample covariance matrix satisfies that
\begin{equation}
    \left\|\frac{1}{n}\sum_{i=1}^nX_iX_i^\top-\Sigma\right\|_2\le 20\kappa^2\|\Sigma\|_2\left(\sqrt{\frac{4\bm{r}(\Sigma)+\log(1/\delta)}{n}}\right)
\end{equation}
for $n\ge 4\bm{r}(\Sigma)+\log(1/\delta)$. Here, the parameter $\kappa$ is a constant such that $\|y^\top X\|_{\psi_2}\le\kappa\sqrt{y^\top\Sigma y}$, where $\|Z\|_{\psi_2}:=\inf\{t>0:\E[\exp(Z^2/t^2)]\le 2\}$ is the sub-Gaussian norm.
\end{lemma}

\begin{lemma}
\label{lem_additivity}
Let $A,B\in\sS^d$ be positive semi-definite (PSD) matrices. Then, we have that
\begin{equation}
    \tr(\sqrt{A+B})\ge \max\{\tr(\sqrt{A}),\tr(\sqrt{B})\}
\end{equation}
\begin{proof}
Since $A+B\succeq A$ and $A+B\succeq B$, by Courant-Fischer min-max theorem, it holds that $\lambda_i(A+B)\ge\lambda_i(A)$ and $\lambda_i(A+B)\ge\lambda_i(B)$ for their $i$-th largest eigenvalues, which concludes the proof.
\end{proof}
\end{lemma}

\begin{lemma}
\label{lem_square_add}
Let $A,B\in\sS^{d}$ be positive semi-definite~(PSD) matrices. Then, we have that
\begin{equation}
    \tr[\sqrt{A+B}]\le\tr[\sqrt{A}]+\tr[\sqrt{B}].
\end{equation}
\begin{proof}
First, we assume that both matrices $A$ and $B$ are positive strictly definite. Then,
\begin{align*}
    \tr(\sqrt{A+B})=&\tr[(A+B)^{-\frac{1}{4}}(A+B)(A+B)^{-\frac{1}{4}}]\\
    = & \tr[(A+B)^{-\frac{1}{4}}A(A+B)^{-\frac{1}{4}}]+\tr[(A+B)^{-\frac{1}{4}}B(A+B)^{-\frac{1}{4}}] \\
    = & \tr[\sqrt{A}(A+B)^{-\frac{1}{2}}\sqrt{A}]+\tr[\sqrt{B}(A+B)^{-\frac{1}{2}}\sqrt{B}]\\
    \le & \tr[\sqrt{A}]+\tr[\sqrt{B}],
\end{align*}
where the last inequality holds by $A^{-\frac{1}{2}}\succeq(A+B)^{-\frac{1}{2}}$ and $B^{-\frac{1}{2}}\succeq(A+B)^{-\frac{1}{2}}$~(see Lemma~\ref{lem_540}). For general PSD matrices $A$ and $B$, note that
\begin{align*}
    \tr(\sqrt{A+B+\epsilon I})\le \tr(\sqrt{A+\epsilon I})+\tr(\sqrt{B+\epsilon I})
\end{align*}
holds for any $\epsilon>0$ by our previous analysis. Since the trace map is continuous,
\end{proof}
\end{lemma}

\begin{lemma}
\label{lem_540}

Let $A,B$ denote two PSD matrices such that $A\succeq B$. Then, their inverse $B^{-1}\succeq A^{-1}$.

\begin{proof}
It suffices to show that $\sqrt{A}B^{-1}\sqrt{A}\succeq I$. Since $\sqrt{A}B^{-1}\sqrt{A}=(\sqrt{A}B^{-\frac{1}{2}})(\sqrt{A}B^{-\frac{1}{2}})^\top$ shares the same eigenvalues with $(\sqrt{A}B^{-\frac{1}{2}})^\top(\sqrt{A}B^{-\frac{1}{2}})=B^{-\frac{1}{2}}AB^{-\frac{1}{2}}\succeq I$, which concludes the proof.
\end{proof}
\end{lemma}

\section{\MakeUppercase{Proofs in Section~\ref{sec:6}: IS-Based Evaluation}}

\subsection{Proof of Theorem~\ref{thm_ois}: Generator-Dependent Optimistic IS}
\label{p_thm_ois}

\begin{proof}
By Lemma~\ref{coro_domcd}, we have that with probability at least $1-\frac{\delta}{2}$, $\gE(\widehat{p}_{Y_g}^n)\ge H(Y_g)$. To show Theorem~\ref{thm_ois}, it suffices to show that 
\begin{equation*}
    H_g(Y_g|X_g)\ge \widehat{H}^n(Y_g|X_g)-\sqrt{\frac{2\widehat{V}_g^n(H(Y_g|X_g))}{n}\log\left(\frac{4d}{\delta}\right)}-\frac{7\log d}{3(n-1)}\log\left(\frac{4d}{\delta}\right)
\end{equation*}
with probability at least $1-\frac{\delta}{2}$, which holds by Theorem~\ref{thm_ebi} and the fact that $H_g(Y_g|X_g)\le\log d$. Therefore, we conclude the proof.
\end{proof}

\subsection{Regret of IS-UCB}
\label{p_thm_is_reg}

\begin{theorem}[Regret of IS-UCB]
\label{thm_is_reg}

With probability at least $1-\delta$, the regret of the IS-UCB algorithm after $T$ steps is bounded by 
\begin{equation}
    \textup{Regret}(T)\le \widetilde{O}\left(G\Delta^{-2}\log d+e^C\cdot\left((dm+\log d)\cdot\sqrt{T}+\log T\right)\right),
\end{equation}
where $C=\widetilde{O}(d)$, $m:=\max_{g\in[G],j\in[d]:p_{Y_g}[j]>0}\{|u'(p_{Y_g}[j])|\}$, $\Delta:=\min_{g\in[G],j\in[d]:p_{Y_g}[j]\neq e^{-1}}\{|p_{Y_g}[j]-e^{-1}|\}$, and $u(x)=-x\log x$.
\end{theorem}

\begin{proof}
By Theorem~\ref{thm_ois} and union bound over $T$ steps, with probability at least $1-\delta$, it holds that $\widehat{\textup{IS}}_{g_t}\ge \textup{IS}_{g_t}$ for all steps $t\in[T]$ and $g\in[G]$, where $\widehat{\textup{IS}}_{g}$ is given by Equation~(\ref{eq_29}). For convenience, we denote by $\widehat{\textup{IS}}_{g_t}$ the optimistic IS of generator $g_t$ computed at the $t$-th step. Hence, we have that
\begin{equation*}
    \textup{IS}^*-\textup{IS}_{g_t}\le\widehat{\textup{IS}}_{g_t}-\textup{IS}_{g_t}=e^{\gE\left(\widehat{p}_{Y_{g_t}}\right)-\widehat{H}(Y_{g_t}|X_{g_t})+\gB_{g_t}}-\textup{IS}_{g_t},
\end{equation*}
where we denote
\begin{equation*}
    \gB_{g_t} := \sqrt{\frac{2\widehat{V}_g^{n_t}(H(Y_g|X_g))}{n_t}\log\left(\frac{4Td}{\delta}\right)} + \frac{7\log d}{3(n_t-1)}\log\left(\frac{4Td}{\delta}\right)
\end{equation*}
for convenience. Let $\gT_1,\gT_2\subset[T]$ be two disjoint set of steps such that $\gT_2=[T]\backslash\gT_1$. Specifically, $\gT_1$ contains steps where no element of the empirical marginal class distribution $\widetilde{p}_{Y_{g_t}}$ is clipped to $e^{-1}$, i.e., $|e^{-1}-\widetilde{p}_{Y_{g_t}}[j]|\ge\bm{\epsilon}_{g_t}[j]$ for all $[j]\in d$. Let $\textup{IS}^*=\argmax_{g\in[G]}\textup{IS}_g$ denote the optimal Inception score. Hence, we have that
\begin{equation*}
    \textup{Regret}(T)=\sum_{t\in\gT_1}(\textup{IS}^*-\textup{IS}_{g_t})+\sum_{t\in\gT_2}(\textup{IS}^*-\textup{IS}_{g_t}).
\end{equation*}
Recall $u(x)=-x\log x$ for $x\in\sR_+$. For the first part, we further derive that
\begin{align*}
    & \sum_{t\in\gT_1}(\textup{IS}^*-\textup{IS}_{g_t})\\
    \le & \sum_{t\in\gT_1}\left(\exp\{\gE\left(\widehat{p}_{Y_{g_t}}\right)-\widehat{H}(Y_{g_t}|X_{g_t})+\gB_{g_t}\}-\exp\{H(Y_g)-H(Y_g|X_g)\}\right)\\
    \le & e^C\cdot\sum_{t\in\gT_1}\left(\gE\left(\widehat{p}_{Y_{g_t}}\right)-\widehat{H}(Y_{g_t}|X_{g_t})+\gB_{g_t}-H(Y_g)+H(Y_g|X_g)\right)\\
    \le & e^C\cdot\sum_{t\in\gT_1}\left(\sum_{j=1}^d\left(u(\widehat{p}_{Y_{g_t}}[j])-u(p_{Y_{g_t}}[j])\right)+2\gB_{g_t}\right)\\
    \le & O\left(e^C\cdot\sum_{t\in\gT_1}\left(m\cdot\sum_{j=1}^d|\widehat{p}_{Y_{g_t}}[j]-p_{Y_{g_t}}[j]|+\gB_{g_t}\right)\right)\\
    \le & O\left(e^C\cdot\sum_{t\in\gT_1}\left(m\cdot\sum_{j=1}^d\bm{\epsilon}_{g_t}[j]+\gB_{g_t}\right)\right)\\
    \le & \widetilde{O}\left(e^C\cdot\sum_{t=1}^T\left((dm+\log d)\sqrt{\frac{1}{N_{g_t}}}+\frac{\log d}{N_{g_t}}\right)\right),\numberthis
\end{align*}
where $m=\max_{g\in[G],j\in[d]:p_{Y_g}[j]>0}\{|u'(p_{Y_g}[j])|\}$, and $C=\widetilde{O}(d)$ is the upper bound of $\gE(\widehat{p}_{Y_{g_t}})-\widehat{H}(Y_{g_t}|X_{g_t})+\gB_{g_t}$. For the second part, note that the number of steps where at least one element of $\widetilde{p}_{Y_{g_t}}$ is clipped to at most $\widetilde{O}(G\Delta^{-2})$, where $\Delta=\min_{g\in[G],j\in[d]:p_{Y_g}[j]\neq e^{-1}}\{|p_{Y_g}[j]-e^{-1}|\}$. Therefore, we have that
\begin{equation*}
    \textup{Regret}(T)\le \widetilde{O}\left(G\Delta^{-2}\log d+e^C\cdot\left((dm+\log d)\cdot\sqrt{T}+\log T\right)\right),
\end{equation*}
which concludes the proof.
\end{proof}

\subsection{Data-dependent optimistic marginal class distribution}
\label{p_coro_domcd}

\begin{lemma}[Optimistic marginal class distribution]
\label{coro_domcd}
Let $x^1,\cdots,x^n\sim p_g$ be $n$ generated images from generator $g$. Define $\bm{\epsilon}^n_g\in\sR_+^d$ denote the error vector whose $j$-th element is given by
\begin{equation*}
    \bm{\epsilon}^n_g[j]=\sqrt{\frac{2\widehat{V}^n(p_{Y|X_g}[j])}{n}\log\left(\frac{4d}{\delta}\right)}+\frac{7}{3(n-1)}\log\left(\frac{4d}{\delta}\right),
\end{equation*}
where $\widehat{V}^n(p_{Y|X_g}[j])=\sV(p_{Y|x^1}[j],\cdots,p_{Y|x^n}[j])$ is the empirical variance for the $j$-th class density. Then, with probability at least $1-\frac{\delta}{2}$, we have that
\begin{equation*}
    \widehat{p}_{Y_g}^n=\textup{Clip}_{e^{-1}}\left(\widetilde{p}^n_{Y_g},\bm{\epsilon}^n_g\right)
\end{equation*}
satisfies that $\gE(\widehat{p}_{Y_g}^n)\ge H(Y_g)$.

\begin{proof}
It suffices that show that $\bm\epsilon_g^n[j]\ge |\widehat{p}_{Y_g}^n[j]-p_{Y_g}[j]|$ for all $j \in [d]$ with probability at least $1-\frac{\delta}{2}$. We evoke Lemma~\ref{thm_ebi} and conclude the proof.
\end{proof}
\end{lemma}

\subsection{Auxiliary Definitions and Lemmas}

\begin{lemma}[Optimistic marginal class distribution]
\label{thm_omcd}
Let $\widetilde{p}^n_{Y_g}:=\frac{1}{n}\sum_{i=1}^np_{Y_g|x^i}$ denote the empirical marginal class distribution. Let
\begin{equation*}
    \widehat{p}_{Y_g}^n=\textup{Clip}_{e^{-1}}\left(\widetilde{p}^n_{Y_g},|\widetilde{p}^n_{Y_g}-p_{Y_g}|\right),
\end{equation*}
where $\textup{Clip}_{e^{-1}}(p,\bm{\epsilon})$ is the following element-wise operator
\begin{equation*}
\begin{cases}
    p[j]+\frac{e^{-1}-p[j]}{|e^{-1}-p[j]|}\bm{\epsilon}[j] & \textup{, if }|e^{-1}-p[j]|\ge \bm{\epsilon}[j]\\
    e^{-1} & \textup{, otherwise.}
\end{cases}
\end{equation*}
for any vectors $p,\bm{\epsilon}\in\sR^d_+$. Then, we have that $\gE(\widehat{p}_{Y_g}^n)\ge H(Y_g)$.

\begin{proof}
Note that the function $u(z)=-z\log z$ is concave and attains maximum at $z=e^{-1}$. It suffices to show that $u(\widehat{p}_{Y_g}[j])\ge u(p_{Y_g}[j])$ for all $j\in[d]$. If $|e^{-1}-\widehat{p}_{Y_g}^n[j]|<|\widetilde{p}^n_{Y_g}[j]-p_{Y_g}[j]|$, then $\widehat{p}_{Y_g}[j]=e^{-1}$, which ensures that $u(\widehat{p}_{Y_g}[j])\ge u(p_{Y_g}[j])$. In addition, if $|e^{-1}-\widehat{p}_{Y_g}^n[j]|\ge|\widetilde{p}^n_{Y_g}[j]-p_{Y_g}[j]|$, then
\begin{equation*}
    \widehat{p}_{Y_g}^n[j]=\widetilde{p}^n_{Y_g}[j]+\frac{e^{-1}-\widetilde{p}^n_{Y_g}[j]}{|e^{-1}-\widetilde{p}^n_{Y_g}[j]|}|\widetilde{p}^n_{Y_g}[j]-p_{Y_g}[j]|=
    \begin{cases}
        \widetilde{p}^n_{Y_g}[j]+|\widetilde{p}^n_{Y_g}[j]-p_{Y_g}[j]|,&\textup{ if }p_{Y_g}[j]<e^{-1}\\
        \widetilde{p}^n_{Y_g}[j]-|\widetilde{p}^n_{Y_g}[j]-p_{Y_g}[j]|,&\textup{ otherwise.}
    \end{cases}.
\end{equation*}
The first case ensures that $p_{Y_g}[j]<\widehat{p}_{Y_g}^n[j]\le e^{-1}$, and the second case ensures that $p_{Y_g}[j]>\widehat{p}_{Y_g}^n[j]\ge e^{-1}$. Both cases satisfy that $u(\widehat{p}_{Y_g}^n[j])\ge u(p_{Y_g}[j])$. Therefore, we have that $H(Y_g)=\gE(p_{Y_g})=\sum_{j=1}^du(p_{Y_g}[j])\le\sum_{j=1}^du(\widehat{p}_{Y_g}^n[j])=\gE(\widehat{p}_{Y_g}^n)$, which concludes the proof.
\end{proof}
\end{lemma}

\begin{lemma}[Restatement of {\citep[Theorem 4]{maurer2009empirical}}]
\label{thm_ebi}
Let $Z,Z_1,\cdots,Z_n$ be i.i.d. random variables with values in $[0,1]$ and let $\delta>0$. we have that with probability at least $1-\delta$ in the i.i.d vector $\bm{Z}=(Z_1,\cdots,Z_n)$ that
\begin{equation*}
    \E Z-\frac{1}{n}\sum_{i=1}^n Z_i\le\sqrt{\frac{2V_n(\bm{Z})}{n}\log\left(\frac{2}{\delta}\right)}+\frac{7}{3(n-1)}\log\left(\frac{2}{\delta}\right),
\end{equation*}
where $V_n(\bm{Z}):=\frac{1}{n(n-1)}\sum_{1\le i<j\le n}(Z_i-Z_j)^2$ is the sample variance.
\end{lemma}

\section{\MakeUppercase{Experimental Details}}
\label{sec:aed}

\paragraph{1. List of pretrained generative models.}

For the CIFAR10 dataset, we compare pretrained generative models including iDDPM-DDIM~\citep{pmlr-v139-nichol21a}, LOGAN~\citep{wu2020loganlatentoptimisationgenerative}, WGAN-GP~\citep{NIPS2017_892c3b1c}, NVAE~\citep{NEURIPS2020_e3b21256}, and RESFLOW~\citep{NEURIPS2019_5d0d5594}. 
For ImageNet, we compare pretrained models including DiT-XL-2~\citep{peebles2023scalablediffusionmodelstransformers}, ADMG~\citep{NEURIPS2021_49ad23d1}, BigGAN~\citep{brock2018large}, RQ-Transformer~\citep{lee2022autoregressiveimagegenerationusing}, and ADM~\citep{NEURIPS2021_49ad23d1}. 
For the FFHQ dataset, we compare StyleNAT~\citep{walton2023stylenatgivingheadnew}, StyleGAN2-ADA~\citep{NEURIPS2020_8d30aa96}, LDM~\citep{rombach2022highresolutionimagesynthesislatent}, Unleashing-Transformers~\citep{bondtaylor2021unleashingtransformersparalleltoken}, and Efficient-vdVAE~\citep{hazami2022efficientvdvae}. We utilize the generated image datasets downloaded from the dgm-eval repository~\citep{NEURIPS2023_0bc795af}.

\paragraph{2. List of embeddings for FD-based evaluation and selection.}

We consider three standard encoders for image data: InceptionV3.Net~\citep{7780677}, DINOv2~\citep{oquab2024dinov2}, and CLIP~\citep{pmlr-v139-radford21a}. Following~\citep{NEURIPS2023_0bc795af}, we utilize the DINOv2 ViT-L/14 model and the OpenCLIP ViT-L/14 trained on DataComp-1B~\citep{gadre2023datacompsearchgenerationmultimodal}. Embeddings extracted by InceptionV3.Net, DINOv2, and CLIP have \num{2048}, \num{1024}, and \num{1024} dimensions, respectively.
For video data, we utilize the I3D model~\citep{carreira2018quovadisactionrecognition} pretrained on the Kinetics-400~\citep{carreira2018quovadisactionrecognition} dataset following~\citep{unterthiner2019fvd}, where the logits layer with \num{400} dimensions is used as the embedding. 
For audio data, we utilize VGGish~\citep{hershey2017cnnarchitectureslargescaleaudio} following~\citep{kilgour2019frechetaudiodistancemetric}, where the activations from the \num{128} dimensional layer prior to the final classification layer are
used as the embedding.

\paragraph{3. Implementation details of FD-UCB.}

As we consider InceptionV3, DINOv2, and (unnormalized) CLIP embeddings, which are generally unbounded, we utilize the collected data to estimate the model-dependent parameters in the bonus~(\ref{fid-b}). This approach preserves the online format of the FD-UCB algorithm. Our numerical results indicate that with only 50 samples, the norm terms in the bound would be close to their underlying values~(Tables~\ref{data-driven-paras},~\ref{data-driven-paras2}, and~\ref{data-driven-paras3}). To have a better estimation of the covariance matrix, we also adopt the thresholding method in~\citep{cai2011adaptivethresholdingsparsecovariance}. We treat the batch size, parameter $\kappa$ in the bonus~(\ref{fid-b}), and parameter $M$ for thresholding as hyperparameters. We conduct ablation study on these hyperparameters summarize the results in Figure~\ref{fd-ucb-ablation}.

\begin{table}[!ht]
    \centering
    \begin{tabular}{c | c c | c c | c c | c c | c}
        \toprule
        \textbf{Parameters} & $\tr(\Sigma)$ & $\tr(\widehat{\Sigma})$ & $\sqrt{\tr[\Sigma^2]}$ & $\sqrt{\tr[\widehat{\Sigma}^2]}$ & $\|\Sigma\|_2$ & $\|\widehat{\Sigma}\|_2$ & $\bm{r}(\Sigma)$ & $\bm{r}(\widehat{\Sigma})$ & Range of $\|f(X_g)\|_2$ \\
        \midrule
         DiT-XL-2 & \num{166.4} & \num{160.0} & \num{12.5} & \num{20.9} & \num{9.6} & \num{11.2} & \num{17.3} & \num{14.3} & $[9.7, 33.6]$ \\
         
         ADMG & \num{157.7} & \num{150.8} & \num{12.1} & \num{19.4} & \num{9.2} & \num{10.5} & \num{17.2} & \num{14.4} & $[10.8, 34.5]$\\
         
         BigGAN & \num{154.5} & \num{161.8} & \num{13.5} & \num{22.3} & \num{11.0} & \num{13.2} & \num{14.0} & \num{12.2} & $[9.7, 31.8]$   \\
         
         RQ-Transformer & \num{166.9} & \num{164.5} & \num{13.8} & \num{22.2} & \num{10.7} & \num{12.7} & \num{15.5} & \num{13.0} & $[11.0, 29.7]$ \\
         
         ADM & \num{176.5} & \num{182.5} & \num{15.7} & \num{25.6} & \num{12.9} & \num{15.2} & \num{13.7} & \num{12.0} & $[0.5, 34.8]$ \\
         
        \midrule
         ImageNet & \num{181.2} & \num{141.2} & \num{12.4} & \num{20.2} & \num{9.2} & \num{8.7} & \num{19.7} & \num{16.2} & $[10.9,33.0]$ \\
         
        \bottomrule
        \end{tabular}
    \caption{Data-dependent parameters on the ImageNet dataset and standard generative models: We present both the estimated and reference values, which are computed from \num{50} and $5,000$ images, respectively. The image data embeddings are extracted by InceptionV3.Net.}
    \label{data-driven-paras}
\end{table}

\begin{table}[!ht]
    \centering
    \begin{tabular}{c | c c | c c | c c | c c | c}
        \toprule
        \textbf{Parameters} & $\tr(\Sigma)$ & $\tr(\widehat{\Sigma})$ & $\sqrt{\tr[\Sigma^2]}$ & $\sqrt{\tr[\widehat{\Sigma}^2]}$ & $\|\Sigma\|_2$ & $\|\widehat{\Sigma}\|_2$ & $\bm{r}(\Sigma)$ & $\bm{r}(\widehat{\Sigma})$ & Range of $\|f(X_g)\|_2$ \\
        \midrule
         DiT-XL-2 & \num{591.2} & \num{601.9} & 
         \num{59.1} & \num{101.4} & 
         \num{38.2} & \num{48.0} & 
         \num{15.5} & \num{12.5} & 
         $[34.6, 36.9]$ \\
         
         ADMG & \num{579.9} & \num{566.1} & 
         \num{62.1} & \num{101.3} & 
         \num{41.2} & \num{50.4} & 
         \num{14.1} & \num{11.2} & 
         $[34.9, 37.0]$\\
         
         BigGAN & \num{514.9} & \num{486.8} &
         \num{62.7} & \num{90.8} &
         \num{43.8} & \num{54.0} &
         \num{11.8} & \num{9.0} &
         $[34.9, 36.8]$   \\
         
         RQ-Transformer & \num{545.5} & \num{551.7} &
         \num{58.8} & \num{99.1} &
         \num{39.3} & \num{51.0} &
         \num{13.9} & \num{10.8} &
         $[35.0, 36.8]$ \\
         
         ADM & \num{560.2} & \num{561.8} &
         \num{59.0} & \num{100.7} &
         \num{38.3} & \num{52.5} &
         \num{14.6} & \num{10.7} &
         $[34.9, 36.8]$ \\
         
        \midrule
         ImageNet & \num{636.8} & \num{573.1} &
         \num{53.5} & \num{118.5} &
         \num{32.1} & \num{74.1} &
         \num{19.9} & \num{7.7} &
         $[35.1,36.7]$ \\
         
        \bottomrule
        \end{tabular}
    \caption{Data-dependent parameters on the ImageNet dataset and standard generative models: We present both the estimated and reference values, which are computed from \num{50} and $5,000$ images, respectively. The image data embeddings are extracted by CLIP.}
    \label{data-driven-paras2}
\end{table}

\begin{table}[!ht]
    \centering
    \begin{tabular}{c | c c | c c | c c | c c | c}
        \toprule
        \textbf{Parameters} & $\tr(\Sigma)$ & $\tr(\widehat{\Sigma})$ & $\sqrt{\tr[\Sigma^2]}$ & $\sqrt{\tr[\widehat{\Sigma}^2]}$ & $\|\Sigma\|_2$ & $\|\widehat{\Sigma}\|_2$ & $\bm{r}(\Sigma)$ & $\bm{r}(\widehat{\Sigma})$ & Range of $\|f(X_g)\|_2$ \\
        \midrule
         DiT-XL-2 & \num{2079.5} & \num{2085.6} & 
         \num{99.9} & \num{309.4} & 
         \num{36.5} & \num{81.4} & 
         \num{57.0} & \num{25.6} & 
         $[38.5, 48.7]$ \\
         
         ADMG & \num{2068.2} & \num{2102.1} & 
         \num{103.0} & \num{311.4} & 
         \num{42.5} & \num{79.1} & 
         \num{48.7} & \num{25.6} & 
         $[38.6, 49.0]$\\
         
         BigGAN & \num{1885.5} & \num{1799.6} &
         \num{120.3} & \num{285.4} &
         \num{68.8} & \num{110.4} &
         \num{27.4} & \num{16.3} &
         $[36.5, 48.6]$   \\
         
         RQ-Transformer & \num{1955.3} & \num{1968.0} &
         \num{111.4} & \num{302.5} &
         \num{54.4} & \num{91.1} &
         \num{36.0} & \num{21.6} &
         $[38.4, 49.3]$ \\
         
         ADM & \num{2013.3} & \num{1988.1} &
         \num{106.7} & \num{306.0} &
         \num{50.5} & \num{101.0} &
         \num{39.9} & \num{19.7} &
         $[39.1, 49.2]$ \\
         
        \midrule
         ImageNet & \num{2112.7} & \num{2095.4} &
         \num{85.3} & \num{347.3} &
         \num{21.0} & \num{104.3} &
         \num{100.7} & \num{20.1} &
         $[42.2,48.3]$ \\
         
        \bottomrule
        \end{tabular}
    \caption{Data-dependent parameters on the ImageNet dataset and standard generative models: We present both the estimated and reference values, which are computed from \num{50} and $5,000$ images, respectively. The image data embeddings are extracted by DINOv2.}
    \label{data-driven-paras3}
\end{table}

\paragraph{4. Implementation details of Naive-UCB.} 

Naive-UCB is a simplification of FD-UCB and IS-UCB which replaces the generator-dependent variables in the bonus function with data-independent and dimension-based terms. For FD-based evaluation, Naive-UCB sets $\tr[\Sigma_g]=O(d),\tr[\Sigma_g^2]=O(d)$, and $\|\Sigma_g\|_2=O(1)$ in FD-UCB. For IS-based evaluation, the Naive-UCB method sets $\widehat{V}^n(Y_g|X_g)=(\log d)^2$ and $\widehat{V}^n(p_{Y|X_g}[j])=1$ for any $j\in[d]$.

\begin{figure*}[!ht]
\centering
    \makebox[25pt][r]{\makebox[20pt]{\raisebox{50pt}{\rotatebox[origin=c]{90}{\textbf{Avg. Regret}}}}}
    \subfigure{\includegraphics[width=0.3\textwidth]{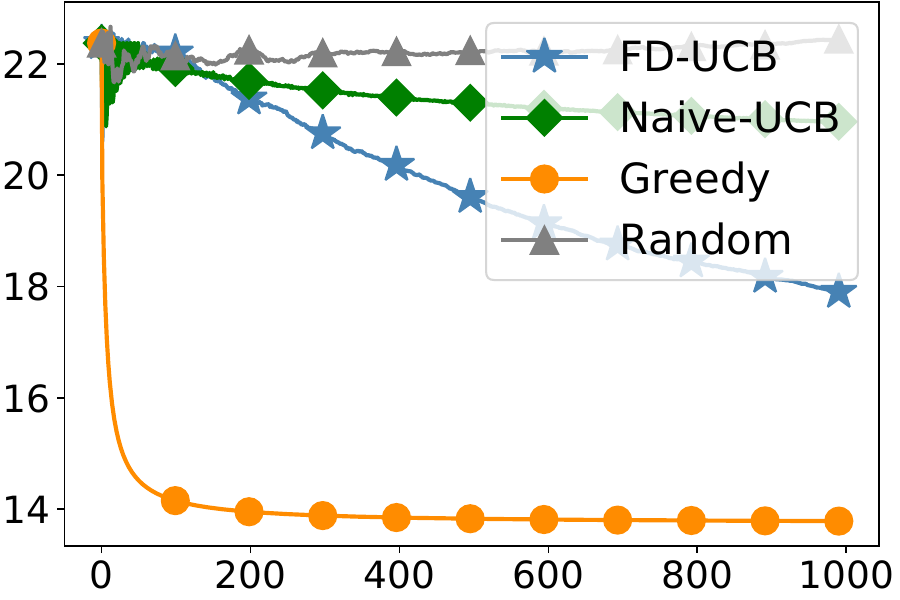}}
    \hfill
    \subfigure{\includegraphics[width=0.3\textwidth]{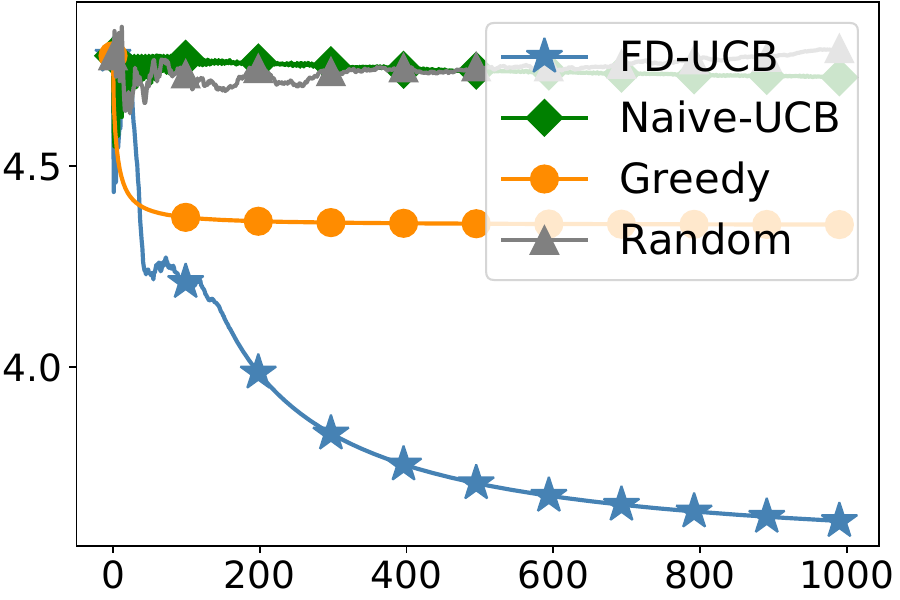}}
    \hfill
    \subfigure{\includegraphics[width=0.3\textwidth]{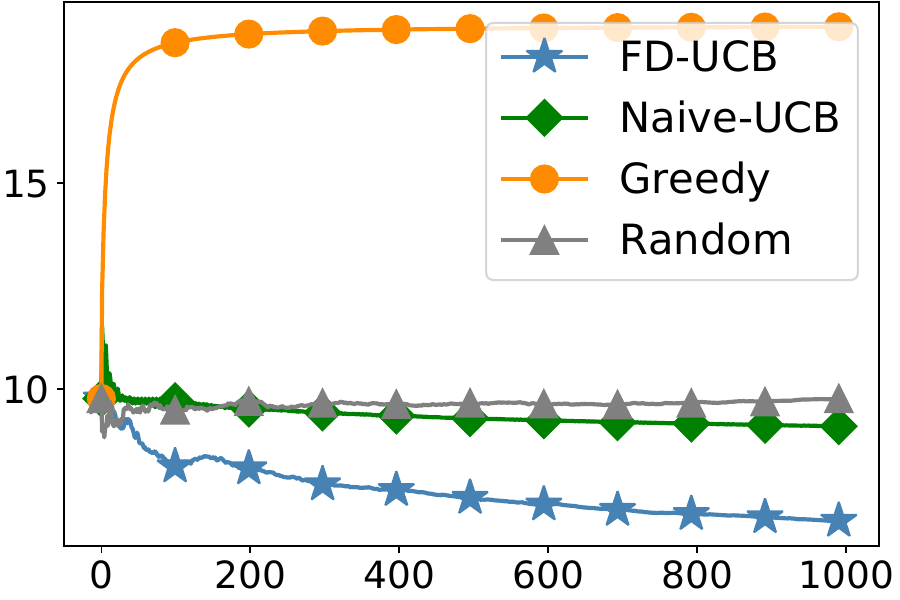}} \\
    
    \makebox[25pt][r]{\makebox[20pt]{\raisebox{50pt}{\rotatebox[origin=c]{90}{\textbf{OPR}}}}}
    \stackunder{\subfigure{\includegraphics[width=0.3\textwidth]{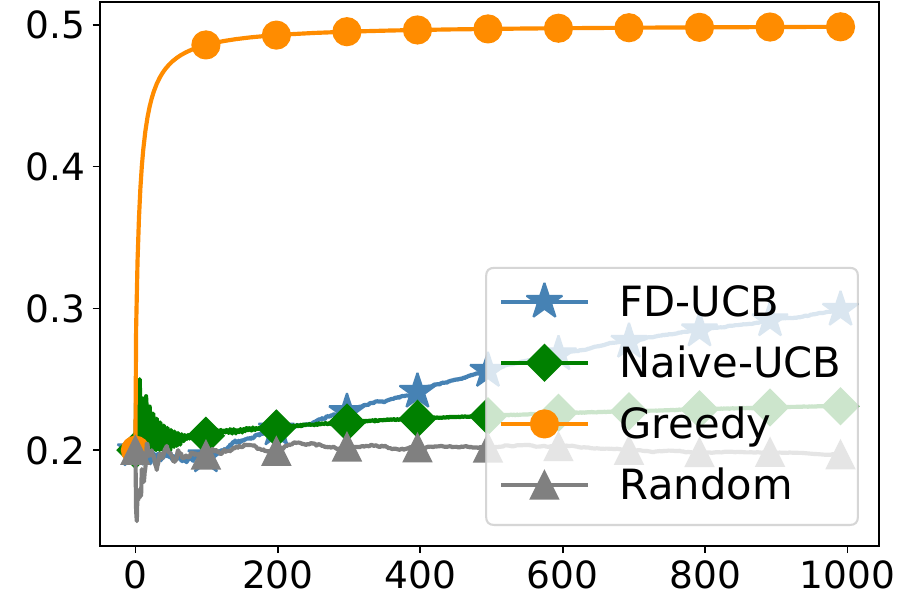}}}{\textbf{CIFAR10}}
    \hfill
    \stackunder{\subfigure{\includegraphics[width=0.3\textwidth]{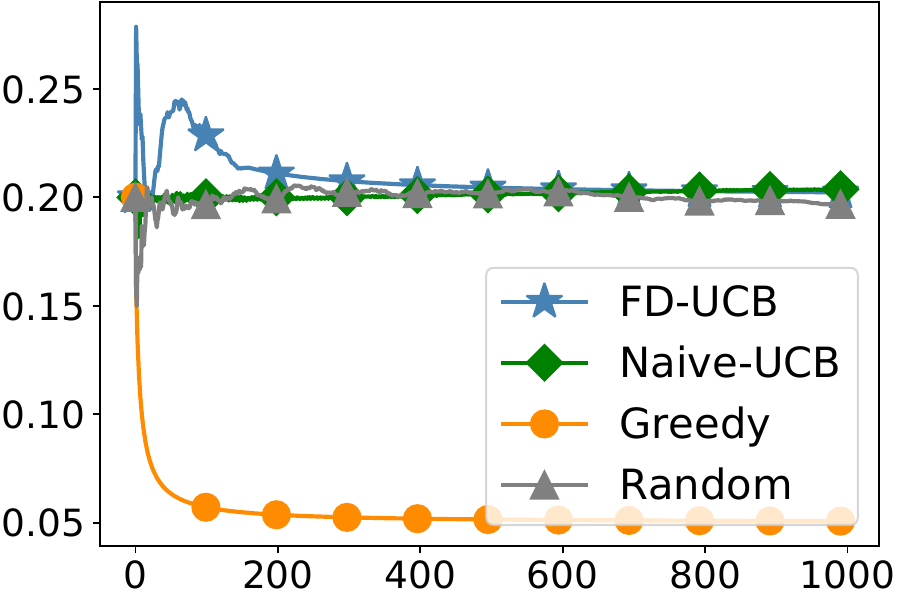}}}{\textbf{ImageNet}}
    \hfill
    \stackunder{\subfigure{\includegraphics[width=0.3\textwidth]{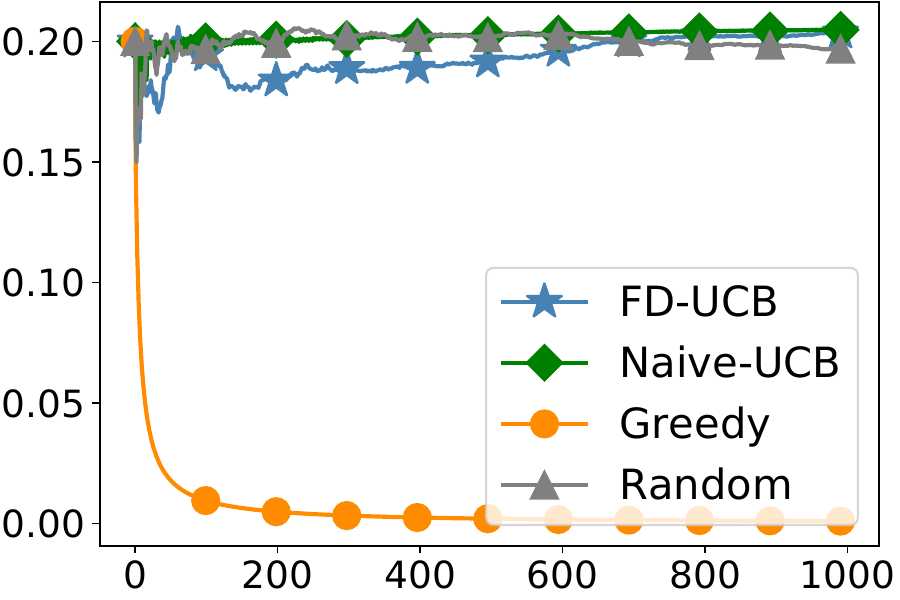}}}{\textbf{FFHQ}}
\caption{Online FD-based evaluation and selection among standard generative models: The image data embeddings are extracted by InceptionV3.Net. Results are averaged over 20 trials.}
\label{fd-inception-results}
\end{figure*}

\begin{figure*}[!ht]
\centering
    \makebox[25pt][r]{\makebox[20pt]{\raisebox{50pt}{\rotatebox[origin=c]{90}{\textbf{Avg. Regret}}}}}
    \subfigure{\includegraphics[width=0.3\textwidth]{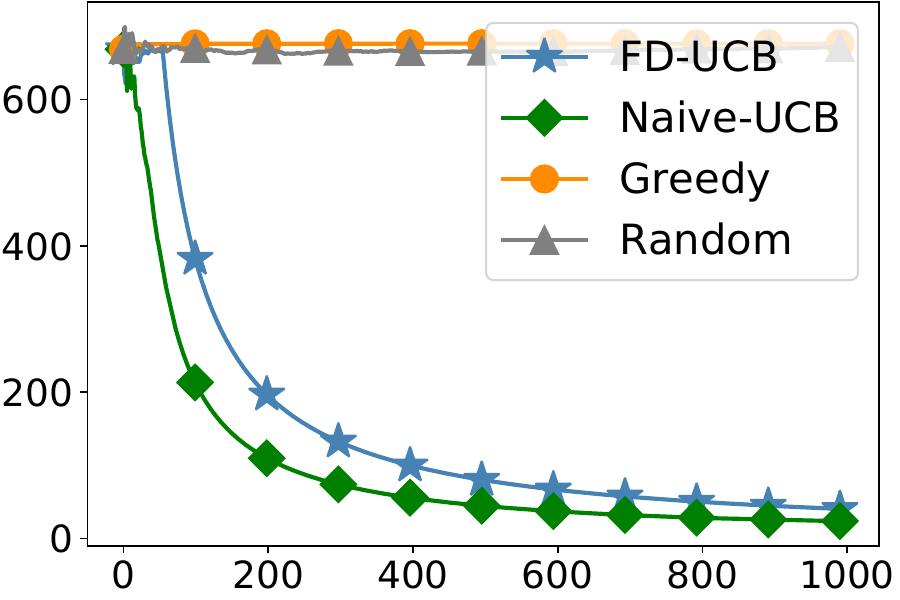}}
    \hfill
    \subfigure{\includegraphics[width=0.3\textwidth]{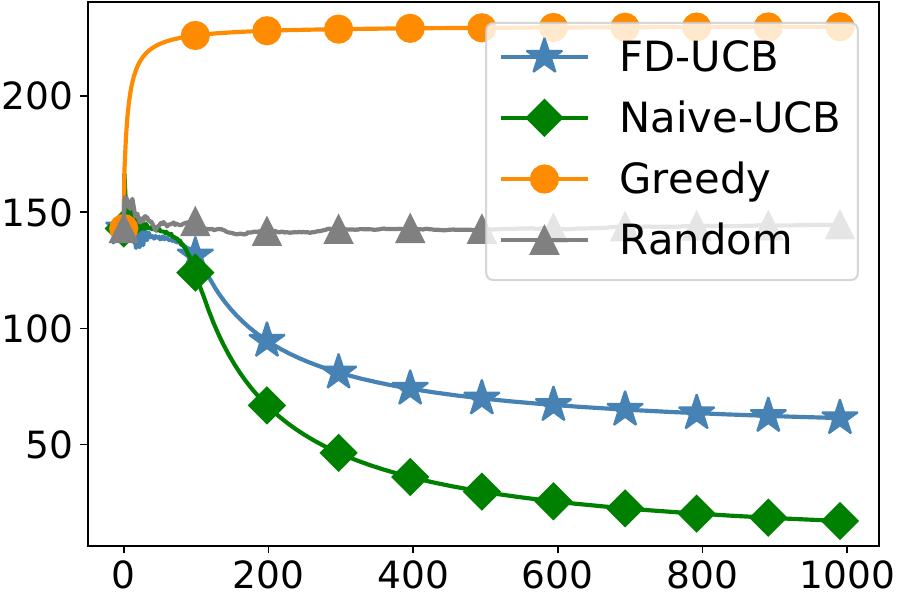}}
    \hfill
    \subfigure{\includegraphics[width=0.3\textwidth]{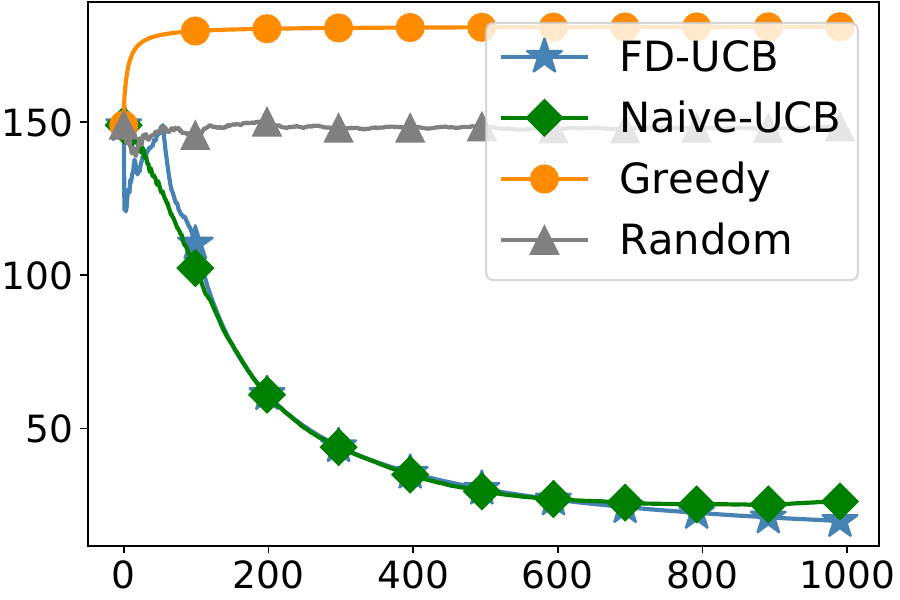}} \\
    
    \makebox[25pt][r]{\makebox[20pt]{\raisebox{50pt}{\rotatebox[origin=c]{90}{\textbf{OPR}}}}}
    \stackunder{\subfigure{\includegraphics[width=0.3\textwidth]{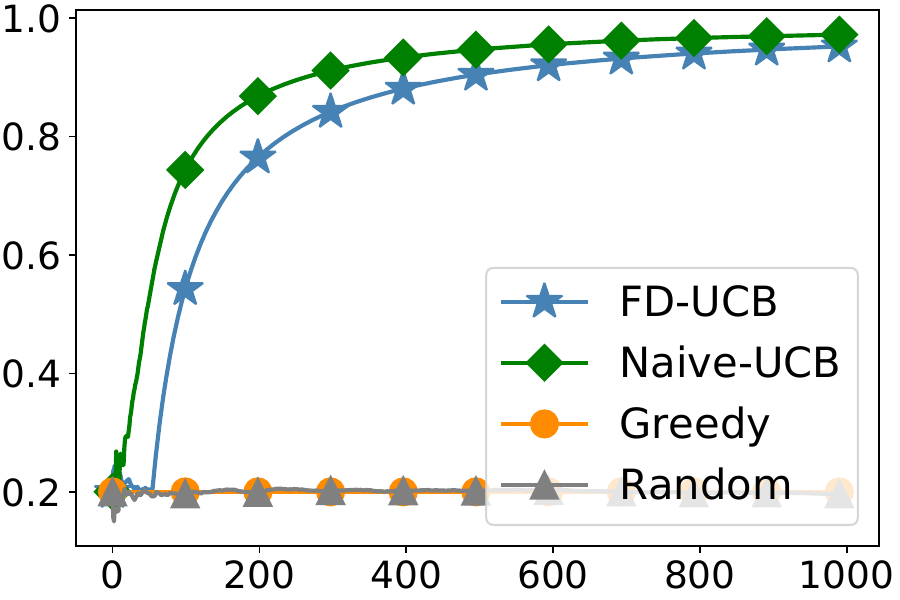}}}{\textbf{CIFAR10}}
    \hfill
    \stackunder{\subfigure{\includegraphics[width=0.3\textwidth]{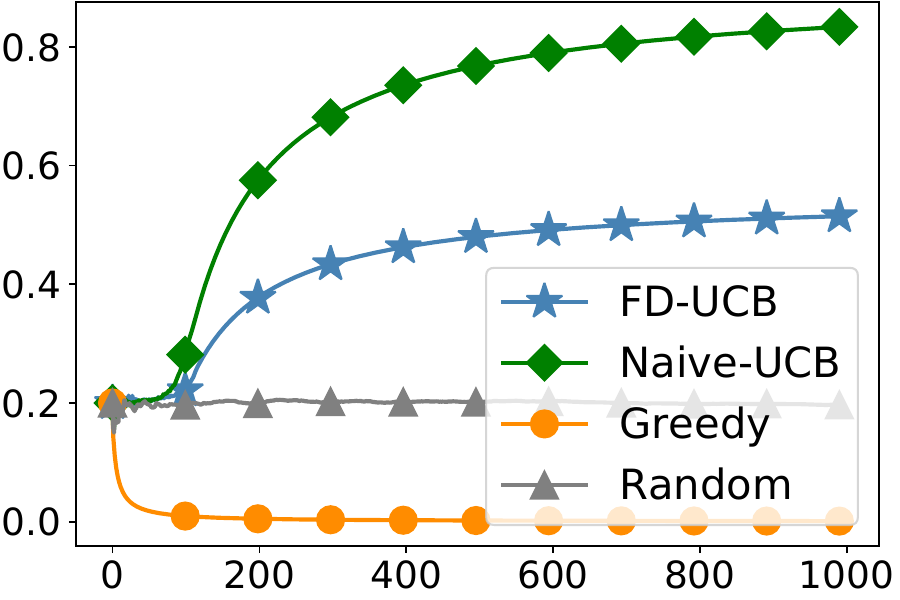}}}{\textbf{ImageNet}}
    \hfill
    \stackunder{\subfigure{\includegraphics[width=0.3\textwidth]{new_results/FD_inception_eval_opr_ffhq.pdf}}}{\textbf{FFHQ}}
\caption{Online FD-based evaluation and selection among standard generative models: The image data embeddings are extracted by DINOv2-ViT-L/14. Results are averaged over 20 trials.}
\label{fd-dinov2-results}
\end{figure*}

\begin{figure*}[!ht]
\centering
    \makebox[25pt][r]{\makebox[20pt]{\raisebox{50pt}{\rotatebox[origin=c]{90}{\textbf{FFHQ (VC)}}}}}
    \subfigure{\includegraphics[width=0.4\textwidth]{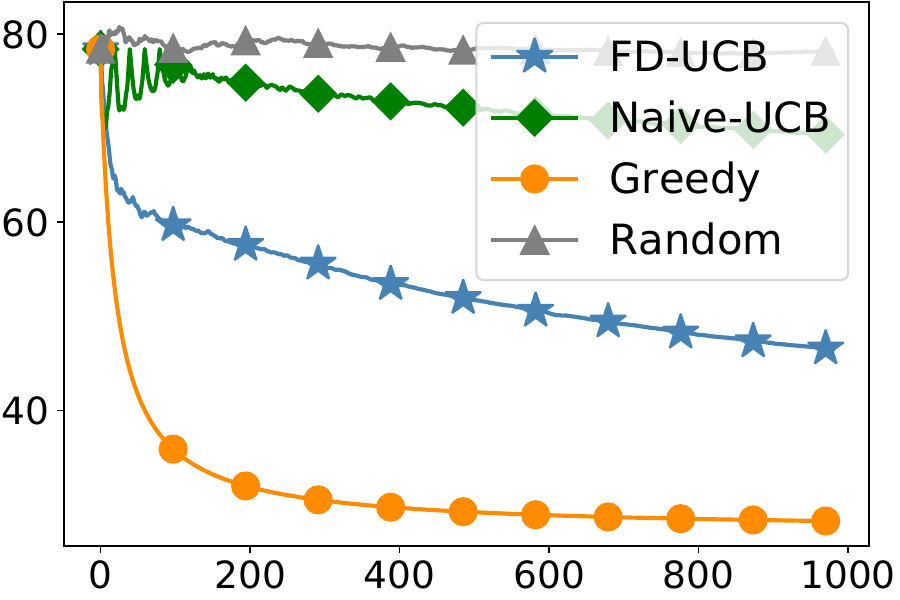}}
    \qquad
    \subfigure{\includegraphics[width=0.4\textwidth]{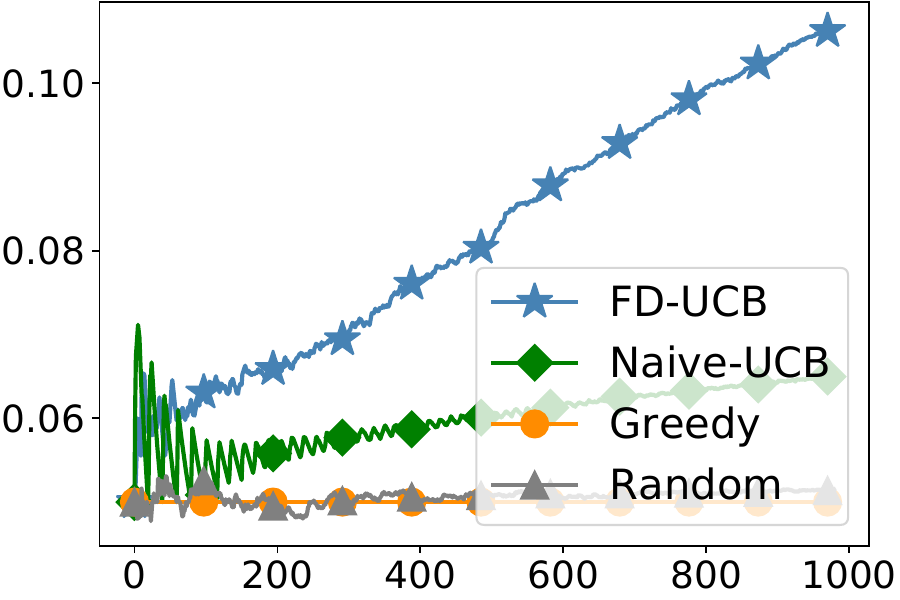}} \\
    \makebox[25pt][r]{\makebox[20pt]{\raisebox{50pt}{\rotatebox[origin=c]{90}{\textbf{AFHQ-Dog (VC)}}}}}
    \stackunder{\subfigure{\includegraphics[width=0.4\textwidth]{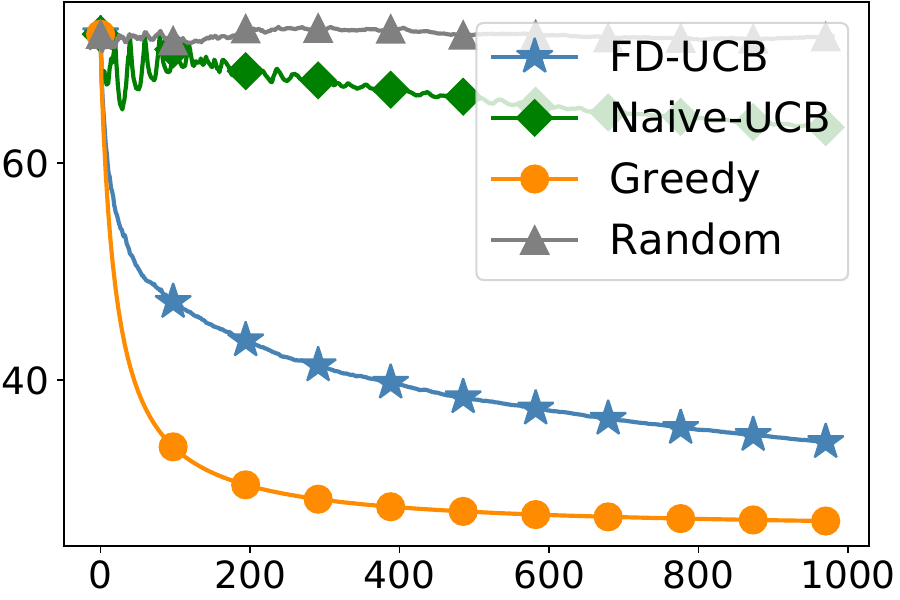}}}{\textbf{Avg. Regret}}
    \qquad
    \stackunder{\subfigure{\includegraphics[width=0.4\textwidth]{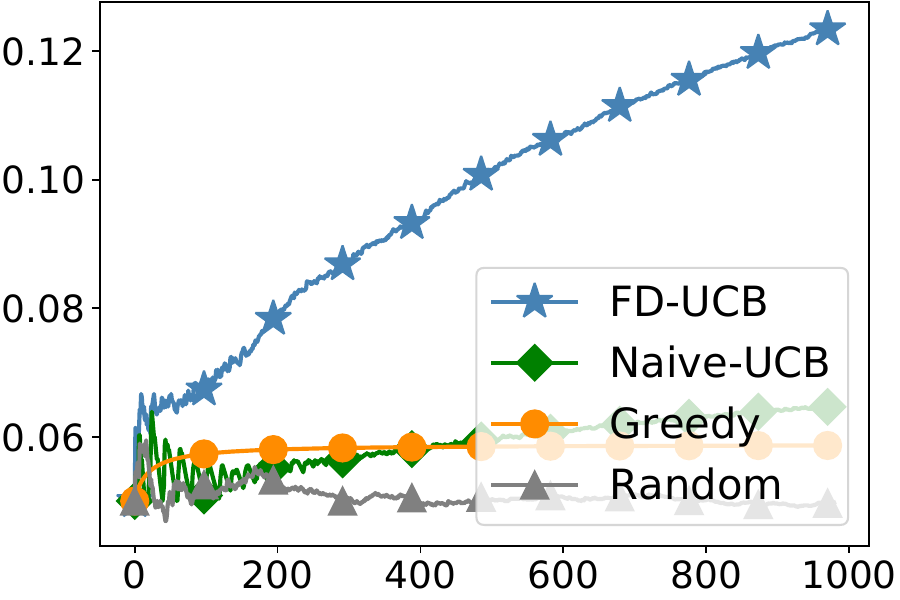}}}{\textbf{OPR}}
\caption{Online FD-based evaluation and selection among variance-controlled~(VC) models: The image data embeddings are extracted by InceptionV3.Net. Results are averaged over 20 trials.}
\label{fd-results-vc-inception}
\end{figure*}

\begin{figure*}[!ht]
\centering
    \makebox[25pt][r]{\makebox[20pt]{\raisebox{50pt}{\rotatebox[origin=c]{90}{\textbf{FFHQ (VC)}}}}}
    \subfigure{\includegraphics[width=0.4\textwidth]{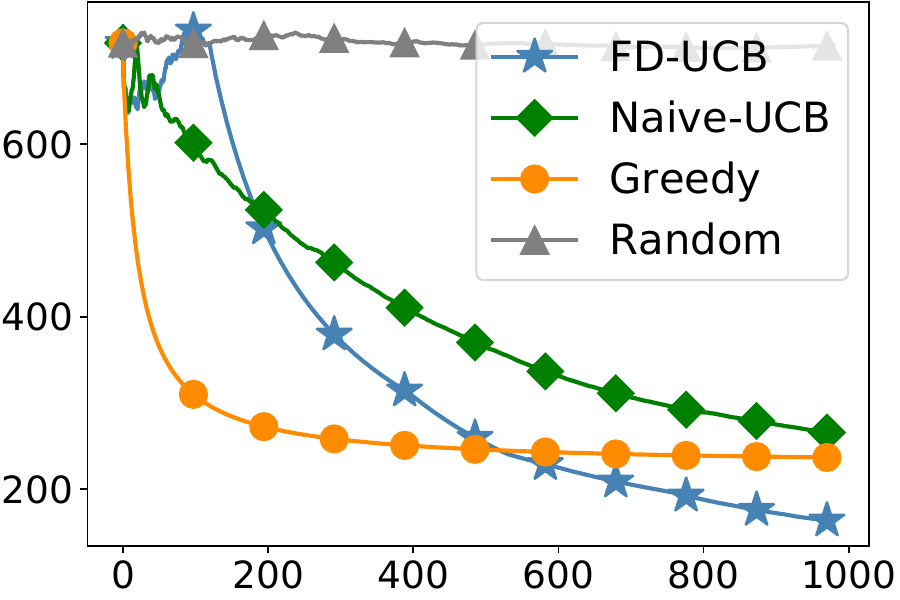}}
    \qquad
    \subfigure{\includegraphics[width=0.4\textwidth]{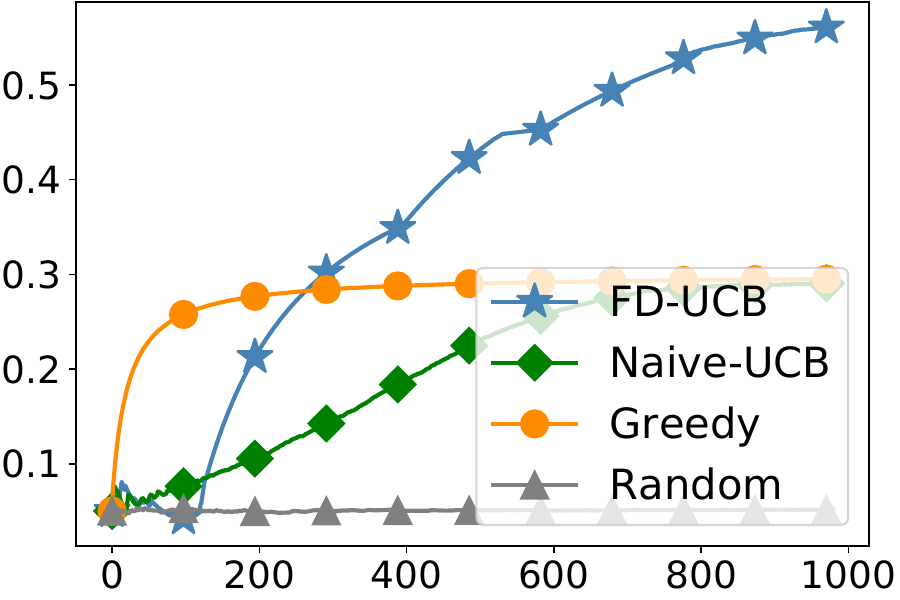}} \\
    \makebox[25pt][r]{\makebox[20pt]{\raisebox{50pt}{\rotatebox[origin=c]{90}{\textbf{AFHQ-Dog (VC)}}}}}
    \stackunder{\subfigure{\includegraphics[width=0.4\textwidth]{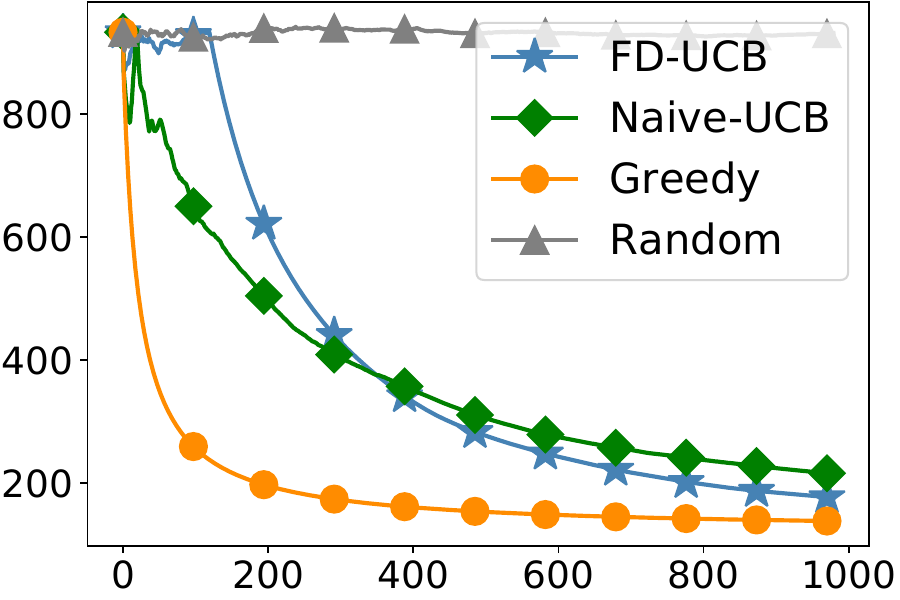}}}{\textbf{Avg. Regret}}
    \qquad
    \stackunder{\subfigure{\includegraphics[width=0.4\textwidth]{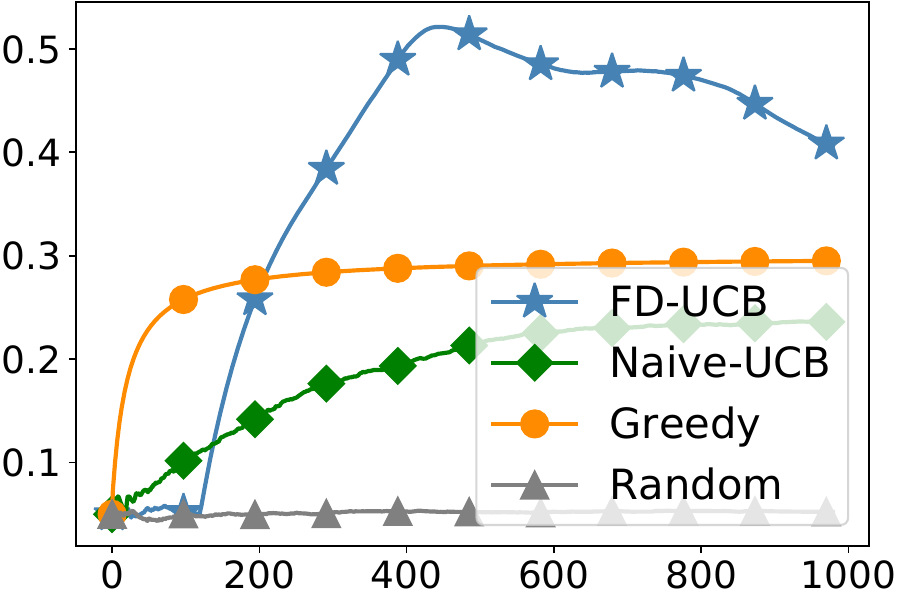}}}{\textbf{OPR}}
\caption{Online FD-based evaluation and selection among variance-controlled~(VC) models: The image data embeddings are extracted by DINOv2. Results are averaged over 20 trials.}
\label{fd-results-vc-dinov2}
\end{figure*}

\begin{figure*}[!ht]
\centering
    \makebox[25pt][r]{\makebox[20pt]{\raisebox{50pt}{\rotatebox[origin=c]{90}{\textbf{FFHQ (VC)}}}}}
    \subfigure{\includegraphics[width=0.4\textwidth]{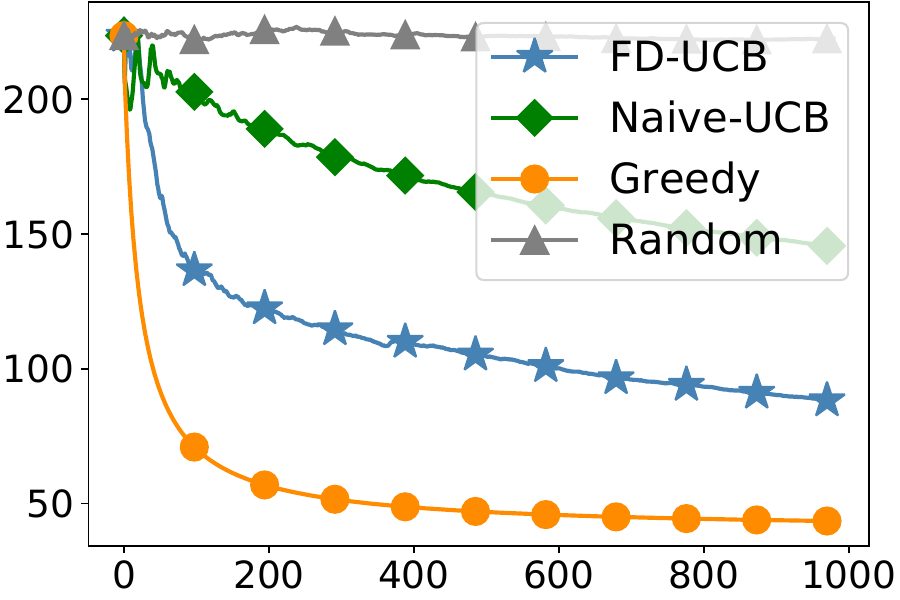}}
    \qquad
    \subfigure{\includegraphics[width=0.4\textwidth]{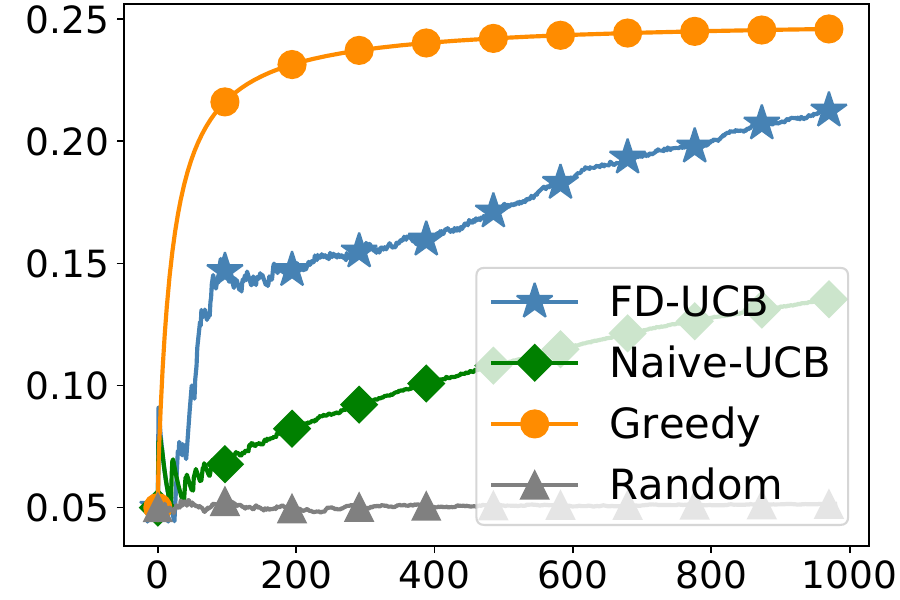}} \\
    \makebox[25pt][r]{\makebox[20pt]{\raisebox{50pt}{\rotatebox[origin=c]{90}{\textbf{AFHQ-Dog (VC)}}}}}
    \stackunder{\subfigure{\includegraphics[width=0.4\textwidth]{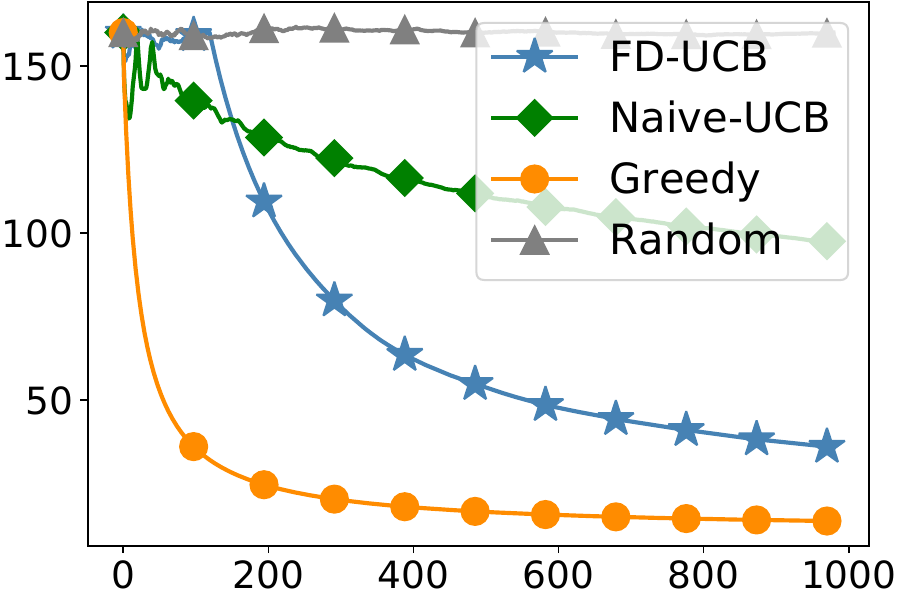}}}{\textbf{Avg. Regret}}
    \qquad
    \stackunder{\subfigure{\includegraphics[width=0.4\textwidth]{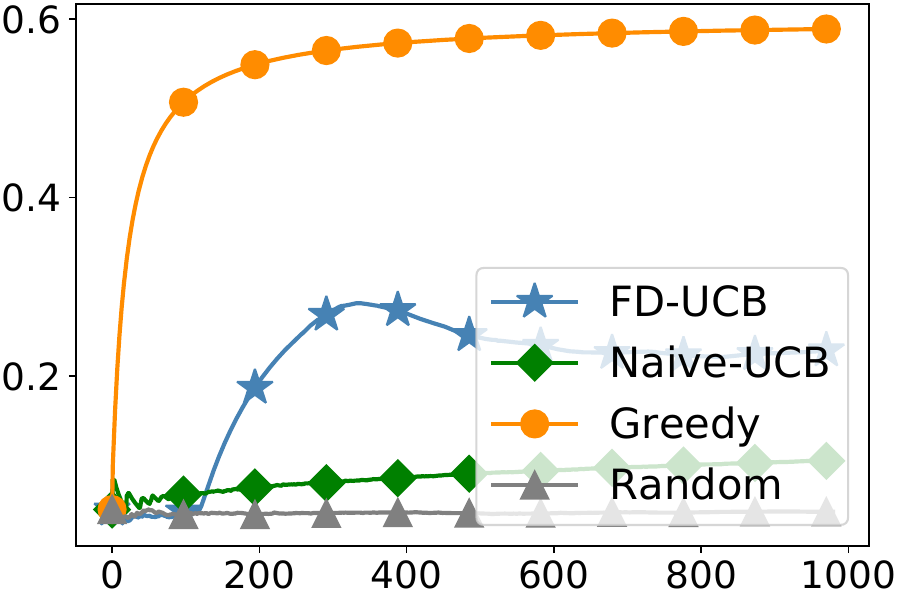}}}{\textbf{OPR}}
\caption{Online FD-based evaluation and selection among variance-controlled~(VC) models: The image data embeddings are extracted by CLIP. Results are averaged over 20 trials.}
\label{fd-results-vc-clip}
\end{figure*}

\begin{figure*}[!ht]
\centering
    \makebox[25pt][r]{\makebox[20pt]{\raisebox{50pt}{\rotatebox[origin=c]{90}{\textbf{MSR-VTT (Video)}}}}}
    \subfigure{\includegraphics[width=0.4\textwidth]{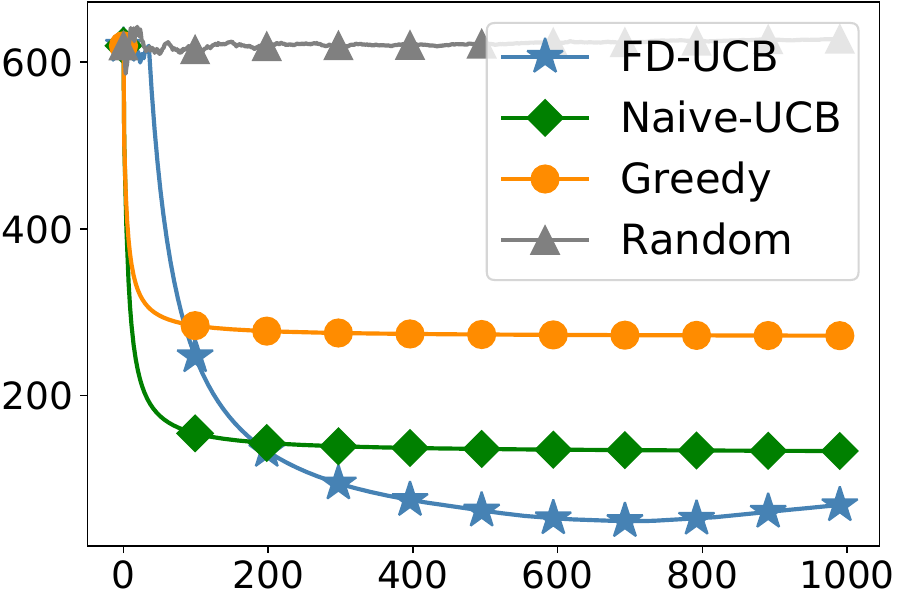}}
    \qquad
    \subfigure{\includegraphics[width=0.4\textwidth]{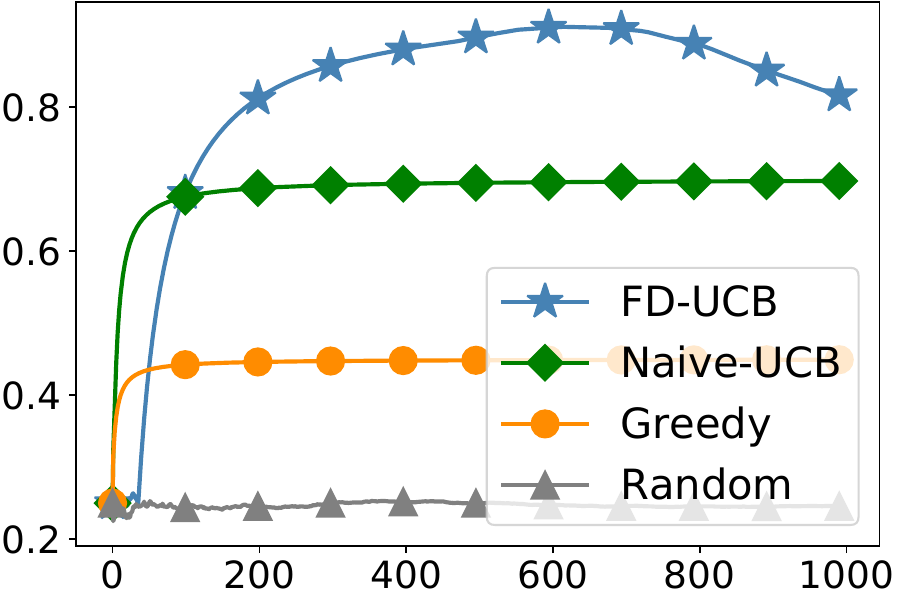}} \\
    \makebox[25pt][r]{\makebox[20pt]{\raisebox{50pt}{\rotatebox[origin=c]{90}{\textbf{Magnatagatune (Audio)}}}}}
    \stackunder{\subfigure{\includegraphics[width=0.4\textwidth]{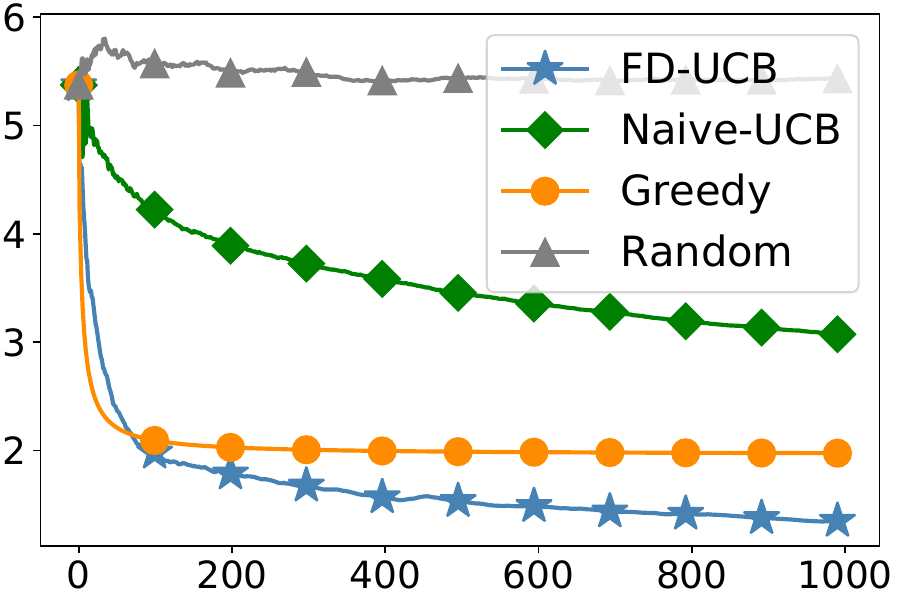}}}{\textbf{Avg. Regret}}
    \qquad
    \stackunder{\subfigure{\includegraphics[width=0.4\textwidth]{new_results/FD_VGGish_eval_opr_Magnatagatune-sync.pdf}}}{\textbf{OPR}}
\caption{Online evaluation and selection on video and audio data: The video data embeddings are extracted by I3D, and the audio data embeddings are extracted by VGGish. Results are averaged over 20 trials.}
\label{fd-results-sync}
\end{figure*}

\begin{figure*}[!ht]
\centering
    \subfigure{\includegraphics[width=0.4\textwidth]{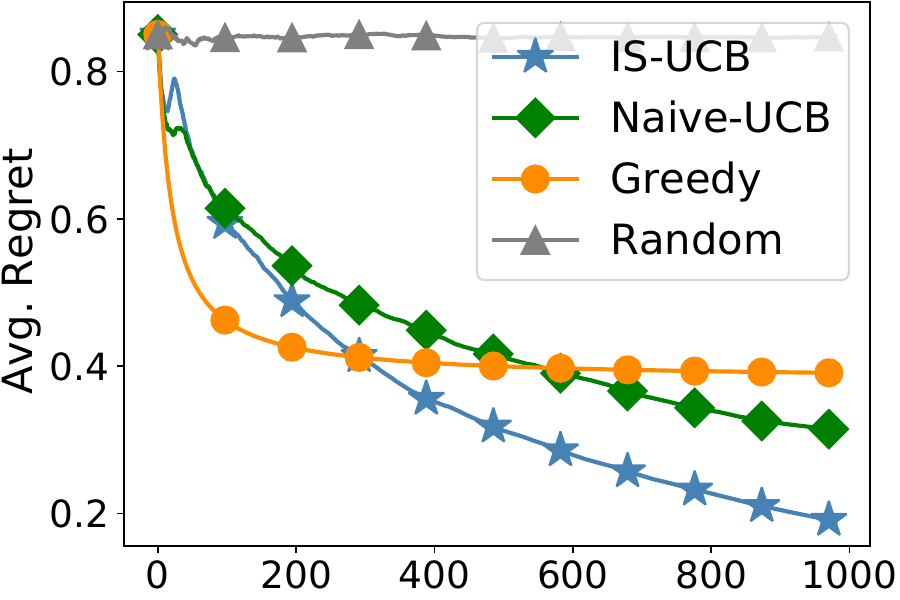}}
    \qquad
    \subfigure{\includegraphics[width=0.4\textwidth]{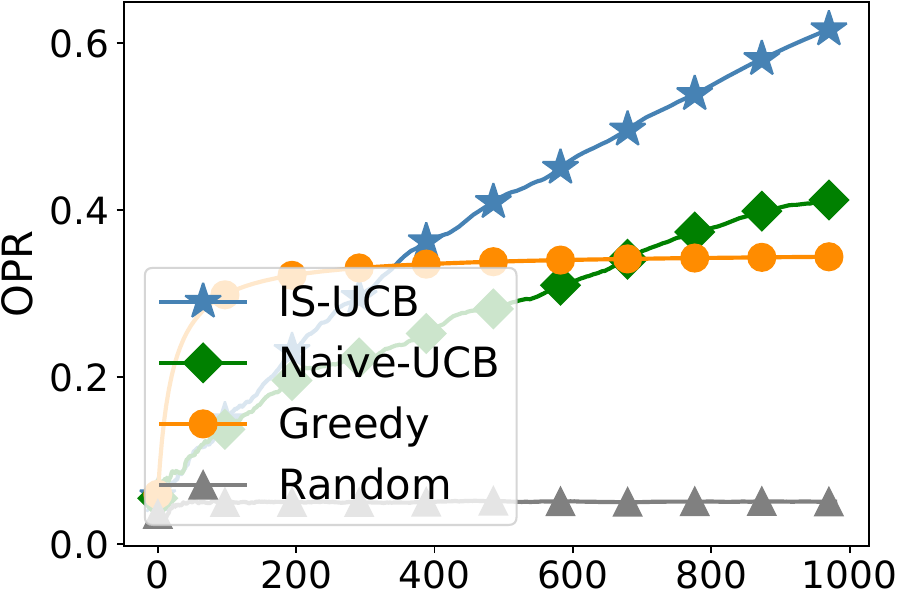}}
\caption{Online IS-based evaluation and selection among variance-controlled models on the AFHQ Dog dataset: IS-UCB can identify models that generate images with more diversity. Results are averaged over 20 trials.}
\label{is-results-afhqdog-vc}
\end{figure*}

\begin{figure*}[!ht]
\centering
    \makebox[25pt][r]{\makebox[20pt]{\raisebox{50pt}{\rotatebox[origin=c]{90}{\textbf{Batch Size}}}}}
    \subfigure{\includegraphics[width=0.4\textwidth]{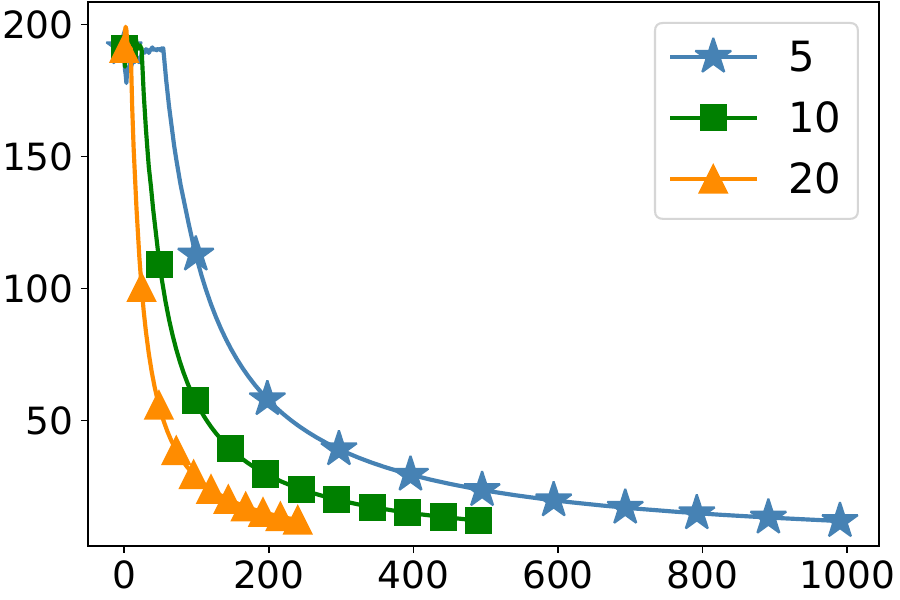}}
    \qquad
    \subfigure{\includegraphics[width=0.4\textwidth]{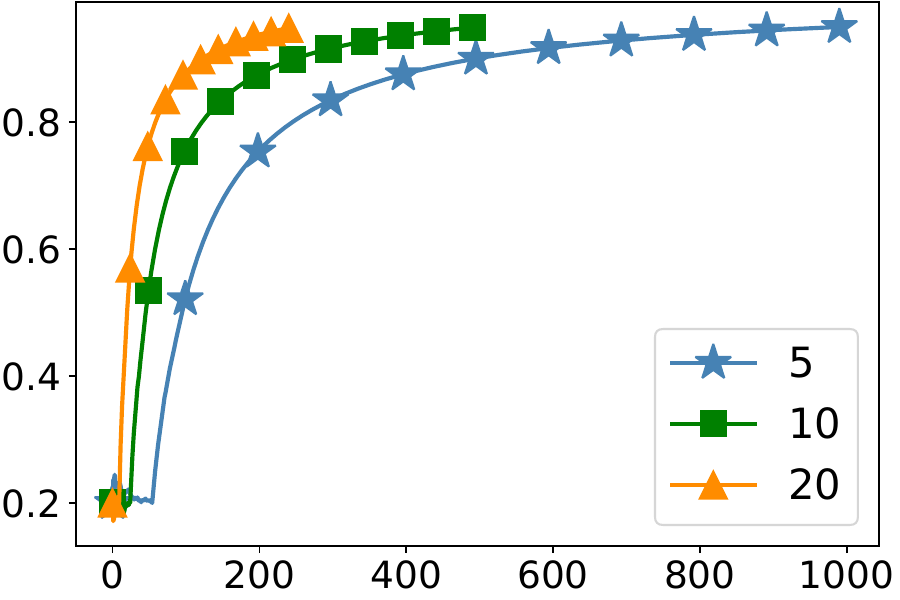}} \\
    
    \makebox[25pt][r]{\makebox[20pt]{\raisebox{50pt}{\rotatebox[origin=c]{90}{\textbf{Parameter $\kappa$}}}}}
    \subfigure{\includegraphics[width=0.4\textwidth]{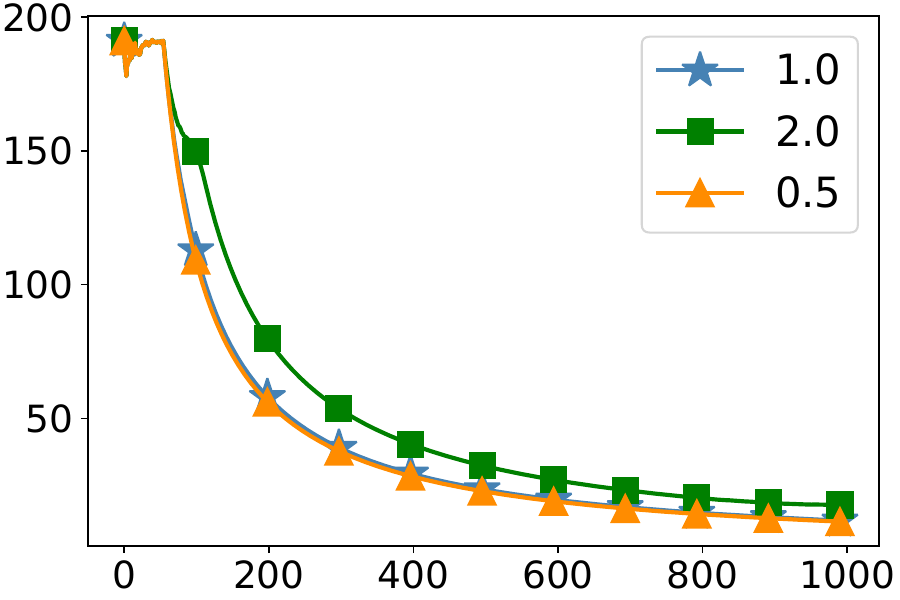}}
    \qquad
    \subfigure{\includegraphics[width=0.4\textwidth]{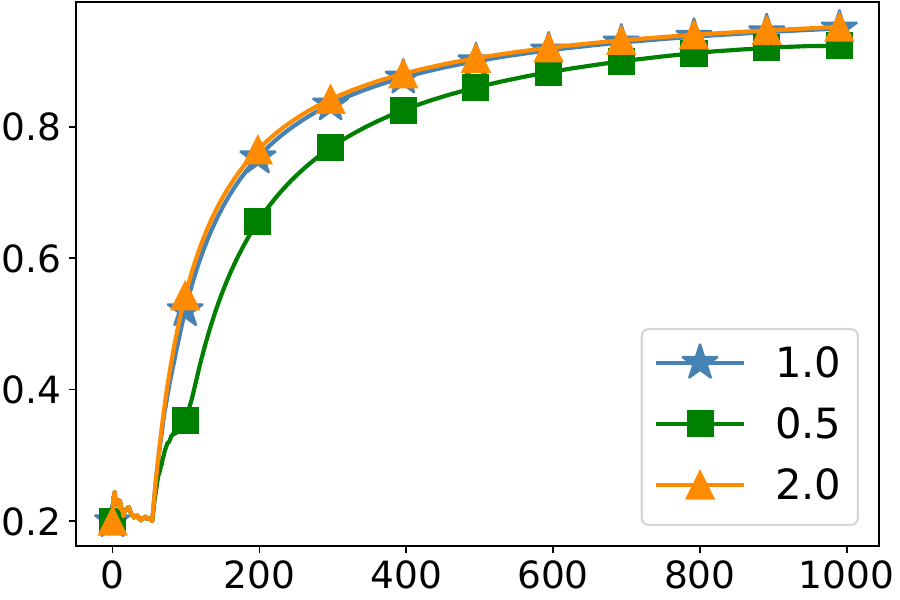}} \\
    
    \makebox[25pt][r]{\makebox[20pt]{\raisebox{50pt}{\rotatebox[origin=c]{90}{\textbf{Parameter $M$}}}}}
    \stackunder{\subfigure{\includegraphics[width=0.4\textwidth]{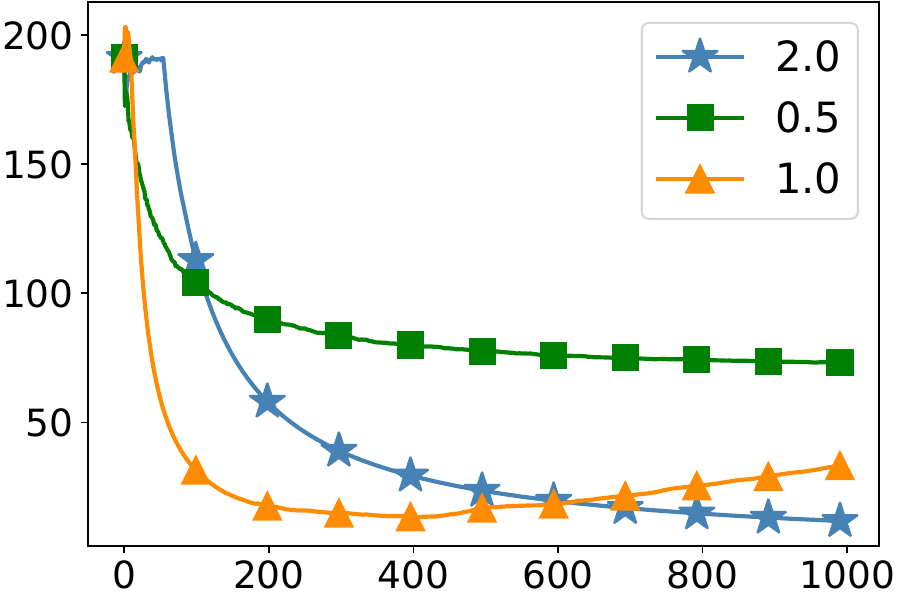}}}{\textbf{Avg. Regret}}
    \qquad
    \stackunder{\subfigure{\includegraphics[width=0.4\textwidth]{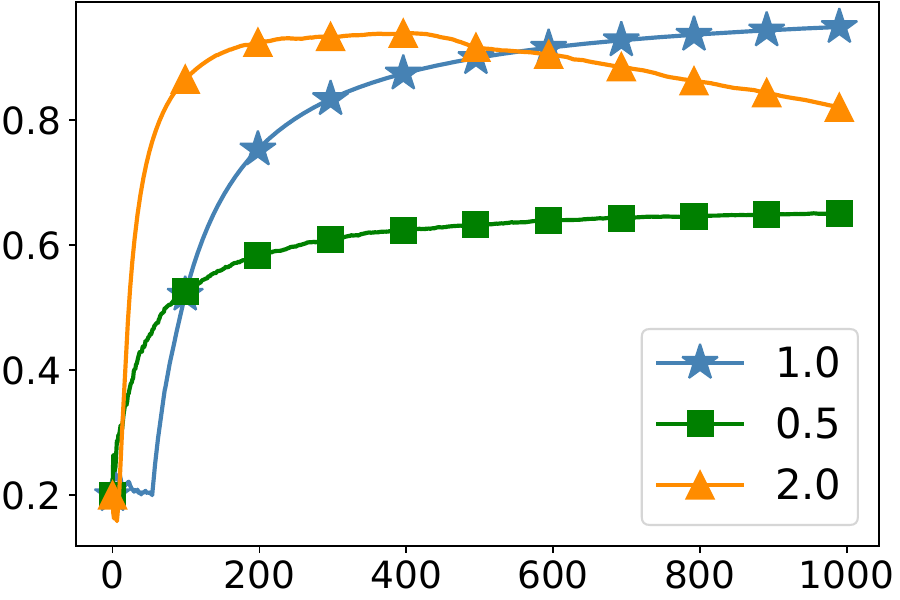}}}{\textbf{OPR}}
\caption{Ablation study on the hyperparameters of FD-UCB on the CIFAR10 dataset: The image data embeddings are extracted by CLIP. Results are averaged over 20 trials.}
\label{fd-ucb-ablation}
\end{figure*}

\begin{figure*}[!ht]
\centering
    \makebox[25pt][r]{\makebox[20pt]{\raisebox{50pt}{\rotatebox[origin=c]{90}{\textbf{Magnatagatune (Audio)}}}}}
    \subfigure{\includegraphics[width=0.4\textwidth]{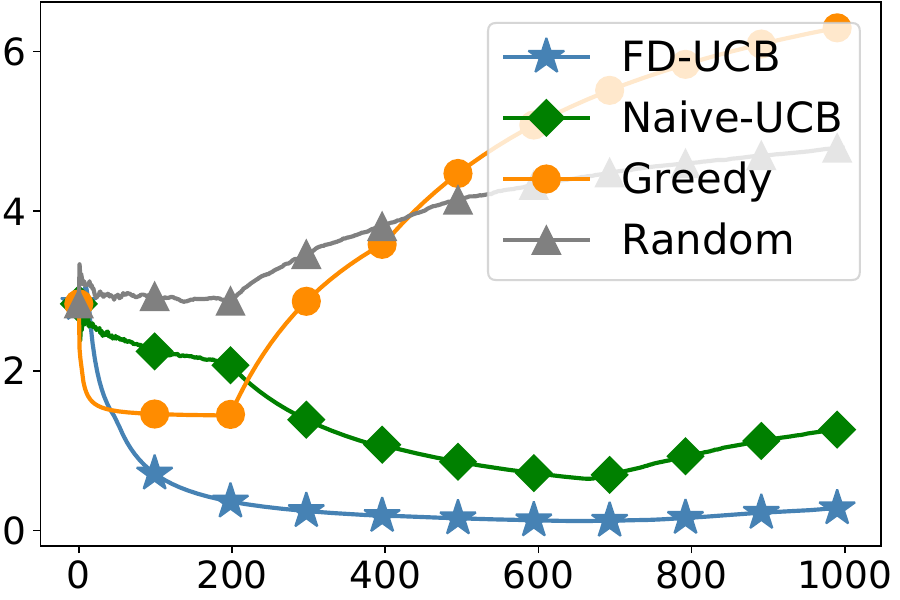}}
    \qquad
    \subfigure{\includegraphics[width=0.4\textwidth]{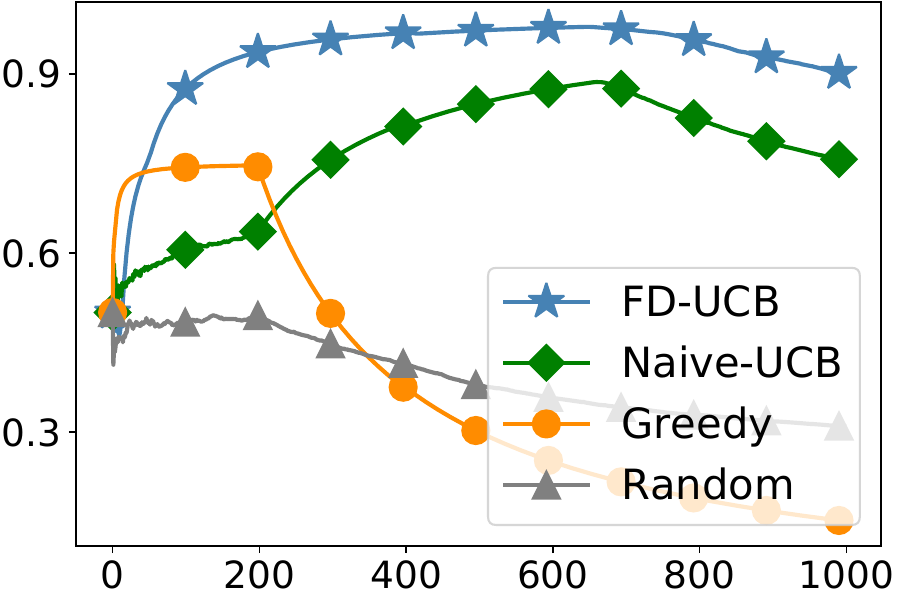}} \\
    \makebox[25pt][r]{\makebox[20pt]{\raisebox{50pt}{\rotatebox[origin=c]{90}{\textbf{CIFAR10}}}}}
    \stackunder{\subfigure{\includegraphics[width=0.4\textwidth]{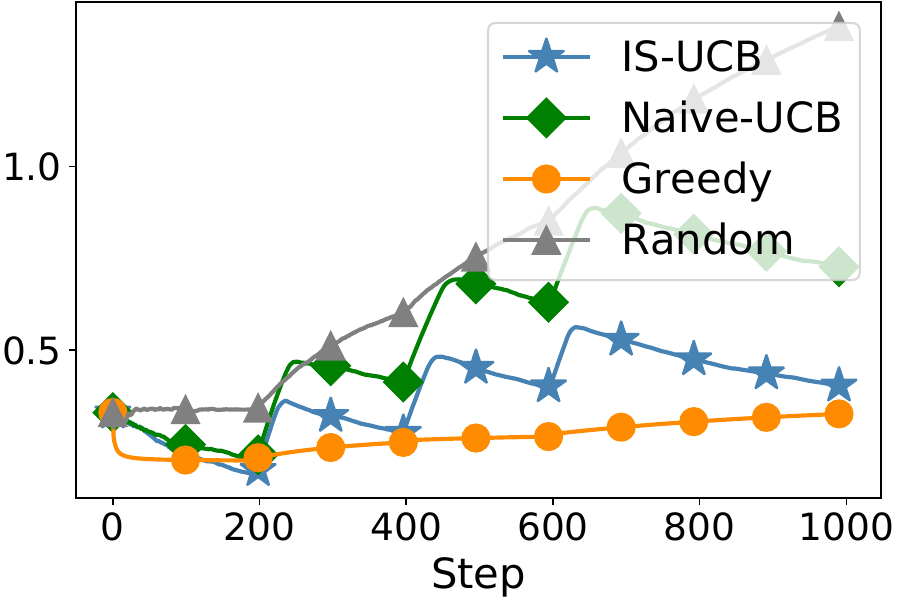}}}{\textbf{Avg. Regret}}
    \qquad
    \stackunder{\subfigure{\includegraphics[width=0.4\textwidth]{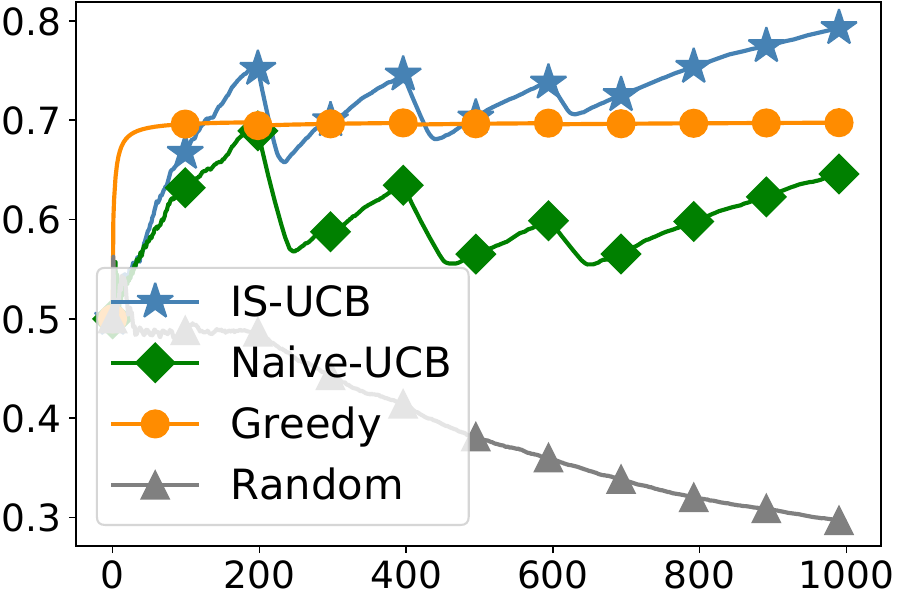}}}{\textbf{OPR}}
\caption{Adaptation to new models: A new model is introduced after each \num{200} steps. Results are averaged over 20 trials.}
\label{new-model}
\end{figure*}

\end{document}